%% file: main_JMLR.tex
\documentclass[twoside,11pt]{article}

\usepackage{amsthm}
\usepackage[preprint]{jmlr2e}

\usepackage[T1]{fontenc}
\usepackage[utf8]{inputenc}
\usepackage{lmodern}
\usepackage{anyfontsize}
\usepackage{microtype}
\usepackage{amsmath,amssymb,amsfonts}
\usepackage{amsmath}
\usepackage{mathtools}
\usepackage{amsfonts}
\usepackage{amssymb}
\usepackage{bm}
\usepackage{bbm}
\usepackage{ifthen}

\usepackage{booktabs}
\usepackage{nicefrac}
\usepackage{tikz}
\usepackage{caption}
\usepackage{subcaption}
\usepackage{wrapfig}
\usepackage{graphicx}
\graphicspath{{Fig/}}

\usepackage{url}
\usepackage{xcolor}
\usepackage{setspace}
\usepackage{enumitem}
\usepackage{comment}
\usepackage{lastpage}

\newcommand{\circled}[1]{%
\tikz[baseline=(char.base)]%
\node[draw,shape=circle,inner sep=1pt](char){#1};}

\input{defs.tex}

\title{On the Generalization Gap of Streaming Federated Learning with Markovian Data}

\author{
\name Masoud Kavian \email masoud.kavian@huawei.com \\
\addr Huawei Paris Research Center, Huawei Technologies Ltd., France
\AND
\name Milad Sefidgaran \email milad.sefidgaran2@huawei.com \\
\addr Huawei Paris Research Center, Huawei Technologies Ltd., France
\AND
\name Youssef Attia El Hili \email youssef.attiaeh@gmail.com \\
\addr Huawei Noah's Ark Lab, Paris, France
\AND
\name Malik Tiomoko \email tiomoko\_malik@yahoo.fr \\
\addr Huawei Noah's Ark Lab, Paris, France
\AND
\name Abdellatif Zaidi \email abdellatif.zaidi@univ-eiffel.fr \\
\addr Universit\'e Gustave Eiffel, France
}

\ShortHeadings{Generalization of Streaming Federated Learning}
{Kavian et al.}

\firstpageno{1}

\begin{document}

\maketitle

\begin{abstract}
Streaming federated learning (SFL) allows clients to train
a shared model from continuously arriving data without
retaining large local datasets. Such streams often exhibit
temporal dependence, yet its effect on generalization
remains poorly understood. In this paper we study the expected
generalization gap of the aggregated model when client
data follow Markov processes, using information-theoretic
and rate-distortion techniques. Our framework accommodates
heterogeneous Markovian dynamics across clients and
distinguishes temporal dependence from information
acquired during local learning. We establish bounds
under cross-entropy loss, specializing them to Linear
Next-Token Prediction (LNTP) and Self-Attention
Next-Token Prediction (SANTP). These results quantify
the effects of client count, communication rounds,
sample size, data dimension, mixing time, and context
length. In particular, they reveal a
communication--generalization tradeoff: although
frequent aggregation may help optimization, it can
weaken generalization guarantees. They also suggest a potential advantage of distributed over
centralized streaming learning under a fixed total data budget.
Temporal dependence enters through the square root
of the mixing time. By analyzing the history-state
process, we obtain refined bounds with improved
context-length dependence under suitable mixing
conditions. Experiments on GPT-style language-model
pretraining with TinyStories and real-world
time-series forecasting illustrate how collaboration
and communication frequency influence generalization
under temporal dependence.
\end{abstract}


\section{Introduction}\vspace{-0.2cm}

Federated Learning (FL) provides a framework in which multiple clients collaboratively train a shared model while keeping their data decentralized~\cite{mcmahanetal2017communication,kairouz2021advances}. Much of the classical FL literature considers clients with locally stored datasets that remain available throughout training. In this setting, clients perform several optimization steps on their local data before communicating model updates to a central server, allowing communication costs to be reduced through periodic aggregation~\cite{konecny2016federated,stich2019local,woodworth2020local}. This approach is well suited to applications in which data can be retained and repeatedly accessed, but it is less appropriate for devices that continuously acquire new observations under limited memory and computational resources. Such constraints arise in environmental monitoring through Internet-of-Things sensors~\cite{hill2010anomaly}, wearable health-monitoring systems~\cite{greco2019edge}, and on-device language modeling~\cite{hard2018federated}. In these applications, clients may need to update their models as observations arrive, rather than relying on repeated access to previously collected data. Streaming Federated Learning (SFL) addresses this setting by combining collaborative model training with the sequential acquisition and processing of local data~\cite{marfoq2023federated,liu2023dynamite}.

A central difficulty in streaming applications is that observations arriving over time are rarely independent. Measurements produced by physical systems, biological processes, and human activity often exhibit persistent temporal dependencies, which can be modeled using stochastic processes with finite memory~\cite{vankampen1992stochastic,paulin2015concentration}. The distinction is important because a collection of temporally correlated observations generally contains less independent statistical information than an equally sized i.i.d.\ sample. In SFL, these dependencies interact with the evolution of the learned model: each client updates its parameters using a local data stream, while successive communication rounds propagate information from earlier observations through the aggregated model. Consequently, the generalization behavior of the final model depends not only on the amount of data collected but also on how quickly the underlying processes mix, how frequently clients communicate, and how information is accumulated across rounds. Although these factors are closely related to the optimization and communication tradeoffs studied in distributed learning~\cite{karimireddy2020scaffold,li2020fedprox,reddi2021adaptive}, their combined influence on the generalization gap of SFL with Markovian data is not yet well understood.

This question is particularly relevant for next-token prediction, where the observations within a sequence are dependent and each prediction is formed from a finite history. Increasing the prediction context allows the model to exploit more temporal information, but it also increases the number of loss terms affected by an individual observation. The resulting dependence is therefore shaped both by the memory of the data-generating process and by the context used by the predictor. While the expressive capabilities and structural properties of self-attention have been extensively investigated~\cite{vaswani2017attention,edelman2022inductive,malach2023auto}, understanding how these two sources of dependence interact with collaborative learning remains a distinct statistical problem. In particular, it is unclear how the number of clients, communication frequency, context length, and mixing properties of the local data streams jointly determine the generalization gap of an aggregated next-token predictor. 
A further question is whether the dependence on context
length can be sharpened by exploiting the mixing behavior
of the underlying history-state process.

In this paper, we address this problem through an information-theoretic and rate-distortion analysis of SFL with Markovian data. We first establish general bounds on the expected generalization gap of the aggregated model under cross-entropy loss, explicitly accounting for the temporal dependence of the data collected by each client. Our analysis accommodates heterogeneous Markovian dynamics across clients and separates the information acquired by local learning algorithms from the statistical effect of temporal dependence. 
We also derive a refined concentration-based bound by
analyzing the Markovian evolution of the history-state
process rather than individual observations. This
refinement provides an alternative characterization
of temporal dependence and can reduce the explicit
context-length penalty under suitable mixing conditions.
 We then develop a rate-distortion formulation that controls generalization through approximate representations of the learned models, extending the applicability of the bounds to settings in which ordinary mutual-information quantities may be uninformative~\cite{Sefidgaran2022,sefidgaran2022rate,arora2018stronger}. Finally, we specialize the general framework to Linear Next-Token Prediction (LNTP-SFL) and Self-Attention Next-Token Prediction (SANTP-SFL) architectures, obtaining explicit bounds in terms of the number of clients, communication rounds, data size and dimension, Markovian mixing time, and prediction context length.

For the special case of homogeneous mixing dynamics across clients, with mixing time $\tau_{\text{min}}$, $K$ clients each collecting $N$ trajectories of length $T$, $R$ communication rounds, dictionary dimension $d$, and context length $\rho$, the resulting bounds for LNTP-SFL and SANTP-SFL take the form

\begin{align}
\mathcal{O}\left(
\frac{1}{KR}
\sum_{k=1}^{K}\sum_{r=1}^{R}
\frac{C^s\rho^{\bar{s}}}{\theta^2}
\sqrt{
\frac{\tau_{k,\min}R L_r\log(\rho dNT)}
{K^2NT}
}
\right),
\end{align}
where $(s,\bar{s})=(2,3)$ for LNTP-SFL and
$(s,\bar{s})=(4,2)$ for SANTP-SFL, and $L_r$
captures the contraction of local updates
across subsequent communication rounds.

These bounds make explicit the interaction between distributed learning, temporal dependence, and model complexity. Our main findings are summarized as follows:
\begin{itemize}

\item \textit{The communication-generalization tradeoff}: Our bounds exhibit an explicit dependence on the number of communication rounds, with a factor of order $\mathcal{O}(\sqrt{R/N})$ when the remaining quantities are fixed. Thus, although more frequent communication may improve optimization, it can weaken the resulting generalization guarantees. This identifies a tradeoff between the optimization benefits of frequent aggregation and the statistical dependence accumulated through repeated communication.

\item \textit{The distributed advantage}: The bounds reveal a favorable dependence on the number of participating clients, including a factor of order $\mathcal{O}(\sqrt{\log K/K})$ under the corresponding specialization. This highlights a potential statistical benefit of collaborative learning and provides a basis for comparing federated and centralized streaming schemes under their respective data-budget assumptions.

\item \textit{Local updates as a buffer}: The analysis identifies the number of local updates between successive communication rounds as an important parameter governing generalization. Allowing clients to process more observations before communicating can reduce the communication-dependent contribution to the bound, suggesting that less frequent aggregation may help limit the statistical cost associated with repeatedly exchanging models learned from dependent data.

\item \textit{Quantifying the impact of memory and context}:
Temporal dependence enters the bounds through the
square root of the Markovian mixing time,
$\sqrt{\tau_{\text{min}}}$, reflecting the reduction
in effective statistical information caused by
slowly mixing data. The prediction context introduces
an additional complexity penalty through the sensitivity
of the loss to changes in neighboring observations.
By considering the history-state process, we further
establish refined bounds that can reduce this explicit
context-length dependence when the corresponding mixing
properties are sufficiently favorable. Together, these
results clarify how the memory of the underlying process
and the context length of the prediction model jointly
affect generalization.

\item \textit{Experiments}: We complement the theoretical results with experiments involving GPT-style decoder pretraining on the TinyStories corpus~\cite{eldan2023tinystoriessmalllanguagemodels}, as well as real-world multi-site air-quality and long-term electricity-forecasting benchmarks. These experiments examine the effects of collaborative learning, communication frequency, and temporal dependence on the generalization gap. We also report the corresponding training and test errors, recognizing that a smaller generalization gap does not necessarily imply a lower test error.

\end{itemize}

\vspace{-0.4cm}

\section{Related Prior Art}\vspace{-0.2cm}

Federated learning has been extensively studied as a means of training shared models from decentralized datasets while limiting the need to transmit raw observations. Early work introduced communication-efficient model averaging as a practical alternative to centralized training~\cite{mcmahanetal2017communication}, while subsequent research examined the effects of local computation, data heterogeneity, and communication constraints on the convergence of distributed optimization algorithms~\cite{konecny2016federated,stich2019local,li2020fedprox,karimireddy2020scaffold,woodworth2020local}. These studies established that the number of local optimization steps and the frequency of communication can substantially influence optimization performance, particularly when clients have heterogeneous data distributions. Related developments include adaptive federated optimization and convergence analyses for FedAvg under non-i.i.d.\ sampling~\cite{reddi2021adaptive,li2020fedavg}. For a broader account of the foundations and challenges of federated learning, we refer to~\cite{kairouz2021advances}. Although this literature provides important insights into communication-efficient training, its primary emphasis is on convergence and optimization error rather than the generalization gap induced by temporally dependent streaming observations.

Streaming FL differs from the conventional setting in that clients receive new data over time and may be unable to retain or revisit their earlier observations. This makes the learning process closely related to online and continual learning, where models must adapt to sequential information under computational or memory constraints. Online learning has traditionally been analyzed through regret and sequential decision-making guarantees~\cite{shalev2012online,odeyomi2021differentially}, whereas continual-learning methods address the preservation and adaptation of previously acquired knowledge as new tasks or observations arrive~\cite{schwarz2018progress}. In distributed environments, these challenges are compounded by communication constraints and possible differences between local data streams~\cite{marfoq2023federated,liu2023dynamite}. Nevertheless, the objectives considered in these lines of research do not directly characterize the statistical difference between the empirical and population losses of the final aggregated model.

A further complication arises when observations within each client follow a Markov process. Optimization under Markovian sampling has received considerable attention, including analyses of stochastic approximation with dependent observations~\cite{duchi2011ergodic,sun2018markov} and more recent studies of distributed and streaming optimization~\cite{huynh2025streaming}. In particular, the latter investigates minibatch and local SGD under Markovian sampling and establishes convergence guarantees that reflect the effect of temporal dependence. These results complement the broader literature on stochastic optimization with Markovian noise, surveyed in~\cite{beznosikov2023first}. Their main concern, however, is how dependence affects convergence toward an optimizer. Our interest is different: we study how the dependence of the learned model on the observed trajectories determines its expected generalization gap, even when the local learning algorithms are not specified through a particular gradient-update rule.

The statistical analysis of dependent observations has a long history. Classical work on concentration and empirical processes developed tools for handling weakly dependent and mixing sequences~\cite{yu1994rates,marton1996bounding,samson2000concentration}. Subsequent results established concentration inequalities for Markov chains and dependent random variables, making it possible to account explicitly for the strength and persistence of temporal correlations~\cite{kontorovich2008concentration,paulin2015concentration}. Generalization bounds beyond independent sampling have also been developed using blocking arguments, mixing coefficients, and complexity-based techniques~\cite{mohri2009stability,ralaivola2010chromatic,kuznetsov2017generalization}. These methods show that temporal dependence can substantially affect the concentration of empirical averages, and that the effective number of observations may be considerably smaller than the nominal sample size. In our setting, the difficulty is further compounded by next-token prediction, where overlapping contexts introduce additional sensitivity of the loss to individual tokens. Our analysis accounts for this effect through bounded-difference arguments combined with concentration inequalities for Markovian trajectories.

A complementary perspective comes from information-theoretic generalization. The foundational works of~\cite{russo2016controlling,xu2017information} relate the expected generalization gap to the information that the learned model retains about its training data. Refinements based on individual-sample information, conditional information, and data-dependent quantities have subsequently provided more detailed descriptions of algorithmic dependence~\cite{bu2020tightening,negrea2019information,steinke2020reasoning}. Unlike uniform complexity measures, these approaches characterize generalization through properties of the learning algorithm and the resulting joint distribution of the training data and the learned parameters. They are therefore particularly attractive for distributed learning, where the sequence of local updates and server aggregations creates a structured flow of information across clients and communication rounds.

Information-theoretic approaches are also closely related to model compression and rate-distortion theory. Rather than measuring the information required to specify the learned model exactly, a lossy representation can be used when its effect on the generalization gap is controlled~\cite{Sefidgaran2022,sefidgaran2022rate,arora2018stronger}. This distinction is important for deterministic learning algorithms with continuous model spaces, for which ordinary mutual information may become infinite and yield vacuous guarantees. Rate-distortion formulations instead quantify the information needed to preserve the aspects of the model that matter for generalization. In the federated setting considered here, this viewpoint allows us to study approximate local updates while retaining the dependence of the final aggregated model on earlier communication rounds.

The work of~\cite{sefidgaran2024lessons} is particularly relevant, as it studies the generalization gap of streaming federated learning through an information-theoretic framework under i.i.d.\ data streams. Its analysis exploits independence across observations to relate generalization to the information conveyed by local models and successive communication rounds. Such an argument cannot be transferred directly to Markovian streams because observations within each trajectory remain statistically dependent, and replacing one sample may affect the distribution of subsequent observations. The distinction is especially pronounced in next-token prediction, where neighboring loss terms also share parts of their input contexts. Consequently, extending the i.i.d.\ analysis requires concentration tools that account for temporal dependence rather than treating every observed token as an independent sample.

Finally, next-token prediction and self-attention have generated a substantial body of theoretical work. The Transformer architecture introduced attention-based sequence modeling as a flexible alternative to recurrent processing~\cite{vaswani2017attention}, while subsequent studies investigated the inductive biases and representational properties of self-attention~\cite{edelman2022inductive}. Other analyses have considered the learning capabilities and statistical behavior of autoregressive prediction~\cite{malach2023auto,yuksel2025sample}. These contributions provide important foundations for understanding sequence models, but they do not directly address the combined effects of decentralized training, communication rounds, and Markovian dependence on the generalization gap.

Our work connects these complementary research directions. We extend information-theoretic and rate-distortion generalization analysis to streaming federated learning with Markovian trajectories, allowing the mixing behavior to vary across clients. By specializing the results to linear and self-attention next-token predictors, we further identify how communication frequency, client count, temporal memory, and prediction context jointly enter the generalization guarantees. The resulting framework thus addresses a statistical aspect of distributed sequence learning that is distinct from the optimization guarantees developed for Markovian SGD and from the generalization results established under independent data streams.


\section{Notation and Problem Setup}

\subsection{Notation}\label{Notation}

Upper-case letters denote random variables, e.g., $X$, while lower-case
letters denote their realizations, e.g., $x$. Calligraphic letters denote
sets or spaces, e.g., $\mathcal{X}$ and $\mathcal{W}$. The probability
distribution of a random variable $X$ is denoted by $P_X$, and
$P_{X|Y}$ denotes the corresponding conditional distribution. Expectations
with respect to a distribution $P$ are written as $\mathbb{E}_{P}[\cdot]$,
with the subscript omitted when the underlying distribution is clear from
the context.

For two probability measures $P$ and $Q$ defined on a common measurable
space, with $P\ll Q$, their Kullback--Leibler (KL) divergence is defined as
\[
D(P\|Q)
\coloneqq
\mathbb{E}_{P}\!\left[
\log\!\left(\frac{dP}{dQ}\right)
\right],
\]
and $D(P\|Q)=\infty$ whenever $P$ is not absolutely continuous with
respect to $Q$. The mutual information between random variables $X$ and
$Y$ is
\[
I(X;Y)
\coloneqq
D\!\left(P_{X,Y}\middle\|P_XP_Y\right).
\]

For $n\in\mathbb{N}$, we use $[n]\coloneqq\{1,\ldots,n\}$, and
$\mathbf{1}_{\{\cdot\}}$ denotes the indicator function. The Euclidean
inner product is denoted by $\langle\cdot,\cdot\rangle$. For vectors,
$\|\cdot\|_\infty$ denotes the maximum norm, while for matrices
$\|\cdot\|$ denotes the Frobenius norm unless stated otherwise.

Finally, for a mapping $g:\mathcal{A}\to\mathcal{B}$ and a positive integer $m$, we denote its element-wise extension by $g^{\odot m}$, defined as $g^{\odot m}(a_1,\ldots,a_m)
\coloneqq
(g(a_1),\ldots,g(a_m)),$
for every $(a_1,\ldots,a_m)\in\mathcal{A}^m$.

\subsection{Data Generation Process}

Consider a finite dictionary of tokens
$\mathcal{D}\coloneqq\{\varnothing,t_1,\ldots,t_d\}$,
where $d\in\mathbb{N}^*$. We represent the tokens
in $\mathbb{R}^d$ using the one-hot encoding map
$\mathcal{E}$, defined as
\begin{equation}
    \mathcal{E}:
    \begin{cases}
        \mathcal{D}\to\mathbb{R}^d,\\
        \varnothing\mapsto 0,\\
        t_i\mapsto e_i,
    \end{cases}
    \qquad
    \text{where $\{e_i\}_{i=1}^d$ is the canonical basis of $\mathbb{R}^d$.}
\end{equation}

We consider a streaming federated learning (SFL)
setting in which clients continuously collect
data and learn from independent trajectories
generated by non-stationary stochastic Markov
processes with finite memory. More precisely,
for a trajectory length $T\in\mathbb{N}^*$ and
a context length $\rho$, let
$\mathbf{Z}=(Z_1,\ldots,Z_T)\in\mathcal{D}^T$
denote a trajectory drawn according to a
distribution $\mu$. The observations within
each trajectory evolve according to
\begin{align}
Z_t \sim \mu\left(
Z_t\mid Z_{t-1},\ldots,Z_{t-\rho}
\right),
\qquad t=2,\ldots,T.
\label{markov_chain_generated_data_1}
\end{align}

Our objective is to study collaboratively
learned next-token predictors that operate
on the embeddings of the preceding $\rho$
tokens and produce a vector of prediction
logits. To this end, for a parameter
$w\in\mathcal{W}$, we consider the class
of mappings
\begin{align}
\mathcal{G}_{\mathcal{W}}
\coloneqq
\left\{
g_w:\mathbb{R}^{\rho d}\to\mathbb{R}^d
\,\middle|\,
w\in\mathcal{W}
\right\}.
\end{align}

Given a context $\vec{z}\in\mathcal{D}^{\rho}$,
a model $w\in\mathcal{W}$ assigns a probability
to the next token $z\in\mathcal{D}$ through
\begin{align}
p_w(z\mid\vec{z})
=
\big\langle
\sigma\big(
g_w(\mathcal{E}^{\odot\rho}(\vec{z}))
\big),
\mathcal{E}(z)
\big\rangle,
\label{conditional_distribution_def}
\end{align}
where
$\sigma(\nu)=(\sigma(\nu)_1,\ldots,\sigma(\nu)_d)$
denotes the softmax mapping, given by
\begin{equation}
\sigma(\nu)_i
=
\frac{\exp(\nu_i)}
{\sum_{j\in[d]}\exp(\nu_j)},
\qquad i\in[d].
\label{soft_max_def}
\end{equation}

Throughout the analysis, we impose the
following boundedness condition on the
outputs of the prediction models.

\begin{assumption}[Boundedness of the Model Output]
\label{subsubsec:bounded_output}
There exists a constant $B>1$ such that,
for every $w\in\mathcal{W}$ and
$\nu\in\mathcal{D}^{\rho}$, $\left\|
\left(g_w\circ\mathcal{E}^{\odot\rho}\right)(\nu)
\right\|_{\infty}
\leq B.$
\end{assumption}
\begin{assumption}[Boundedness of the Model Parameters]\label{Bound_size_model}
The random parameter $W$ is assumed to be almost surely bounded. That is, there exists a constant $C>1$ such that $\mathbb{P}\left(\|W\|\leq C\right)=1$.
\end{assumption}


\subsection{Streaming FL}
\label{Streaming_FL_Sec}

In conventional federated learning, clients train on
pre-collected datasets that remain available throughout
the learning process. Consequently, each client can
repeatedly access its local observations across
communication rounds. Streaming federated learning
(SFL), by contrast, considers settings in which clients
continuously acquire new data and may have limited
capacity to store previously observed samples. The
data available for local training therefore changes
from one communication round to the next.

We consider an SFL system comprising $K$ clients
that communicate with a common server over $R$
rounds. During each round $r\in[R]$, client $k$
collects $N_R$ independent trajectories, each
containing $T$ observations generated by an
underlying non-stationary Markov process with
finite memory. We denote the $n$-th trajectory
collected by client $k$ in round $r$ by
\begin{align}
S_{k,r}^{(n)}
&=
\big\{
X_{k,r}^{(n)}[1],\ldots,
X_{k,r}^{(n)}[T]
\big\},
\qquad n\in[N_R].
\end{align}
The collection of trajectories available to
client $k$ during round $r$ is denoted by
$S_{k,r}$, while $S_k$ represents all
trajectories collected by that client
throughout the learning process. Thus,
\begin{align}
S_{k,r}
&=
\big\{
S_{k,r}^{(1)},\ldots,
S_{k,r}^{(N_R)}
\big\},
\qquad
S_k
=
\bigcup_{r\in[R]}S_{k,r}.
\end{align}

Training proceeds through successive rounds
of local learning and server aggregation.
At the beginning of round $r$, each client
receives the aggregated model
$\overline{W}^{(r-1)}$ from the server and
uses it together with its newly collected
data to update its local model. Let
$\mathcal{A}_k$ denote the possibly
stochastic learning algorithm employed
by client $k$. The resulting local model is
\begin{align}
W_k^{(r)}
=
\mathcal{A}_k
\big(
S_{k,r},\overline{W}^{(r-1)}
\big).
\end{align}
After receiving the local models from
all participating clients, the server
computes their arithmetic average
according to the Federated Averaging
(FedAvg) rule,
\begin{align}
\overline{W}^{(r)}
=
\frac{1}{K}
\sum_{k=1}^{K}W_k^{(r)},
\end{align}
and broadcasts the updated global model
to the clients for the next round. This
procedure is repeated until all $R$
communication rounds have been completed.

\subsection{Generalization Gap}

Our analysis focuses on the generalization
gap of the final aggregated model
$\overline{W}^{(R)}$. We denote the complete
collection of training trajectories and
the sequence of local and aggregated
models generated during learning by
\begin{align}
\mathbf{S}
&=
\bigcup_{k\in[K]}
\bigcup_{r\in[R]}S_{k,r},
\nonumber\\
\mathbf{W}
&=
\Big\{
W_k^{(r)},\overline{W}^{(r)}
\Big\}_{k\in[K],\,r\in[R]}.
\end{align}

Since the local learning algorithms may
be stochastic, the model produced by
client $k$ in round $r$ is characterized
by the conditional distribution
$P_{W_k^{(r)}\mid
S_{k,r},\overline{W}^{(r-1)}}$.
Together with the server aggregation
rule, these conditional distributions
determine the law of the entire sequence
of learned models given the training data.
In particular, the conditional distribution
factorizes as
\begin{align}
P_{\mathbf{W}\mid\mathbf{S}}
=
\prod_{r\in[R]}
\Bigg(
P_{\overline{W}^{(r)}
\mid W_{[K]}^{(r)}}
\prod_{k\in[K]}
P_{W_k^{(r)}
\mid S_{k,r},\overline{W}^{(r-1)}}
\Bigg),
\end{align}
where $W_{[K]}^{(r)}$ denotes the
collection of local models generated
by all clients during round $r$.
Under the assumed independence of
the data collected across clients
and communication rounds, the joint
distribution of the training data
and learned models admits the
corresponding factorization
\begin{align}
P_{\mathbf{S},\mathbf{W}}
=
\prod_{r\in[R]}
\Bigg(
P_{\overline{W}^{(r)}
\mid W_{[K]}^{(r)}}
\prod_{k\in[K]}
P_{S_{k,r}}
\otimes
P_{W_k^{(r)}
\mid S_{k,r},\overline{W}^{(r-1)}}
\Bigg).
\end{align}

The clients are trained for next-token
prediction using the standard
cross-entropy loss. For an aggregated
model $\overline{w}^{(R)}\in\mathcal{W}$,
we define the loss evaluated along a
trajectory $s_{k,r}^{(n)}$ as
\begin{align}
\ell\big(
s_{k,r}^{(n)},\overline{w}^{(R)}
\big)
=
-\frac{1}{T}
\sum_{t\in[T]}
\log p_{\overline{w}^{(R)}}
\big(
x_{k,r,t}^{(n)}
\mid
x_{k,r,t-1}^{(n)},
\ldots,
x_{k,r,t-\rho}^{(n)}
\big),
\label{loss_function_one_sample}
\end{align}
where the conditional probability
$p_{\overline{w}^{(R)}}$ is defined
in~\eqref{conditional_distribution_def}.
Thus, the loss measures the average
negative log-likelihood assigned by
the final aggregated model to the
observed tokens along a trajectory.

Averaging the trajectory-level loss
over all clients and communication
rounds gives the empirical risk
\begin{align}
\hat{\mathcal{L}}
\big(
\mathbf{S},\overline{w}^{(R)}
\big)
=
\frac{1}{KN}
\sum_{k\in[K]}
\sum_{r\in[R]}
\sum_{n\in[N_R]}
\ell\big(
s_{k,r}^{(n)},\overline{w}^{(R)}
\big),
\end{align}
where $N=RN_R$ denotes the total
number of trajectories collected
by each client over the $R$ rounds.
The corresponding population risk
is defined by averaging the expected
loss over the client-specific
data-generating distributions:
\begin{align}
\mathcal{L}
\big(\overline{w}^{(R)}\big)
=
\frac{1}{K}
\sum_{k\in[K]}
\mathbb{E}_{S_k^{(n)}\sim\mu_k}
\left[
\ell\big(
S_k^{(n)},\overline{w}^{(R)}
\big)
\right].
\label{population_risk_acroos_client}
\end{align}
Here, $\mu_k$ denotes the distribution
of a fresh trajectory generated by
the underlying Markov process at
client $k$. The generalization gap
of the final aggregated model is
then given by the difference between
its population and empirical risks:
\begin{align}
\gen\big(
\mathbf{S},\overline{w}^{(R)}
\big)
=
\mathcal{L}
\big(\overline{w}^{(R)}\big)
-
\hat{\mathcal{L}}
\big(
\mathbf{S},\overline{w}^{(R)}
\big).
\label{generalization_croos_entropy_def}
\end{align}

\section{Bounds on the Generalization Gap}
\label{headings} \vspace{-0.1cm}

As discussed in the related work, extending
generalization arguments developed for i.i.d.\
data, such as those in~\cite{sefidgaran2024lessons},
to streaming federated learning with Markovian
observations introduces an additional source
of statistical dependence. In the i.i.d.\
setting, dependence arises primarily through
the evolution of the learned models, whereas
in the Markovian setting, observations within
each local trajectory are themselves temporally
dependent. These two forms of dependence must
therefore be accounted for simultaneously.
In particular, replacing an individual training
observation with an independent sample, as in
standard ghost-sample or stability arguments,
does not generally preserve the dependence
structure of the surrounding trajectory.
To address this difficulty, our analysis
relies on concentration inequalities for
dependent sequences, which quantify the loss
of effective statistical information caused
by temporal correlations through the mixing
properties of the underlying Markov processes.

\begin{definition}[Mixing time]
Let $\mathbf{X}=(X_1,\ldots,X_T)$ be a
Markov chain taking values in the product
of Polish spaces
$\mathcal{X}_1\times\cdots\times\mathcal{X}_T$,
with $X_i\in\mathcal{X}_i$ for each
$i\in[T]$. For any $\epsilon\in(0,1)$,
the mixing time $\tau(\epsilon)$ is the
smallest number of steps required for
the conditional distributions of the
chain, starting from any two admissible
states, to be within $\epsilon$ in total
variation distance, uniformly over all
valid starting times. More precisely, let
\begin{align}
\overline{d}(t)
&\coloneqq
\max_{1\leq i\leq T-t}
\sup_{x,y\in\mathcal{X}}
d_{\mathrm{TV}}\big(
\mathbb{P}(X_{i+t}\mid X_i=x),
\mathbb{P}(X_{i+t}\mid X_i=y)
\big).
\end{align}
The mixing time and its associated
mixing-time parameter are then defined by
\begin{align}
\tau(\epsilon)
&\coloneqq
\min\big\{
t\in\mathbb{N}:
\overline{d}(t)\leq\epsilon
\big\},
\nonumber\\
\tau_{\min}
&\coloneqq
\inf_{0\leq\delta<1}
\tau(\delta)
\left(
\frac{2-\delta}{1-\delta}
\right)^2.
\label{mixing_time_mcdiarmid}
\end{align}
\end{definition}

\subsection{Loss-less generalization bound}

\begin{theorem}
\label{multi_round_markovian_general_bound}
Under Assumptions~\ref{subsubsec:bounded_output}
and~\ref{Bound_size_model}, the expected
generalization gap of the model obtained
by SFL after $R\geq 1$ communication rounds
under cross-entropy loss satisfies
\begin{equation}
\mathbb{E}\left[
\gen(\mathbf{S},\overline{W}^{(R)})
\right]
\leq
\frac{\left(2B(1+2\rho)+\log d\right)}{KR}
\sum_{k\in[K]}\sum_{r\in[R]}
\sqrt{
\frac{
\tau_{k,\min}R\,
I\left(
S_{k,r};W_k^{(r)}
\mid\overline{W}^{(r-1)}
\right)
}{2NT}
}.
\end{equation}
\end{theorem}

A complete and detailed proof is provided in
Appendix~\ref{Proof_multi_round}. Here, we present
a proof sketch highlighting the main ideas.

\emph{Proof sketch.}
Fix $\lambda_{k,r}>0$ for each client $k$ and
communication round $r$, and let $S'_{k,r}$ be
an independent copy of $S_{k,r}$. We first
decompose the expected generalization gap
across clients and rounds. Applying the
Donsker--Varadhan variational inequality
to each local dataset, and using the
conditional independence structure of
the federated learning process, yields
\begin{align}
\Psi_{k,r}(\lambda)
&\coloneqq
\sum_{n=1}^{N_R}
\mathbb{E}_{\overline{W}^{(R)}}
\left[\log \mathbb{E}\left[
\exp\left(
\frac{\lambda}{N_R}
\Delta_{k,r}^{(n)}
\right)\right]\right],
\end{align}
where
\begin{align}
\Delta_{k,r}^{(n)}
&\coloneqq
\mathbb{E}_{S_{k,r}'^{(n)}}
\left[\ell(S_{k,r}'^{(n)},\overline{W}^{(R)})\right]-
\ell(S_{k,r}^{(n)},\overline{W}^{(R)}).
\end{align}
Then, the Donsker--Varadhan inequality yields
\begin{align}
&\mathbb{E}[\gen(\mathbf{S},\overline{W}^{(R)})]
\leq
\frac{1}{KR}
\sum_{k=1}^{K}\sum_{r=1}^{R}
\frac{
I(S_{k,r};W_k^{(r)}|\overline{W}^{(r-1)})+\Psi_{k,r}(\lambda_{k,r})
}{\lambda_{k,r}}.
\end{align}
where the decomposition across $n$ follows
from the independence of the trajectories
collected by each client during each round.
In particular, the mutual-information term
follows from the Markov structure of the
learning process and the independence of
$S_{k,r}$ from $\overline{W}^{(r-1)}$.

It remains to control the moment-generating
function of the loss along each Markov
trajectory. For a fixed aggregated model
$\overline{w}$, we compare the losses
associated with two trajectories by
successively replacing their tokens.
Since the prediction at time $t$ depends
only on the preceding $\rho$ tokens,
changing one token affects its own
prediction term and at most $\rho$
subsequent context windows. The resulting
telescoping argument gives
\[
\ell(s_{k,r}^{(n)},\overline{w})
-\ell(s_{k,r}'^{(n)},\overline{w})
\leq
\frac{(C_1+\rho C_2)}{T}
\sum_{t=1}^{T}
\mathbf{1}\{x_{k,r,t}^{(n)}
\neq x_{k,r,t}'^{(n)}\},
\]
where $C_1$ controls the change in the
predicted token for a fixed context,
while $C_2$ controls the effect of changing
one token in the context. The bounded-output
assumption yields
\[
C_1\leq 2B+\log d,
\qquad C_2\leq 4B,
\]
and hence the bounded-difference
coefficients can be chosen as
\[
c_t\leq
\frac{(2B(1+2\rho)+\log d)}{T},
\qquad t\in[T].
\]

Applying the McDiarmid-type concentration
inequality for Markov chains separately
to each client's trajectories bounds the
corresponding logarithmic moment-generating
functions in terms of $\tau_{k,\min}$.
Substituting these bounds into the preceding
information-theoretic inequality gives

\begin{align}
\mathbb{E}[\gen(\mathbf{S},\overline{W}^{(R)})]
&\leq \frac{1}{KR}
\sum_{k=1}^{K}\sum_{r=1}^{R}
\Bigg[
\frac{
I\left(S_{k,r};W_k^{(r)}|\overline{W}^{(r-1)}\right)
}{\lambda_{k,r}}
+
\frac{
\lambda_{k,r}\:\tau_{k,\min}
\big(2B(1+2\rho)+\log d\big)^2
}{8N_RT}
\Bigg].
\label{eq:fed-proof-sketch}
\end{align}
Optimizing over $\lambda_{k,r}>0$
and using $N_R=N/R$ yields the desired bound.

The bound offers several insights. (i) It reveals a
tradeoff between optimization and generalization:
more frequent communication can increase the
information retained from local data streams,
potentially weakening the generalization guarantee.
(ii) The effect of temporal dependence is captured
by the mixing times across clients, rather than
being determined solely by the slowest-mixing
process. (iii) Slowly mixing data lead to weaker
generalization guarantees, with a penalty that
scales with the square root of the mixing time,
reflecting the reduced effective sample size.
(iv) Increasing the per-client data budget $N$
improves the bound, although this benefit remains
limited by temporal dependence and communication
frequency.

\begin{remark}[i.i.d.\ specialization]
When the observations are i.i.d.\ across clients,
rounds, and time (i.e., $\rho=0$), the cross-entropy
loss is bounded by $(2B+\log d)$ and is therefore
$(2B+\log d)/2$-sub-Gaussian by Hoeffding's lemma.
The corresponding generalization bound becomes
\begin{align*}
\mathbb{E}[\gen(\mathbf{S},\overline{W}^{(R)})]
\leq
\frac{(2B+\log d)}{KR}
\sum_{k\in[K]}\sum_{r\in[R]}
\sqrt{
\frac{
I(S_{k,r};W_k^{(r)}
|\overline{W}^{(r-1)})
}{2N_RT}
}.
\end{align*}
This bound has a structure similar to that of
Theorem~3.1 in~\cite{sefidgaran2024lessons},
although the latter considers a different
loss setting and provides PAC-Bayesian guarantees.
\end{remark}

\begin{lemma}\label{Hoeffdinf_Lemma}
Under Assumptions \ref{subsubsec:bounded_output}-\ref{Bound_size_model},
let $Z_{k,r,t}:=(X_{k,r,t},X_{k,r,t-1},\ldots,X_{k,r,t-\rho+1})$, and let
$\tilde{\tau}_{\min}$ be the mixing-time quantity associated with
the history-state process $\{Z_{k,r,t}\}_{k\in[K],r\in[R],t\in[T]}$. Then, the expected
generalization gap of the model learned in SL under the
cross-entropy loss satisfies
\begin{align}
    \mathbb{E}\left[\gen(\mathbf{S},\overline{W}^{(R)})\right]
    \leq\frac{\left(2B+\log d\right)}{KR}\sum_{k\in[K]}\sum_{r\in[R]}
    \sqrt{
        \frac{\tilde{\tau}_{k,\min} I(\mathbf{S};W^{(r)}_k|\overline{W}^{(r-1)})}
        {2NT}
    },
    \nonumber
\end{align}
where
\begin{equation}
\tilde{\tau}_{k,\min}
:=
\inf_{0<\delta<1}
\tilde{\tau}(\delta)
\left(\frac{2-\delta}{1-\delta}\right)^2.\label{new_tau-min}
\end{equation}
In particular, the resulting bound is tighter than that of
Theorem~\ref{multi_round_markovian_general_bound} whenever
\begin{align}
    \tilde{\tau}_{k,\min}
    <
    \left(
        \frac{2B(1+2\rho)+\log d}
        {2B+\log d}
    \right)^2
    \tau_{k,\min},
    \nonumber
\end{align}
where $\tau_{\min}$ is defined in
\eqref{eq:mixing_time_mcdiarmid}. The proof is provided in Appendix~\ref{Proof_Hoeffding_employing}.
\end{lemma}






\subsection{Lossy Generalization Bound}

We now develop a lossy refinement of
Theorem~\ref{multi_round_markovian_general_bound},
following the rate-distortion approach
of~\cite{Sefidgaran2022,sefidgaran2022rate}.
Rather than requiring an exact representation
of the learned model updates, this approach
allows controlled approximation errors and
measures the information needed to preserve
their generalization behavior. This is
particularly useful for deterministic learning
algorithms over continuous hypothesis spaces,
where the mutual-information terms in the
lossless bound may be infinite.

Fix $\epsilon>0$. For each client $k$ and
communication round $r$, define the
rate-distortion function
\begin{align}
R_{\mathcal{D}}(k,r,\epsilon_k)
\coloneqq
\inf
I\Big(S_{k,r};
\hat{\overline{W}}_k^{(R,r)}|\overline{W}^{(r-1)},
W_{[K]\setminus k}^{(r)},U
\Big),
\label{general_bound_lossy_mutual_1}
\end{align}

where the infimum is taken over all conditional
Markov kernels satisfying the distortion constraint
\begin{align}
\mathbb{E}\Big[
\gen(S_{k,r},\overline{W}^{(R)})
-
\gen(S_{k,r},\hat{\overline{W}}_k^{(R,r)})
\Big]
\leq \epsilon_{k,r},
\label{def:Rdist}
\end{align}
for a fixed collection of non-negative distortion
tolerances $\{\epsilon_{k,r}\}_{k,r}$ satisfying
\begin{align}
\frac{1}{KR}\sum_{k\in[K]}\sum_{r\in[R]}
\epsilon_{k,r}\leq \epsilon.
\end{align}

The mutual information is evaluated under
the distribution induced by
\[
P_U P_{S_{k,r}}
P_{\overline{W}^{(r-1)}}
P_{\hat{W}_k^{(r)}
\mid S_{k,r},\overline{W}^{(r-1)},U},
\]
while the expectation in~\eqref{def:Rdist}
is taken with respect to the corresponding
joint distribution, including
\[
P_{\hat{\overline{W}}_k^{(R,r)}
\mid W_{[K]\setminus k}^{(r)},
\hat{W}_k^{(r)},U}.
\]

\begin{theorem}
\label{lossy_inexpectation_bound_federated}
Under Assumptions~\ref{subsubsec:bounded_output}
and~\ref{Bound_size_model}, the expected
generalization gap of the model obtained
by SFL after $R\geq1$ communication rounds
under cross-entropy loss satisfies
\begin{align}
\mathbb{E}\left[
\gen(\mathbf{S},\overline{W}^{(R)})
\right]
\leq
\frac{\left(2B(1+2\rho)+\log d\right)}{KR}
\sum_{k\in[K]}\sum_{r\in[R]}
\sqrt{
\frac{\tau_{k,\min}R}{2NT}
R_{\mathcal{D}}(k,r,\epsilon_k)
}
+\epsilon,
\label{distortion_federated_condition_federated}
\end{align}
where $R_{\mathcal{D}}(k,r,\epsilon_k)$
is defined in~\eqref{general_bound_lossy_mutual_1}.
\end{theorem}

The proof is provided in
Appendix~\ref{proof_lossy_case_federated}.
The benefit of this refinement is that
the additional distortion term $\epsilon$
can be offset by a sufficiently large
reduction in the information rate appearing
under the square root. As shown in the next
section, this tradeoff is particularly
useful for LNTP-SFL and SANTP-SFL models,
where suitable approximations of the
learned parameters yield explicit
generalization bounds in terms of the
system parameters that are not directly
available from the lossless result.

\vspace{-0.2cm}

\section{Generalization bound on Next token prediction problem }
\label{Next_token_prediction_section}

\vspace{-0.2cm}

\subsection{Margin-Based Generalization Gap}

\vspace{-0.2cm}

Next-token prediction is based on the cross-entropy loss, which, by selecting the token that is most likely given the past $\rho$ tokens, essentially seeks to maximize the likelihood gap between the most likely and second most likely tokens. As such, it bears some similarity to classification problems: under supervised fine-tuning, each context is paired with a single target token, and greedy decoding selects the arg-max over the vocabulary. Minimizing the loss to within a margin is known to enable better rates, especially in classification problems~\cite{NIPS2017_b22b257a,gronlund2020}, where it also leads to more efficient algorithms. The corresponding distinction is that, under the cross-entropy loss, the algorithm continues to push the likelihoods further apart across the iterates, whereas, under the margin loss, it stops doing so once the margin is met. It is precisely for this reason that a lower sample complexity is required for the margin loss to achieve a desired accuracy. In this section, we consider a margin-based loss obtained by incorporating a margin parameter $\theta>0$. 

For an aggregated model $\overline{w}^{(R)}$ and data sequence $s^{(n)}_k$, let
\begin{equation*}
\eta(\overline{w}^{(R)},t,\vec{x}^{(n)}_k)\coloneqq\log\frac{C_{1}(\overline{w}^{(R)},t,\vec{x}^{(n)}_k)}{C_{2}(\overline{w}^{(R)},t,\vec{x}^{(n)}_k)},\qquad\ell_{\theta,k}(s^{(n)}_k, \overline{w}^{(R)})
    \coloneqq \frac{1}{T}\sum_{t \in [T]}
    \mathbf{1}\left\{\eta(\overline{w}^{(R)},t,\vec{x}^{(n)}_k) < \theta \right\},
\end{equation*}
where
\begin{align}
C_{1}(\overline{w}^{(R)},t,\vec{x}^{(n)}_k) &= 
p_{\overline{w}^{(R)}}\!\left(x^{(n)}_{k,t}|x^{(n)}_{k,t-1}, \dots, x^{(n)}_{k,t-\rho}\right),\label{constants_empirical_risk_LNTP_1}\\
C_{2}(\overline{w}^{(R)},t,\vec{x}^{(n)}_k) &= \max_{x'^{(n)}_{k,t} \neq x^{(n)}_{k,t}}
p_{\overline{w}^{(R)}}\!\left({x}'^{(n)}_{k,t}|x^{(n)}_{k,t-1}, \dots, x^{(n)}_{k,t-\rho}\right).\label{constants_empirical_risk_LNTP_2}
\end{align}

The population and empirical risks are evaluated respectively as
\begin{align}
    \mathcal{L}_{\theta}(\overline{w}^{(R)})
    \coloneqq \frac{1}{KT}\sum\nolimits_{k\in[K]} \sum\nolimits_{t \in [T]}
    \mathbb{E}\left[
        \mathbf{1}\left\{\eta(\overline{w}^{(R)},t,\vec{X}_k)  < {\theta} \right\}
    \right],
    \label{Margin_population_crossentropy_risk}
\end{align}
and
\begin{align}
    \hat{\mathcal{L}}_{\theta}(\mathbf{s},\overline{w}^{(R)})
    \coloneqq \frac{1}{NTK} \sum\nolimits_{k\in[K]} \sum\nolimits_{n \in [N]} \left[\sum\nolimits_{t \in [T]}
    \mathbf{1}\left\{\eta(\overline{w}^{(R)},t,\vec{x}^{(n)}_k) < \theta \right\}\right].
\label{Margin_empirical_crossentropy_risk}
\end{align}
The margin generalization gap induced by aggregated model $\overline{w}^{(R)}$ obtained after $R$ rounds of SFL  is 
\begin{align}
\operatorname{gen}_{\theta} (\mathbf{s}, \overline{w}^{(R)})
    =  \mathcal{L}_{-\theta/2}(\overline{w}^{(R)})
      - \hat{\mathcal{L}}_{\theta/2}(\mathbf{s},\overline{w}^{(R)}).
\label{margin_theta_generalization_error}
\end{align}


\begin{theorem}\label{Extension_Lemma_1_Margin_generalization}
For any $\theta\in\mathbb{R}$, the margin loss defined in
\eqref{margin_theta_generalization_error} satisfies the
bounded-difference condition in \eqref{Mdiarmid_type_lemma_1}.
In particular, for any two sequences
$s_{k,r}=\{x_{k,r,1},\ldots,x_{k,r,T}\}$ and
${s}'_{k,r}=\{x'_{k,r,1},\ldots,x'_{k,r,T}\}$,
\begin{align}
    \ell_{\theta}(s_{k,r},w)
    -
    \ell_{\theta}({s}'_{k,r},w)
    &\leq
    \frac{\rho+1}{T}
    \sum_{t=1}^{T}
    \mathrm{1}\left\{x_{k,r,t}\neq x'_{k,r,t}\right\},
    \nonumber\\
    \ell_{\theta}(s_{k,r},w)
    -
    \ell_{\theta}(s'_{k,r},w)
    &\leq
    \frac{1}{T}
    \sum_{t=1}^{T}
    \mathrm{1}\left\{
        \eta(w,t,{x_{k,r}})
        \neq
        \eta(w,t,{x}'_{k,r})
    \right\}.
    \nonumber
\end{align}

Consequently, Theorem~\ref{multi_round_markovian_general_bound}
yields the following expected generalization bound under the
margin loss:
\begin{align}
    \mathbb{E}\left[\gen_{\theta}(S_{[K]},\overline{W}^{(R)})\right]
    &\leq\frac{(\rho+1)}{KR}\sum_{k\in[K]}\sum_{r\in[R]}
    \sqrt{
        \frac{
            \tau_{k,\min}\: R\: I(\mathbf{S}_{k,r};W^{(r)}_k|\overline{W}^{(r-1)})
        }{
            2NT
        }
    },
    \label{Margin_generalization_bound_tau_out}
\end{align}
In addition, Lemma~\ref{Hoeffdinf_Lemma} yields
the refined bound
\begin{align}
    \mathbb{E}\left[\gen_{\theta}(\mathbf{S},W)\right]
    &\leq\frac{1}{KR}\sum_{k\in[K]}\sum_{r\in[R]}
    \sqrt{
        \frac{
            \tau_{k,\min}\: R\: I(\mathbf{S}_{k,r};W^{(r)}_k|\overline{W}^{(r-1)})
        }{
            2NT
        }
    },
    \label{Margin_generalization_bound_tau_in}
\end{align}
where $\tilde{\tau}_{\min}$ is defined in \eqref{new_tau-min}.

The refined bound removes the explicit factor $(\rho+1)$ and is
strictly tighter than \eqref{Margin_generalization_bound_tau_out}
whenever
\begin{align}
    \tilde{\tau}_{\min}<(\rho+1)^2\tau_{\min}.
    \nonumber
\end{align}
\end{theorem}

The proof follows directly from the arguments used in the proof of Theorem~\ref{multi_round_markovian_general_bound} and is therefore omitted.

\begin{remark}
Just as Theorem~\ref{lossy_inexpectation_bound_federated}
provides a lossy counterpart to
Theorem~\ref{multi_round_markovian_general_bound},
Theorem~\ref{Extension_Lemma_1_Margin_generalization}
admits a corresponding lossy extension through the same
approach. Since the argument follows directly from the
preceding analysis, we omit its proof. This extension
will serve as a key tool in deriving the generalization
bounds for distributed linear and single-attention
next-token prediction in
Theorems~\ref{Theorem_Multi_round_LNTP}
and~\ref{Theorem_Multi_round_SANTP}, respectively.
\end{remark}

\subsection{Distributed Linear Next-Token Prediction}

We consider a distributed extension of the centralized linear
next-token prediction (LNTP) model studied in~\cite{yuksel2025sample},
closely related to the linear autoregressive models
of~\cite{malach2023auto}. For each client $k\in[K]$ and $C>0$,
let
$\mathcal{M}_{C,k}^{d_1\times d_2}
\coloneqq \{M_k\in\mathbb{R}^{d_1\times d_2}:
\|M_k\|\leq C\}$
denote the class of matrices with Frobenius norm bounded by $C$.
The parameter space of client $k$ is then defined as
\begin{equation}
\mathcal{W}_k
\coloneqq
\left\{
\mathbf{M}_k=(M_{k,1},\ldots,M_{k,\rho})
\in\mathcal{M}_{C,k}^{d\times\rho d}
\right\}.
\label{Model_multiple_matrices_tau}
\end{equation}
The corresponding class of logit mappings is
$\mathcal{G}_{\mathcal{W}_k}
\coloneqq\{g_{\mathbf{M}_k}:\mathbf{M}_k\in\mathcal{W}_k\}$,
where, for an input context
$\vec{x}_k=(x_{k,t-1},\ldots,x_{k,t-\rho})
\in\mathbb{R}^{\rho d}$, the mapping is given by
\[
g_{\mathbf{M}_k}(\vec{x}_k)
\coloneqq
\sum_{i=1}^{\rho}M_{k,i}x_{k,t-i}.
\]

To state the generalization bound, we introduce an auxiliary
quantity that captures the effect of the contraction factor
across communication rounds. For each $r\in[R]$, define
\begin{align}
L_r
&\coloneqq
\inf_{t\geq q_{e,b}^{(R-r)}}
\left\{
t^2\log\max\left(\frac{K\theta}{Ct},2\right)
\right\}\leq
q_{e,b}^{2(R-r)}
\log\max\left(
\frac{K\theta}{Cq_{e,b}^{(R-r)}},2
\right).
\label{Auxiliary_parameters_multi_round_proof}
\end{align}

\begin{theorem}\label{Theorem_Multi_round_LNTP}
Under Assumptions~\ref{a:contract}--\ref{a:KValue},
the expected margin generalization gap of LNTP-SFL
is upper-bounded by
\begin{align}
\mathcal{O}\left(
\frac{1}{KR}
\sum_{k=1}^{K}\sum_{r=1}^{R}
\left(\frac{C\rho^{3/2}}{\theta}\right)^2
\sqrt{
\frac{\tau_{k,\min}R L_r\log(dNT)}
{K^2NT}
}
\right),
\nonumber
\end{align}
where $L_r$ is defined in
\eqref{Auxiliary_parameters_multi_round_proof}.
\end{theorem}

The proof follows from a lossy extension of
Theorem~\ref{Extension_Lemma_1_Margin_generalization}.
In particular, we use a Johnson--Lindenstrauss
transformation to construct an auxiliary lossy
algorithm with reduced complexity while controlling
the distortion in the generalization gap.
The complete proof is provided in
Appendix~\ref{Proof_Theorem_Multi_round_LNTP}, Here, we briefly present a proof sketch.

\emph{Proof sketch.}
We proceed by constructing a lossy approximation of the final
global model for each client--round pair $(k,r)$ and applying
the lossy margin-generalization bound. The main distinction
from the centralized setting is that the contribution of
client $k$ at round $r$ propagates through the subsequent
communication rounds. Using the first-order expansion of
the local SGD updates, we express this contribution through
the product of the corresponding Jacobians, whose norm is
controlled by $q_{e,b}^{R-r}$.

We then exploit the linear structure of the predictor,
which allows the log-margin to be written as a difference
of inner products between token-specific parameter vectors
and input contexts. To compress these vectors, we use a
Gaussian projection $\mathsf A\in\mathbb R^{m\times\rho d}$,
retain those exceeding a prescribed norm threshold, and
perturb their projections with independent uniform noise
before reconstruction. The resulting approximation
$\hat{\overline W}_k^{(R,r)}$ is designed to preserve the
margin up to the prescribed distortion. In particular,
the margin perturbation is controlled by
\[
D_{i,k,r,t}\leq
2\max_{\ell\in[d]}
\left|
\left\langle
V(\overline W^{(R)}-\hat{\overline W}_k^{(R,r)},\ell),
V(X_{i,k,r,t-1:t-\rho})
\right\rangle
\right|.
\]
The distortion analysis separates the error of the
first-order approximation from those introduced by
Gaussian projection and additive noise. Concentration
bounds ensure that their combined contribution satisfies
the required distortion constraint.

Finally, the information carried by the compressed
representation is bounded using data processing, the
number of retained parameter vectors, and the entropy
of the added uniform noise. Substituting this rate
estimate into the lossy generalization bound and
optimizing over the admissible projection parameters
yields the factor $L_r$ and completes the proof.

\subsection{Distributed Self-Attention Next-Token Prediction}

We next consider a distributed next-token prediction model based on a self-attention architecture. For each client $k\in[K]$, let $\mathcal{M}_{k,1}$, $\mathcal{M}_{k,2}$, and $\mathcal{M}_{k,3}$ denote sets of $d\times d$ matrices with bounded Frobenius norms, representing the query, key, and value transformations, respectively. Given an input context $\vec{x}_k=(x_{k,t-1},\ldots,x_{k,t-\rho})\in\mathbb{R}^{\rho d}$, we define the attention-weight mapping $g_{M_{k,[2]}}:\mathbb{R}^{\rho d}\to\mathbb{R}^{\rho}$ and the prediction mapping $f_{M_{k,[3]}}:\mathbb{R}^{\rho d}\to\mathbb{R}^{d}$ as
\begin{align}
g_{M_{k,[2]}}(\vec{x}_k)
&\coloneqq \sigma\left(
\frac{1}{\sqrt{\rho}}
\left\langle M_{k,1}\mathcal{E}(x_{k,t-1}),
M_{k,2}\mathcal{E}(x_{k,t-j})\right\rangle
\right)_{j=1}^{\rho},\nonumber\\
f_{M_{k,[3]}}(\vec{x}_k)
&\coloneqq M_{k,3}\left(
\sum_{j=1}^{\rho}
\left(g_{M_{k,[2]}}(\vec{x}_k)\right)_j
\mathcal{E}(x_{k,t-j})\right).
\nonumber
\end{align}

At communication round $r$, client $k$ selects its model from the constrained parameter space

\[
\mathcal{W}^{(r)}_k
\coloneqq
\left\{
M^{(r)}_{k,[3]}=
\left(M^{(r)}_{k,1},M^{(r)}_{k,2},M^{(r)}_{k,3}\right)
\;\middle|\;
\begin{gathered}
M^{(r)}_{k,i}\in\mathcal{M}^{d\times d}_{k,i},
\quad \|M^{(r)}_{k,i}\|\leq\sqrt{C},\quad i\in[2],\\
\|M^{(r)}_{k,3}\|\leq C
\end{gathered}
\right\}.
\]

The corresponding hypothesis class is defined by
\[
\mathcal{F}_{k,r}
\coloneqq
\left\{f_{M^{(r)}_{k,[3]}}:
M^{(r)}_{k,[3]}\in\mathcal{W}^{(r)}_k\right\}.
\]

The following theorem characterizes the expected generalization gap of distributed self-attention next-token prediction under streaming federated learning.

\begin{theorem}\label{Theorem_Multi_round_SANTP}
Under Assumptions~\ref{a:contract}--\ref{a:KValue}, the expected generalization gap of SANTP-SFL is upper-bounded by
\begin{align}
\mathcal{O}\left(
\frac{1}{KR}\sum_{k=1}^{K}\sum_{r=1}^{R}
\left(\frac{C^2\rho}{\theta}\right)^2
\sqrt{\frac{\tau_{k,\min}R L_r\log(\rho dNKT)}{K^2NT}}
\right),
\nonumber
\end{align}
where $L_r$ is defined in~\eqref{Auxiliary_parameters_multi_round_proof}.
\end{theorem}

The proof is provided in Appendix~\ref{Proof_Theorem_Multi_round_SANTP}. We outline its main ideas below.

\emph{Proof sketch.}
We follow the lossy approach developed for LNTP-SFL,
with an additional step to account for the nonlinear
attention mechanism. We first reparameterize the
self-attention model by introducing
$M_{k,4}=M_{k,2}^{\top}M_{k,1}$, so that its
prediction rule is determined by the pair
$(M_{k,3},M_{k,4})$. For each client--round pair
$(k,r)$, we then construct a lossy approximation
of the final global model by propagating the local
updates through the subsequent communication rounds.
As in the linear case, the first-order expansion
and contraction properties of SGD control the
resulting approximation error.

The two matrices are compressed separately using
independent Gaussian projections $\mathsf{A}_3$
and $\mathsf{A}_4$, with dimensions $m_3$ and $m_4$,
respectively. The construction retains parameter
vectors above a prescribed norm threshold and adds
independent uniform noise to their projections
before reconstruction. To verify the distortion
constraint, we introduce an intermediate predictor
in which only the attention matrix is approximated.
This separates the change in the prediction margin
into the error induced by the attention weights
and that arising from the value transformation.
The Lipschitz continuity of softmax, together with
Gaussian concentration and the bounds on the
SGD remainder, controls both contributions.

Finally, we bound the conditional mutual information
of the compressed model through data processing.
Conditioning on the support indicators of the retained
vectors, we bound the information carried by each
compressed matrix using the entropy of its noisy
representation and the volume of the corresponding
noise ball. Combining these estimates with the
support-entropy bounds and substituting the resulting
rate into the lossy margin-generalization inequality
yields the claimed bound after optimization over
$t\geq q_{e,b}^{R-r}$.

\begin{remark}
The generalization bounds for LNTP-SFL and SANTP-SFL can be expressed in a unified form by choosing $(s,\bar{s})=(4,2)$ for SANTP and $(s,\bar{s})=(2,3)$ for LNTP. In particular, the corresponding order-wise bounds take the form
\begin{align}
\mathcal{O}\left(
\frac{1}{KR}\sum_{k=1}^{K}\sum_{r=1}^{R}
\frac{C^s\rho^{\bar{s}}}{\theta^2}
\sqrt{\frac{\tau_{k,\min}R L_r\log(\rho dNT)}{K^2NT}}
\right).
\label{Summary_bound_brief_1}
\end{align}
\end{remark}

\begin{remark}
The bounds established in Theorems~\ref{Theorem_Multi_round_LNTP} and~\ref{Theorem_Multi_round_SANTP} can be further refined using the mixing-time parameter $\tilde{\tau}_{k,\min}$, yielding
\begin{align}
\mathcal{O}\left(
\frac{1}{KR}\sum_{k=1}^{K}\sum_{r=1}^{R}
\frac{C^s\rho^{\bar{s}}}{\theta^2}
\sqrt{\frac{\tilde{\tau}_{k,\min}R L_r\log(\rho dNT)}{K^2NT}}
\right),
\label{Summary_bound_brief_2}
\end{align}
where $(s,\bar{s})=(4,1)$ for SANTP and $(s,\bar{s})=(2,2)$ for LNTP, and $\tilde{\tau}_{\min}$ is defined in~\eqref{new_tau-min}. These refined bounds are strictly tighter than their scalar-mixing counterparts whenever
\begin{align}
\tilde{\tau}_{k,\min}<(\rho+1)^2\tau_{k,\min},
\nonumber
\end{align}
for any $k\in[K]$.
\end{remark}

\section{Application}

\vspace{-0.2cm}

\subsection{Advantage of distributed processing}

Substituting $K :=1$ and $N := NK$ in Theorem~\ref{Theorem_Multi_round_LNTP}, one readily obtains  the following bound on the generalizations gap of centralized LNTP-SFL operating on $NK$ data trajectories of $T$ samples each, 
\begin{align}
    \frac{\rho}{KR}\sum_{k\in[K]}\sum_{r\in[R]}\sqrt{\frac{C^4\:\rho^6\:\tau_{k,\min}\: R\:L_r\log\left(dNKT\right)\:}{\theta^4KNT}}\leq\mathcal{O}\left(\sqrt{\frac{\log\left(NK\right)}{NK}}\right).
\end{align}
Similarly, one readily obtains a similar bound on  the generalizations gap of centralized SANTP-SFL by substituting $K :=1$ and $N := NK$ in Theorem~\ref{Theorem_Multi_round_SANTP}. In both cases, comparison of these bounds with their distributed versions reveal an advantage of $\mathcal{O}\big(\sqrt{\log K/K}\big)$ of the distributed streaming over centralized streaming under the condition of identical total data budgets. 

The above has important implications. For instance, there is an increasing shift toward substituting centralized training on massive datasets with collaborative training on distributed, smaller data streams, for it offers several advantages, including reduced memory and storage requirements, improved data privacy and ability to adapt models in real-time to non-stationary, site-specific environments. However, this paradigm shift necessitates rigorous safeguards to ensure that the "distributed advantage" does not come at the cost of severe performance degradation. Our theoretical results demonstrate that, \textit{from a generalization gap perspective}, this can be done safely: collaborative streaming is even more efficient than centralized streaming under the same total data budget from a generalization-gap viewpoint. However, we hasten to mention that caution should be exercised in the interpretation of this result. In particular, advantage in terms of generalization gap does not necessarily translate into one in terms of population or test risk (see the reported gap, train and test errors in the Experiments section).


\vspace{-0.2cm}

\subsection{Effect of communication}

\vspace{-0.1cm}

Investigating the results of Thorem~\ref{Theorem_Multi_round_LNTP} and Theorem~\ref{Theorem_Multi_round_SANTP} it can be seen that the (bounds on the) generalization gap increases with the number of communication rounds ($R$), scaling at $\mathcal{O}(\sqrt{R/N})$. This reveals that while frequent communication may accelerate optimization, it may deteriorate the aggregated model's power to generalize. While a similar observation was already reported in~\cite{chor2023more} for SFL on i.i.d. data,  extending this behavior to data with memory validates that the "communication-generalization" tradeoff is an intrinsic property of SFL, rather than an artifact of independent sampling. Also, its proof necessitates in our case a more rigorous analytical framework that accounts for the coupling between model updates and temporally dependent error events, thereby showing that even with the "noise" of Markovian dependencies, frequent server interactions act as a complexity multiplier that can lead to overfitting.

\begin{wrapfigure}[12]{r}{0.35\textwidth}
    \vspace{-20pt}
    \centering
    \includegraphics[width=0.38\textwidth]{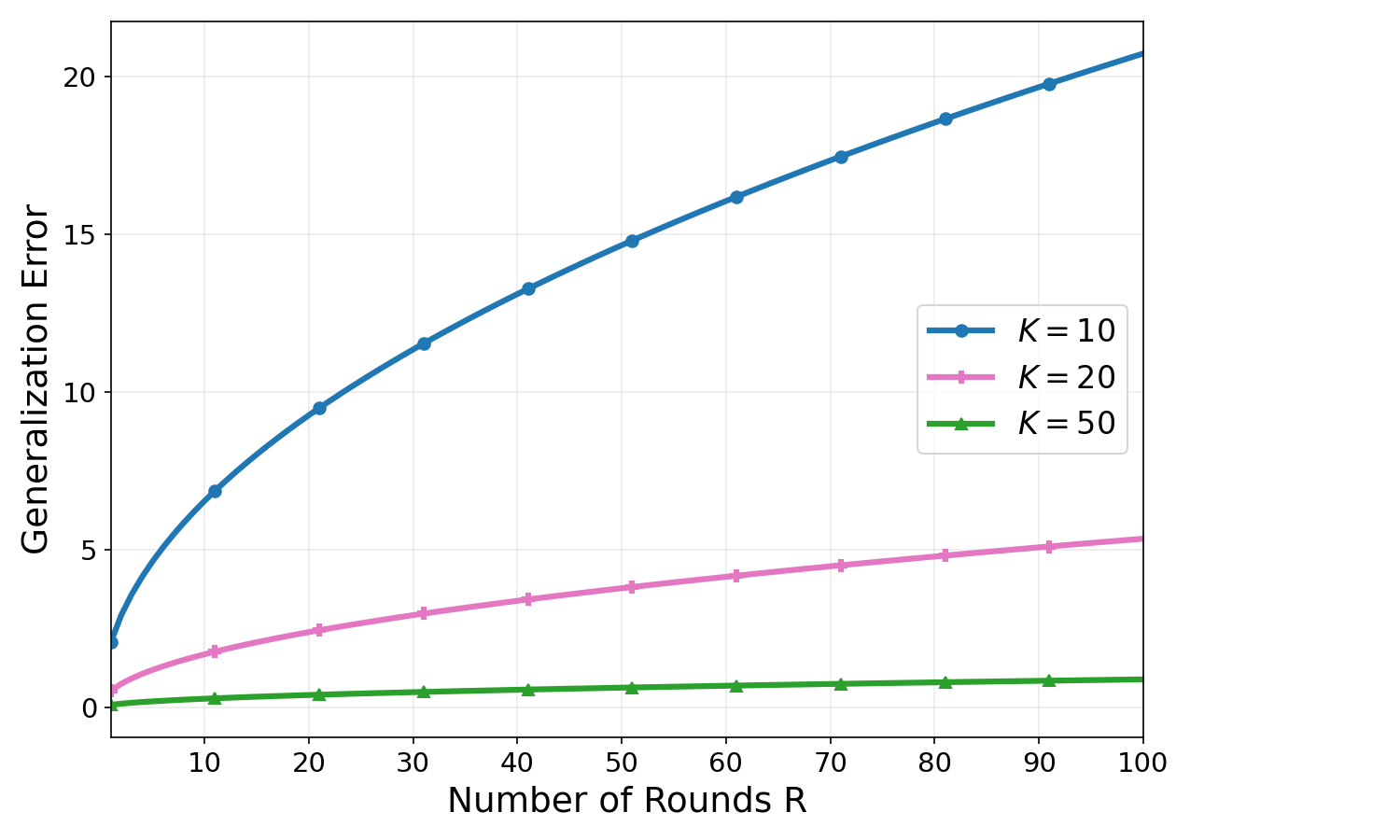}
    \caption{Evolution of the bound of Theorem~\ref{Theorem_Multi_round_LNTP} versus $R$, for $N=100, T=10, \rho=4, \tau_{k,\min}=2, 
    q_{e,b}=1, C=1, d=8, \theta=5$.}
    \label{fig:comparison_bounds}
    \vspace{-10pt}
\end{wrapfigure}

\textbf{Discussion:} A natural baseline is to assume that each round operates on a fixed number of data samples. Under such a setting, increasing the number of rounds $R$ leads to a proportional increase in the total amount of processed data samples. However, this formulation does not allow for a meaningful comparison across different values of $R$, as the total data budget is not controlled. To enable a fair and principled comparison, we instead consider a fixed total data budget. Consequently, the number of samples processed at each round becomes a function of $R$. This setting ensures that the final model is learned from the same total amount of data, regardless of the number of rounds.




\vspace{-0.2cm}

\subsection{Effect of memory (mixing-time)}

\vspace{-0.2cm}

 An interesting question in SFL is how the nature of the data stream -- i.i.d. versus Markovian -- impacts the generalization performance for a fixed data budget. Our analysis demonstrates that Markovian dependencies impose a "memory penalty" that is absent in the i.i.d. setting. Specifically, while i.i.d. samples provide maximal independent information per time step, Markovian samples are statistically coupled, effectively reducing the information rate of the stream. This cost is mathematically quantified in our bounds by the square root of the mixing time ($\sqrt{\tau_{\text{min}}}$), which acts as an inflation factor on the generalization gap. Consequently, to achieve a generalization power equivalent to that of an i.i.d. system, a Markovian SFL framework must collect a larger volume of data to compensate for the redundancy inherent in its temporal dependencies. This is consistent with prior results on hidden Markov models, where it has been shown that prediction performance converges more slowly under weak mixing conditions, as long-range dependencies obscure the latent structure of the process.

\vspace{-0.2cm}

\subsection{Effect of data heterogeneity across clients}

\vspace{-0.2cm}

The bounds established in Theorems~\ref{Theorem_Multi_round_LNTP} and \ref{Theorem_Multi_round_SANTP} reveal a nuanced dependency on the heterogeneous temporal dynamics of the Markov processes seen by the clients, as the generalization gap scales with the summation of local mixing times, $\frac{1}{K}\sum_{k=1}^K \sqrt{\tau_{k,\text{min}}}$. In particular, this means that the aggregated model's generalization performance is determined by the aggregate "memory" across all clients rather than being strictly bottlenecked by the slowest-mixing Markov process. Consequently, replacing individual mixing times with a global maximum ($\max_k \tau_{k,min}$) yields a pessimistic upper bound that fails to capture the mitigation effect provided by faster-mixing clients. This demonstrates that in practical SFL deployments, the presence of clients with rapidly decorrelating data streams can partially offset the "noise" introduced by slower, high-memory processes, leading to a more favorable generalization gap than a worst-case analysis would suggest.




\vspace{-0.4cm}

\section{Experiments}\label{Experimental_Result_main_part}

\vspace{-0.2cm}

\subsection{Language Model Pretraining}
\label{sec:lm-pretraining}

\vspace{-0.2cm}

To probe the SANTP regime predicted by Theorem~\ref{Theorem_Multi_round_SANTP}, we pretrain a GPT-style decoder on the TinyStories \citep{eldan2023tinystoriessmalllanguagemodels} corpus. The training set is sharded contiguously across $K \in \{2, 4, 8\}$ clients, so that each client observes its own dependent token stream, a faithful instance of the Markovian SFL setup of Section~\ref{Streaming_FL_Sec}. We follow the FedAvg protocol and, to isolate the effect of communication from that of total compute, hold the per-client gradient-step budget fixed while sweeping $R$ and allocating the budget evenly across rounds. Every configuration thus consumes the same amount of data and compute; only the communication frequency varies.

\begin{figure*}[h]
    \centering
    \includegraphics[width=\textwidth]{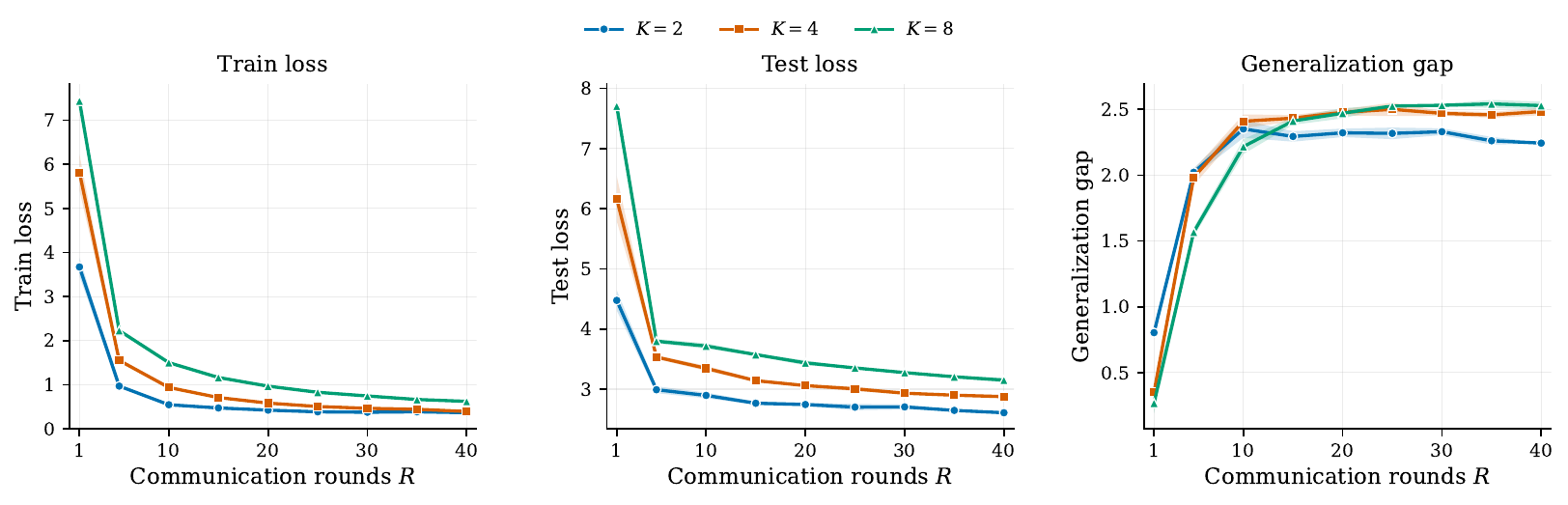}
    \caption{\textbf{Communication--generalization tradeoff in TinyStories pretraining.} Train loss (left), test loss (middle), and generalization gap (right) versus rounds $R$ for $K \in \{2,4,8\}$ clients. Per-client gradient-step budget is fixed across settings and evenly allocated across rounds: only communication frequency varies. Curves show seed averages, with shaded bands denoting standard deviation across three seeds.
    }
    \label{fig:lm-pretraining}
\end{figure*}

Figure~\ref{fig:lm-pretraining} depicts the evolution of the final train loss, test loss, and generalization gap as functions of $R$. As $R$ grows, both the training and the test losses decrease. Yet the gap between the two widens with $R$, in close agreement with Theorem~\ref{Theorem_Multi_round_SANTP} and~(\ref{Summary_bound_brief_1}-\ref{Summary_bound_brief_2}): more frequent aggregation inflated the generalization gap. The effect is consistent across all values of $K$, and is not attributable to overfitting from extra compute, since the total step count is held fixed by construction. This confirms, in a realistic transformer pretraining setting, the communication--generalization tradeoff.

\vspace{-0.3cm}

\subsection{Time Series Forecasting}
\label{sec:ts-forecasting}

\vspace{-0.2cm}

To test whether the communication--generalization tradeoff extends beyond text, we replicate the protocol of Section~\ref{sec:lm-pretraining} on the ETTh2 electricity transformer dataset~\citep{haoyietal-informer-2021}. The univariate target series is uniformly binned into a vocabulary of $B$ symbols, yielding a discrete next-token-prediction problem on temporally dependent data, in line with the Markovian SFL setting of Section~\ref{Streaming_FL_Sec}. Train and test windows are drawn from the same marginal via a block-level random split with a $\tau$-step buffer at each boundary to prevent receptive-field leakage; the training stream is then sharded contiguously across $K \in \{2, 4, 8\}$ clients. We replace the transformer with a GRU-based next-token predictor.

\begin{figure*}[h]
    \centering
    \includegraphics[width=\textwidth]{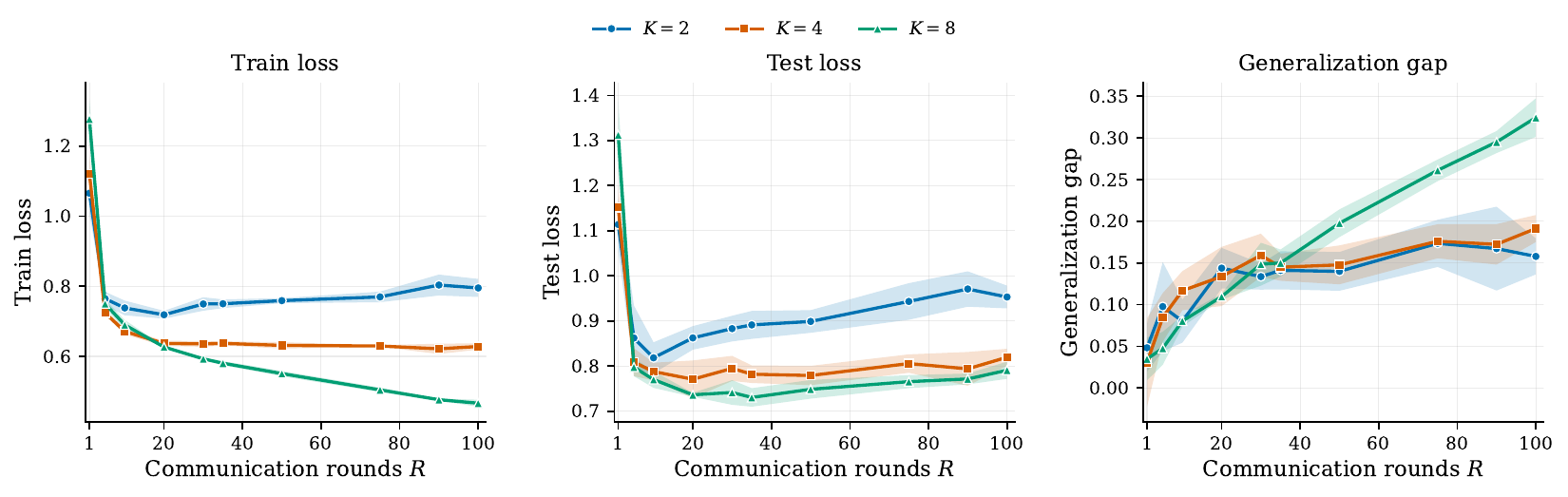}
    \caption{\textbf{Communication--generalization tradeoff in ETTh2 forecasting.}
    Train loss (left), test loss (middle), and generalization gap (right) as functions of rounds count $R$, for $K \in \{2, 4, 8\}$ clients. As in Figure~\ref{fig:lm-pretraining}. Curves show the mean across seeds and shaded bands indicate standard deviation across three seeds.}
    \label{fig:ts-forecasting}
\end{figure*}

Figure~\ref{fig:ts-forecasting} reports the same three quantities as a function of $R$. The qualitative picture mirrors that of the language modeling experiment: train and test losses decrease with $R$, while the generalization gap widens. The effect persists across all values of $K$. This confirms that the communication--generalization tradeoff is not a peculiarity of text pretraining but a robust feature of SFL on dependent data, including the time series regime that motivates much of FL applications in practice (e.g., sensors, energy systems, healthcare monitoring). More details and additional experiments on other time series (ETTh1~\cite{haoyietal-informer-2021} and Beijing Multi-Site Air-Quality dataset~\cite{chen2017beijing}) can be found in Appendix~\ref{app:ts-forecasting-details}.


\bibliography{bib}

\medskip

\newpage

\appendix







\newpage
\onecolumn
\begin{center}
\Large \bf Appendices
\end{center}

The appendices are organized as follows:
\begin{itemize}
    \item We present the assumptions needed for the established generalization bound for Federated Next Token Prediction in Appendix~\ref{assumption_sgd_problem_ntp}.

    \item In Appendix~\ref{Markov_property_random_variable}, we review some of the main properties of the Markovian random variables
    
    \item In Appendices~\ref{Proof_multi_round}-\ref{Proof_max_probability_conditional_markov}, we present all the proofs of the results of the papers, in the order of their appearance, that is:
    \begin{itemize}
        \item Proof of Theorem~\ref{multi_round_markovian_general_bound} presented in Appendix~\ref{Proof_multi_round},
        
        \item Proof of Theorem~\ref{lossy_inexpectation_bound_federated} presented in Appendix~\ref{proof_lossy_case_federated},
        
        \item Proof of Theorem~\ref{Theorem_Multi_round_LNTP} presented in Appendix~\ref{Proof_Theorem_Multi_round_LNTP},

        \item Proof of Theorem~\ref{Theorem_Multi_round_SANTP} presented in Appendix~\ref{Proof_Theorem_Multi_round_SANTP},

        \item Proof of Lemma \ref{Hoeffdinf_Lemma} Presented in Appendix \ref{Proof_Hoeffding_employing}

         \item Proof of Lemma~\ref{lem:window-loss-concentration} presented in Appendix~\ref{Proof:lem:window-loss-concentration},

        \item Proof of Lemma~\ref{max_probability_conditional_markov} presented in Appendix~\ref{Proof_max_probability_conditional_markov},
    \end{itemize}
    \item Finally, Appendices~\ref{app:lm-pretraining-details} and \ref{app:ts-forecasting-details} provide experimental details for language model pretraining and time series forecasting, respectively, along with additional experiments.    
\end{itemize}



\section{Assumptions for Federated Next-Token Prediction}
\label{assumption_sgd_problem_ntp}

Let $m\coloneqq eN_R/b$ denote the number of
local SGD iterations performed by each client
during a communication round, where $e$ is
the number of local epochs, $N_R$ is the
number of local training samples, and $b$
is the mini-batch size.

\begin{assumption}
\label{a:contract}
Consider a $(K,R,n,e,b)$-FL-SGD model. For every
$r\in[2:R]$, every pair of initializations
$(\overbar{w}^{(r-1)},\overbar{w}^{\prime(r-1)})$,
and every client $k\in[K]$, suppose that the
corresponding local models $w_k^{(r)}$ and
$w_k^{\prime(r)}$ are obtained from the same
dataset $s_k^{(r)}$ using identical mini-batches
$\mathcal{B}_{k,r,t}\subseteq s_k^{(r)}$,
$t\in[m]$. We assume that there exists
$q_{e,b}\in\mathbb{R}^+$\footnote{In general,
$q_{e,b}$ may also depend on $m$; this
dependence is omitted for brevity.} such that,
for every such pair of local models,
\begin{align}
\left\|w_k^{(r)}-w_k^{\prime(r)}\right\|
\leq q_{e,b}
\left\|
\overbar{w}^{(r-1)}
-\overbar{w}^{\prime(r-1)}
\right\|.
\end{align}
\end{assumption}

When $q_{e,b}<1$, we say that SGD is
\emph{contractive}. This property has been
established under suitable regularity conditions,
including settings in which the surrogate loss
$\tilde{\ell}$ is smooth and strongly convex.
In general, the contraction factor may depend
on the communication round $r$ and the learning
rates $\eta_{r,t}$. Here, as specified in the
learning setup, we take the learning rates
to be identical across clients and rounds
for simplicity. This restriction is not
essential, and the result can be extended
to round-dependent contraction factors.

The generalization result applies to any
value of $q_{e,b}$, although the case
$q_{e,b}\leq1$ is particularly relevant.
For convex, Lipschitz, and smooth surrogate
losses, the local SGD updates can be
non-expansive, yielding $q_{e,b}\leq1$
under the corresponding step-size conditions.
Strict contractivity, $q_{e,b}<1$, can
further be established for strongly convex
losses under suitable assumptions. In
federated support vector machines (FSVM),
for example, the hinge loss is convex
and becomes strongly convex upon adding
an $L_2$ regularizer. Even without
regularization, where strong convexity
is absent and the theoretical guarantee
only yields $q_{e,b}\leq1$, numerical
results suggest that the contraction
factor can be strictly smaller than one.

\begin{assumption}
\label{a:first_order}
Consider the $(K,R,n,e,b)$-FL-SGD model with
$\mathcal{W}=\mathfrak{B}_d(C)$. Fix a client
$k\in[K]$, a round $r\in[2:R]$, an initial
model $\overbar{w}^{(r-1)}$, and a local
dataset $s_k^{(r)}$ together with its
mini-batches
$\mathcal{B}_{k,r,t}\subseteq s_k^{(r)}$,
$t\in[m]$.We assume that there exists a constant $\alpha\in\mathbb{R}^+$
such that, for every $w_\varepsilon\in\mathbb{R}^d$ satisfying
$\overbar{w}^{(r-1)}+w_\varepsilon/K\in\mathcal{W}$,
\begin{align}
\left\|
w_k^{\prime(r)}
-
\left(
w_k^{(r)}
+
\frac{1}{K}
\mathsf{D}_{(\overbar{w}^{(r-1)},
\{\mathcal{B}_{k,r,t}\}_t)}
w_\varepsilon
\right)
\right\|
\leq
\frac{\alpha}{K^2}\|w_\varepsilon\|^2.
\label{eq:expansion}
\end{align}
Here, $w_k^{(r)}$ and $w_k^{\prime(r)}$
are the local models obtained using
the same mini-batches
$\{\mathcal{B}_{k,r,t}\}_t$, but initialized
at $\overbar{w}^{(r-1)}$ and
$\overbar{w}^{(r-1)}+w_\varepsilon/K$,
respectively. The matrix
$\mathsf{D}_{(\overbar{w}^{(r-1)},
\{\mathcal{B}_{k,r,t}\}_t)}
\in\mathbb{R}^{d\times d}$
represents the overall derivative of the
local SGD procedure with respect to its
initialization, for the prescribed
mini-batches. We further assume that
its spectral norm is bounded by $q_{e,b}$.
\end{assumption}

The expansion in~\eqref{eq:expansion}
describes the response of the local SGD
procedure to a perturbation of its
initialization. Its remainder is of
second order in the perturbation magnitude
and can be controlled under suitable
smoothness assumptions, such as bounded
higher-order derivatives. Since the
initial perturbation is scaled by $1/K$,
the approximation is particularly
relevant when the number of clients
is moderate or large.

More precisely, the derivative matrix
can be expressed as
\begin{align}
\mathsf{D}_{(\overbar{w}^{(r-1)},
\{\mathcal{B}_{k,r,t}\}_t)}
=
\prod_{t\in[m]}
\left(
\mathrm{I}_d
-
\mathrm{\vc{H}}^t_{(\overbar{w}^{(r-1)},
\{\mathcal{B}_{k,r,t'}\}_{t'\in[t]})}
\right),
\end{align}
where $\mathrm{I}_d$ is the $d\times d$
identity matrix, and
\begin{align*}
\mathrm{\vc{H}}^t_{(\overbar{w}^{(r-1)},
\{\mathcal{B}_{k,r,t'}\}_{t'\in[t]})}
\coloneqq
\hess\nolimits_{g_{\mathcal{B}_{k,r,t}}}
\left(W_{k,t}^{(r)}\right).
\end{align*}
Here, $W_{k,t}^{(r)}$ denotes the model
obtained after $t$ local SGD iterations,
starting from $\overbar{w}^{(r-1)}$
and using the mini-batches
$\{\mathcal{B}_{k,r,t'}\}_{t'\in[t]}$.
The notation $\hess_g(\cdot)$ denotes
the Hessian of $g$, with
\begin{align*}
g_{\mathcal{B}}(w)
\coloneqq
\frac{\eta}{b}
\sum_{z\in\mathcal{B}}
\tilde{\ell}(z,w).
\end{align*}

Lastly, we state an assumption on $K$ being sufficiently large.

\begin{assumption}\label{a:KValue}
Suppose Assumptions~3 and~4 hold for some constants $q_{e,b}$ and $\alpha$. Then, for $R\geq 2$,\footnote{No condition is needed for $R=1$.} the number of clients $K$ satisfies at least one of the following conditions.
\begin{itemize}[leftmargin=*]

    \item \textbf{Condition 1:}
    \begin{align}
    K^2 >
    \min_{\beta\in(0,1]}
    \max\Bigg\{
    &\max_{r\in[R-1]}
    \frac{
    \alpha^2 C^2
    \big((1+\beta)q_{e,b}^2\big)^{r-1}
    }{
    q_{e,b}^2\big(\sqrt{1+\beta}-1\big)^2
    },
    \max_{r\in[R-1]}
    \frac{
    48\alpha\sqrt{\rho}\,C^3
    \left(
    q_{e,b}^{r}
    -
    \big((1+\beta)q_{e,b}^2\big)^{r}
    \right)
    }{
    \theta
    \left(
    q_{e,b}
    -
    (1+\beta)q_{e,b}^2
    \right)
    }
    \Bigg\}.
    \label{eq:svmKcond1}
    \end{align}

    \item \textbf{Condition 2:}
    Suppose $q_{e,b}<1$. There exists some
    $\nu\in[0,(1-q_{e,b}^2)/\alpha)$ such that
    \begin{align}
    K^2 >
    \max\Bigg\{
    &\frac{
    \alpha^2 C^2
    }{
    \big(\sqrt{1-\alpha\nu}-q_{e,b}\big)^2
    },
    \max_{r\in[R-1]}
    \frac{
    48\alpha\sqrt{\rho}\,C^3
    \left(
    q_{e,b}^{r}
    -
    (1-\alpha\nu)^{r}
    \right)
    }{
    \theta
    \left(
    q_{e,b}
    -
    (1-\alpha\nu)
    \right)
    }
    \Bigg\}.
    \label{eq:svmKcond2}
    \end{align}

\end{itemize}
\end{assumption}

Any removable singularities in \eqref{eq:svmKcond1} and \eqref{eq:svmKcond2} are understood by continuity.

The purpose of Assumption~\ref{a:KValue} is purely technical. It is used to control the recursive propagation of the perturbation introduced at round $r$ through the subsequent communication rounds. More precisely, the two conditions provide alternative sufficient guarantees under which the accumulated approximation error satisfies
\[
\|\widetilde{\Delta}\|_2<\frac{\epsilon_1}{6\sqrt{\rho}}.
\]
These conditions are introduced only to keep the recursive argument explicit and tractable, and we expect the result to remain valid for more moderate values of $K$ as well.

Condition~1 controls the recursion through the factor $(1+\beta)q_{e,b}^2$, whereas Condition~2 uses the alternative factor $1-\alpha\nu$. In both cases, the second term in the corresponding condition accounts for the cumulative propagation of the second-order approximation errors over the remaining communication rounds.

It is useful to record a simpler sufficient condition for Condition~1. Suppose that $q_{e,b}<1$ and choose
\[
\beta\in
\left(
0,
\min\left\{
1,\frac{1}{q_{e,b}^2}-1
\right\}
\right).
\]
Then $(1+\beta)q_{e,b}^2<1$, so that
\[
\max_{r\in[R-1]}
\big((1+\beta)q_{e,b}^2\big)^{r-1}=1.
\]
Moreover, for any $x,y\geq0$,
\[
\frac{x^r-y^r}{x-y}
=
\sum_{j=0}^{r-1}x^{r-1-j}y^j
\leq
r\max\{x,y\}^{r-1}.
\]
Consequently, a sufficient condition for \eqref{eq:svmKcond1} is
\begin{align*}
K^2 >
\inf_{\beta\in
\left(
0,
\min\left\{
1,\frac{1}{q_{e,b}^2}-1
\right\}
\right)}
\max\Bigg\{
&
\frac{
\alpha^2 C^2
}{
q_{e,b}^2
\big(\sqrt{1+\beta}-1\big)^2
},
\max_{r\in[R-1]}
\frac{
24\alpha\sqrt{\rho}\,C^3\,r
\max\left\{
q_{e,b},
(1+\beta)q_{e,b}^2
\right\}^{r-1}
}{
\epsilon_1
}
\Bigg\}.
\end{align*}

Since both $q_{e,b}$ and $(1+\beta)q_{e,b}^2$ are strictly smaller than one over the above range of $\beta$, the second term inside the maximum decays to zero as $r\to\infty$. Finally, using
\[
r\max\left\{
q_{e,b},
(1+\beta)q_{e,b}^2
\right\}^{r-1}
\leq R,
\qquad r\in[R-1],
\]
the above condition can be further relaxed to
\[
K^2 >
\inf_{\beta\in
\left(
0,
\min\{
1,\frac{1}{q_{e,b}^2}-1\}
\right)}
\max\left\{
\frac{
\alpha^2 C^2
}{
q_{e,b}^2
\big(\sqrt{1+\beta}-1\big)^2
},
\frac{
24\alpha\sqrt{\rho}\,C^3R
}{
\epsilon_1
}
\right\}.
\]




\section{On the Markovian random variables}\label{Markov_property_random_variable}

\begin{definition}[Mixing time of a Markov chain]
Let $(X_1,\ldots,X_T)$ be a Markov chain taking values in the Polish state space
$\mathcal{X}_1 \times \cdots \times \mathcal{X}_T$, where $X_i \in \mathcal{X}_i$ for all $i$.
For $i \geq 1$ and $t \geq 0$, let $\mathbb{P}(X_{i+t}|X_i = x)$ denote the conditional distribution
of $X_{i+t}$ given $X_i = x$.

For any $0 < \epsilon < 1$, the \emph{mixing time} $\tau(\epsilon)$ is defined as the smallest integer $t$
such that, uniformly over all admissible time indices, the conditional distributions of $X_{i+t}$
initialized from any two states are within $\epsilon$ in total variation distance.
Formally, define
\begin{align}
    \overline{d}(t)
    &\coloneqq
    \max_{1 \leq i \leq T - t}
    \sup_{(x,y) \in \mathcal{X}\times\mathcal{X}}
    d_{TV}\!\left(
    \mathbb{P}(X_{i+t} \mid X_i = x),
    \mathbb{P}(X_{i+t} \mid X_i = y)
    \right),
    \\
    \tau(\epsilon)
    &\coloneqq
    \min\bigl\{ t \in \mathbb{N} : \overline{d}(t) \leq \epsilon \bigr\}.
\end{align}
\end{definition}

\begin{definition}[Marton coupling]\label{Marton_coupling}
Let $X \coloneqq (X_1,\ldots,X_T)$ be a random vector taking values in
$\mathcal{X}_1 \times \cdots \times \mathcal{X}_T$.
A \emph{Marton coupling} for $X$ is defined as a family of coupled random vectors
\begin{align}
    \left(
    X^{(x_1,\ldots,x_i,x'_i)},
    X'^{(x_1,\ldots,x_i,x'_i)}
    \right)
    \in
    \mathcal{X} \times \mathcal{X},
\end{align}
constructed for every $i \in [T]$ and every choice of
$x_1 \in \Omega_1, \ldots, x_i \in \Omega_i, x'_i \in \Omega_i$,
and satisfying the following conditions:
\begin{itemize}
    \item
    $
    X_1^{(x_1,\ldots,x_i,x'_i)} = x_1,\ \ldots,\ 
    X_i^{(x_1,\ldots,x_i,x'_i)} = x_i;
    $
    
    \item
    $
    {X'_1}^{(x_1,\ldots,x_i,x'_i)} = x_1,\ \ldots,\ 
    {X'_{i-1}}^{(x_1,\ldots,x_i,x'_i)} = x_{i-1},\ 
    {X'_i}^{(x_1,\ldots,x_i,x'_i)} = x'_i;
    $
    
    \item
    $
    \left(
    X^{(x_1,\ldots,x_i,x'_i)}_{i+1}, \ldots,
    X^{(x_1,\ldots,x_i,x'_i)}_T
    \right)
    \sim
    \mathbb{P}\!\left(
    X_{i+1}, \ldots, X_T
    \mid
    X_1 = x_1, \ldots, X_i = x_i
    \right);
    $
    
    \item
    $
    \left(
    {X'_{i+1}}^{(x_1,\ldots,x_i,x'_i)}, \ldots,
    {X'_T}^{(x_1,\ldots,x_i,x'_i)}
    \right)
    \sim
    \mathbb{P}\!\left(
    X_{i+1}, \ldots, X_T
    \mid
    X_1 = x_1, \ldots, X_{i-1} = x_{i-1}, X_i = x'_i
    \right);
    $
    
    \item
    If $x_i = x'_i$, then
    $
    X^{(x_1,\ldots,x_i,x'_i)} =
    X'^{(x_1,\ldots,x_i,x'_i)}.
    $
\end{itemize}

Associated with a Marton coupling, the \emph{mixing matrix}
$\Gamma \coloneqq (\Gamma_{i,j})_{i,j \leq T}$
is defined as an upper triangular matrix with diagonal entries
$\Gamma_{i,i} \coloneqq 1$ for all $i \leq T$, and off-diagonal entries given by
\begin{align}
    \Gamma_{j,i} &\coloneqq 0, \qquad 1 \leq j < i \leq T, \\
    \Gamma_{i,j}
    &\coloneqq
    \sup_{x_1,\ldots,x_i,x'_i}
    \mathbb{P}\!\left[
    X^{(x_1,\ldots,x_i,x'_i)}_j
    \neq
    X'^{(x_1,\ldots,x_i,x'_i)}_j
    \right],
    \qquad 1 \leq i < j \leq T.
\end{align}
\end{definition}

\begin{definition}[Partition]
A \emph{partition} of a set $S$ is a collection of disjoint, non-empty subsets whose union equals $S$.
Analogously, let $X \coloneqq (X_1,\ldots,X_T)$ be a vector of random variables. We say that
$\hat{X} \coloneqq (\hat{X}_1,\ldots,\hat{X}_{\hat{T}})$ is a partition of $X$ if
$\{\hat{X}_i\}_{i=1}^{\hat{T}}$ forms a partition of the set $\{X_1,\ldots,X_T\}$.

For a given partition $\hat{X}$ of $X$, we denote by $s(\hat{X}_i)$ the number of elements in $\hat{X}_i$
(i.e., the size of $\hat{X}_i$), and define the \emph{size of the partition} as
\[
s(\hat{X}) \coloneqq \max_{1 \leq i \leq \hat{T}} s(\hat{X}_i).
\]

Moreover, we denote by $\mathcal{I}(\hat{X}_i)$ the set of indices corresponding to the elements of
$\hat{X}_i$, that is,
\[
X_j \in \hat{X}_i \quad \text{if and only if} \quad j \in \mathcal{I}(\hat{X}_i).
\]
For any index set $S \subseteq [T]$, we write
\[
X_S \coloneqq \{X_j : j \in S\}.
\]
In particular, $\hat{X}_i = X_{\mathcal{I}(\hat{X}_i)}$.

Finally, if $X$ takes values in the product space
$\mathcal{X} \coloneqq \mathcal{X}_1 \times \cdots \times \mathcal{X}_T$,
then the partitioned random vector $\hat{X}$ takes values in
\[
\hat{\mathcal{X}} \coloneqq \hat{\mathcal{X}}_1 \times \cdots \times \hat{\mathcal{X}}_{\hat{T}},
\]
where each $\hat{\mathcal{X}}_i$ corresponds to the product space indexed by
$\mathcal{I}(\hat{X}_i)$.
\end{definition}

\begin{lemma}[McDiarmid-type inequality for dependent random variables]
\label{Mdiarmid_type_lemma}
Let $X \coloneqq (X_1,\ldots,X_T)$ be a sequence of random variables taking values in
$\mathcal{X}$, with distribution $P$. Let
$\hat{X} \coloneqq (\hat{X}_1,\ldots,\hat{X}_{\hat{T}})$
be a partition of $X$, taking values in $\hat{\mathcal{X}}$ with distribution $\hat{P}$.
Suppose that a Marton coupling exists for $\hat{X}$ with associated mixing matrix $\Gamma$.
Let $c \in \mathbb{R}_+^{T}$ and define $C(c) \in \mathbb{R}_+^{\hat{T}}$ by
\begin{align}
    C_i(c)
    \coloneqq
    \sum_{j \in \mathcal{I}(\hat{X}_i)} c_j,
    \qquad i \leq \hat{T}.
\end{align}
Assume that the function $f:\mathcal{X}\to\mathbb{R}$ satisfies
\begin{align}
    f(x)-f(y)
    \leq
    \sum_{i=1}^{T} c_i \mathbf{1}\{x_i \neq y_i\},
    \label{Mdiarmid_type_lemma_1}
\end{align}
for all $x,y \in \mathcal{X}$. Then, for any $\lambda \in \mathbb{R}$,
\begin{align}
    \log \mathbb{E}\!\left[
    e^{\lambda (f(X)-\mathbb{E}f(X))}
    \right]
    \leq
    \frac{\lambda^2 \|\Gamma C(c)\|^2}{8}.
    \label{mgf_ineq_1}
\end{align}
Consequently, for any $t \geq 0$,
\begin{align}
    \mathbb{P}\left(
    |f(X)-\mathbb{E}f(X)| > t
    \right)
    \leq
    2 \exp\!\left(
    -\frac{2t^2}{\|\Gamma C(c)\|^2}
    \right).
\end{align}
\end{lemma}

\begin{corollary}[McDiarmid's inequality for Markov chains]
\label{mcdiarmid_inequality_markov_chain}
Let $X \coloneqq (X_1,\ldots,X_T)$ be a Markov chain taking values in
$\mathcal{X}_1 \times \cdots \times \mathcal{X}_T$, with mixing time $\tau(\delta)$
for $0 \leq \delta < 1$. Define
\begin{align}
    \tau_{\min}
    \coloneqq
    \inf_{0 \leq \delta < 1}
    \tau(\delta)
    \left(\frac{2-\delta}{1-\delta}\right)^2.
    \label{eq:mixing_time_mcdiarmid}
\end{align}
Suppose that the function $f:\mathcal{X}^T \to \mathbb{R}$ satisfies
\eqref{Mdiarmid_type_lemma_1} for some $c \in \mathbb{R}_+^{T}$.
Then, for any $t \geq 0$,
\begin{align}
    \mathbb{P}\!\left(
    |f(X)-\mathbb{E}f(X)| \geq t
    \right)
    \leq
    2 \exp\!\left(
    -\frac{2t^2}{\|c\|^2 \tau_{\min}}
    \right).
\end{align}
Equivalently, for any $\lambda \in \mathbb{R}$,
\begin{align}
    \log \mathbb{E}\!\left[
    e^{\lambda (f(X)-\mathbb{E}f(X))}
    \right]
    \leq
    \frac{\lambda^2 \|c\|^2 \tau_{\min}}{8}.
    \label{mixing_time_moment_mcdiarmid}
\end{align}
\end{corollary}


\begin{lemma}
\label{lem:window-loss-concentration}
Let $\{X_{t}\}_{t\in[T]}$ be a Markov process of order $\rho$, and let
$Z_{t}:=(X_{t},X_{t-1},\ldots,X_{t-\rho+1})$. For a fixed
$w\in\mathcal W$, define
\[
Y_{t}:=\frac{1}{T}
\log\frac{1}{
p_w(X_{t}|X_{t-1},\ldots,X_{t-\rho})}.
\]
Let $\tilde{\tau}(\delta)$ denote the mixing time of the
history-state process $\{Z_{t}\}_{t\in[T]}$, defined by
\[
\tilde{\tau}(\delta)
:=
\min\left\{
m\geq1:
\max_i\sup_{Z',Z''}
d_{\mathrm{TV}}\!\left(
\mathbb{P}(Z_{i+m}\in\cdot\mid Z_i=z'),
\mathbb{P}(Z_{i+m}\in\cdot\mid Z_i=z'')
\right)
\leq\delta
\right\}.
\]
Then, for any function $f:\mathbb R^T\to\mathbb R$ satisfying
the bounded-difference condition in
\eqref{Mdiarmid_type_lemma_1} with coefficients
$c=(c_1,\ldots,c_T)\in\mathbb R_+^T$, and for every
$\lambda\in\mathbb R$,
\begin{equation}
\log\mathbb E
\left[e^{\lambda\big(f(Y)-\mathbb Ef(Y)\big)}\right]
\le
\frac{\lambda^2\tilde{\tau}_{\min}\|c\|_2^2}{8},
\label{Concentration_under_condition_A}
\end{equation}
where
\begin{equation}
\tilde{\tau}_{\min}
:=
\inf_{0<\delta<1}
\tilde{\tau}(\delta)
\left(\frac{2-\delta}{1-\delta}\right)^2.
\end{equation}
\end{lemma}


\section{Proof of Theorem \ref{multi_round_markovian_general_bound}}\label{Proof_multi_round}

Starting from the definition of the generalization gap in~\eqref{generalization_croos_entropy_def}, we decompose the expected gap into its client- and round-specific contributions. For each $(k,r)$, let $S'_{k,r}$ be an independent copy of $S_{k,r}$, and fix $\lambda_{k,r}>0$. Expressing the population risk through this independent copy and applying the Donsker--Varadhan variational inequality gives
\begin{align}
    &\mathbb{E}_{S_{[K],[R]},\overline{W}^{(R)}}\left[\gen(S_{[K],[R]}, \overline{W}^{(R)})\right]\label{multi_round_generalization_bound_1}\\
    &\qquad=\frac{1}{KR}\sum_{r=1}^{R}\sum_{k=1}^{K}\mathbb{E}\left[\gen(S_{k,r}, \overline{W}^{(R)})\right]\label{multi_round_generalization_bound_2}\\
    &\qquad=\frac{1}{KR}\sum_{r=1}^{R}\sum_{k=1}^{K}\mathbb{E}\left[\mathcal{L}(\overline{W}^{(R)})-\hat{\mathcal{L}}(S_{k,r},\overline{W}^{(R)})\right]\label{multi_round_generalization_bound_3}\\
    &\qquad=\frac{1}{KR}\sum_{k=1}^{K}\sum_{r=1}^{R}\frac{1}{\lambda_{k,r}}\mathbb{E}\Bigg[-\frac{\lambda_{k,r}}{T}\sum_{t=1}^{T}\mathbb{E}_{S'_{k,r}}\left[\log p_{\overline{W}^{(R)}}\left(X'_{k,t}|X'_{k,t-1},\cdots,X'_{t-\rho}\right)\right]\label{multi_round_generalization_bound_4}\\
    &\qquad\qquad\qquad\qquad\qquad\qquad+\frac{\lambda_{k,r}}{N_R T}\sum_{n=1}^{N_R}\sum_{t=1}^{T}\mathbb{E}\left[\log p_{\overline{W}^{(R)}}\left(X^{(n)}_{k,r,t}|X^{(n)}_{k,r,t-1},\cdots,X^{(n)}_{k,r,t-\rho}\right)\right]\Bigg]\label{multi_round_generalization_bound_5}\\
    &\qquad=\frac{1}{KR}\sum_{k=1}^{K}\sum_{r=1}^{R}\frac{1}{\lambda_{k,r}}\mathbb{E}\Bigg[\mathbb{E}\Bigg[-\frac{\lambda_{k,r}}{N_R T}\sum_{t=1}^{T}\sum_{n=1}^{N_R}\mathbb{E}_{S'_{k,r}}\left[\log p_{\overline{W}^{(R)}}(X'^{(n)}_{k,r,t}|X'^{(n)}_{k,r,t-1},\cdots,X'^{(n)}_{k,r,t-\rho})\right]\label{multi_round_generalization_bound_6}\\
    &\qquad\qquad\qquad\qquad\qquad\qquad+\frac{\lambda_{k,r}}{N_R T}\sum_{t=1}^{T}\sum_{n=1}^{N_R}\log p_{\overline{W}^{(R)}}(X^{(n)}_{k,r,t}|X^{(n)}_{k,r,t-1},\cdots,X^{(n)}_{k,r,t-\rho})\Bigg]\Bigg]\label{multi_round_generalization_bound_8}\\
    &\qquad\leq\frac{1}{KR}\sum_{k=1}^{K}\sum_{r=1}^{R}\frac{1}{\lambda_{k,r}}\mathbb{E}\left[D\left(P_{S_{k,r}|\overline{W}^{(R)},W^{(r)}_k,\overline{W}^{(r-1)}}\|P_{S_{k,r}}\right)+\log\mathbb{E}_{S_{k,r}}\left[e^{\phi(S_{k,r},S'_{k,r},\overline{W}^{(R)})}\right]\right]\label{multi_round_generalization_bound_9}\\
    &\qquad=\frac{1}{KR}\sum_{k=1}^{K}\sum_{r=1}^{R}\frac{1}{\lambda_{k,r}}\left[I(S_{k,r};\overline{W}^{(R)},W^{(r)}_k,\overline{W}^{(r-1)})+\mathbb{E}_{\overline{W}^{(R)}}\left[\log\mathbb{E}_{S_{k,r}}\left[e^{\phi(S_{k,r},S'_{k,r},\overline{W}^{(R)})}\right]\right]\right]\label{multi_round_generalization_bound_10}\\
    &\qquad=\frac{1}{KR}\sum_{k=1}^{K}\sum_{r=1}^{R}\frac{1}{\lambda_{k,r}}\left[I(S_{k,r};W^{(r)}_k,\overline{W}^{(r-1)})+\mathbb{E}_{\overline{W}^{(R)}}\left[\log\mathbb{E}_{S_{k,r}}\left[e^{\phi(S_{k,r},S'_{k,r},\overline{W}^{(R)})}\right]\right]\right]\label{multi_round_generalization_bound_11}\\
    &\qquad=\frac{1}{KR}\sum_{k=1}^{K}\sum_{r=1}^{R}\frac{1}{\lambda_{k,r}}\left[I(S_{k,r};W^{(r)}_k|\overline{W}^{(r-1)})+\mathbb{E}_{\overline{W}^{(R)}}\left[\log\mathbb{E}_{S_{k,r}}\left[e^{\phi(S_{k,r},S'_{k,r},\overline{W}^{(R)})}\right]\right]\right]\label{multi_round_generalization_bound_12}
\end{align}
Here, the expectations in~\eqref{multi_round_generalization_bound_2}--\eqref{multi_round_generalization_bound_3} are taken over $(S_{k,r},\overline{W}^{(r-1)},W_k^{(r)},\overline{W}^{(R)})$. In~\eqref{multi_round_generalization_bound_4}--\eqref{multi_round_generalization_bound_8}, we condition on the learned models when evaluating the local loss, and then average over their joint distribution. To obtain~\eqref{multi_round_generalization_bound_9}, we apply the Donsker--Varadhan inequality with
\begin{align}
        &\phi\left(S_{k,r},S'_{k,r},\overline{W}^{(R)}\right)=\frac{\lambda_{k,r}}{N_R}
         \sum_{n=1}^{N_R}
        \left(
            \mathbb{E}_{S'_{k,r}}\!\left[
                \ell\left(S^{(n)}_{k,r},\overline{W}^{(R)}\right)
            \right]
            -\ell\left(S^{(n)}_{k,r},\overline{W}^{(R)}\right)
        \right),\label{Donsker_varadhan_general_theorem_1}
    \end{align}
and choose $P=P_{S_{k,r}\mid\overline{W}^{(R)},W_k^{(r)},\overline{W}^{(r-1)}}$ and $Q=P_{S_{k,r}}$ in
\begin{align}
        \mathbb{E}_{P}[\phi]
        \;\leq\;
        D(P\|Q) + \log \mathbb{E}_{Q}\!\left[e^{\phi}\right].\label{Donsker_varadhan_general_theorem_2}
    \end{align}
The Kullback--Leibler divergence in~\eqref{multi_round_generalization_bound_9} becomes the mutual-information term in~\eqref{multi_round_generalization_bound_10} after averaging over $(\overline{W}^{(R)},W_k^{(r)},\overline{W}^{(r-1)})$, since
\begin{align}
    I\!\left(S_{k,r}; \overline{W}^{(R)}, W^{(r)}_k, \overline{W}^{(r-1)}\right)
    = \mathbb{E}\left[
        D\!\left(
            P_{S_{k,r}\,|\,\overline{W}^{(R)}, W^{(r)}_k, \overline{W}^{(r-1)}}
            \,\big\|\, 
            P_{S_{k,r}}
        \right)
    \right],
\end{align}
We next use the conditional independence induced by the learning protocol. Because the data collected in different rounds and by different clients are independent, and subsequent model updates depend on $S_{k,r}$ only through the local update $W_k^{(r)}$ and the preceding global model $\overline{W}^{(r-1)}$, we have the Markov chain
\begin{align}
    S_{k,r} \;-\; (W^{(r)}_k, \overline{W}^{(r-1)}) \;-\; \overline{W}^{(R)},
\end{align}
Consequently, the chain rule for mutual information eliminates the dependence on the final aggregated model, yielding
\begin{align}
    I\!\left(S_{k,r}; \overline{W}^{(R)}, W^{(r)}_k, \overline{W}^{(r-1)}\right)
    &= I\!\left(S_{k,r}; W^{(r)}_k, \overline{W}^{(r-1)}\right)
    + \underbrace{I\!\left(S_{k,r}; \overline{W}^{(R)}\big|W^{(r)}_k, \overline{W}^{(r-1)}\right)}_{=\,0},
\end{align}
Moreover, $S_{k,r}$ is independent of $\overline{W}^{(r-1)}$, which is determined by earlier rounds. A second application of the chain rule therefore gives
\begin{align}
    I(S_{k,r}; W^{(r)}_k, \overline{W}^{(r-1)})
    = \underbrace{I(S_{k,r}; \overline{W}^{(r-1)})}_{=\,0}
      + I(S_{k,r}; W^{(r)}_k\big|\overline{W}^{(r-1)}).
\end{align}
Substituting these identities into~\eqref{multi_round_generalization_bound_1} leaves a conditional mutual-information term and a logarithmic moment-generating function:

\begin{align}
&\mathbb{E}_{S_{[K],[R]},\overline{W}^{(R)}}
\left[
\gen(S_{[K],[R]},\overline{W}^{(R)})
\right]
\nonumber\\
&\leq
\frac{1}{KR}
\sum_{k=1}^{K}\sum_{r=1}^{R}
\frac{1}{\lambda_{k,r}}
\Bigg[
I\left(S_{k,r};W_k^{(r)}
\mid\overline{W}^{(r-1)}\right)
\nonumber\\
&\qquad\qquad\qquad\qquad+
\mathbb{E}_{\overline{W}^{(R)}}
\log\mathbb{E}_{S_{k,r}}
\exp\left(
\phi(S_{k,r},S'_{k,r},\overline{W}^{(R)})
\right)
\Bigg]
\label{bounding_moment_7}\\
&=
\frac{1}{KR}
\sum_{k=1}^{K}\sum_{r=1}^{R}
\frac{1}{\lambda_{k,r}}
\Bigg[
I\left(S_{k,r};W_k^{(r)}
\mid\overline{W}^{(r-1)}\right)
\nonumber\\
&\qquad\qquad\qquad\qquad+
\mathbb{E}_{\overline{W}^{(R)}}
\log\mathbb{E}_{S_{k,r}}
\exp\left(
\sum_{n=1}^{N_R}
\frac{\lambda_{k,r}}{N_R}
\Delta_{k,r}^{(n)}
\right)
\Bigg]
\nonumber\\
&=
\frac{1}{KR}
\sum_{k=1}^{K}\sum_{r=1}^{R}
\frac{1}{\lambda_{k,r}}
\Bigg[
I\left(S_{k,r};W_k^{(r)}
\mid\overline{W}^{(r-1)}\right)
\nonumber\\
&\qquad\qquad\qquad\qquad+
\mathbb{E}_{\overline{W}^{(R)}}
\log
\prod_{n=1}^{N_R}
\mathbb{E}_{S_{k,r}^{(n)}}
\exp\left(
\frac{\lambda_{k,r}}{N_R}
\Delta_{k,r}^{(n)}
\right)
\Bigg]
\nonumber\\
&=
\frac{1}{KR}
\sum_{k=1}^{K}\sum_{r=1}^{R}
\frac{1}{\lambda_{k,r}}
\Bigg[
I\left(S_{k,r};W_k^{(r)}
\mid\overline{W}^{(r-1)}\right)
\nonumber\\
&\qquad\qquad\qquad\qquad+
\sum_{n=1}^{N_R}
\mathbb{E}_{\overline{W}^{(R)}}
\log\mathbb{E}_{S_{k,r}^{(n)}}
\exp\left(
\frac{\lambda_{k,r}}{N_R}
\Delta_{k,r}^{(n)}
\right)
\Bigg].
\label{bounding_moment_10}
\end{align}

The last step follows by factoring the exponential moment over the $N_R$ independent trajectories in $S_{k,r}$, with the aggregated model held fixed. It remains to control this moment-generating function under temporal dependence within each trajectory. For a fixed $\overline{w}$ and $n\in[N_R]$, consider the trajectory-level cross-entropy loss
\begin{align}
    \ell\left( s^{(n)}_{k,r},\overline{w} \right) = -\frac{1}{T} \sum_{t \in [T]} \log p_{\overline{w}}(x^{(n)}_{k,r,t}|x^{(n)}_{k,r,t-1},\cdots,x^{(n)}_{k,r,t-\rho})\qquad \forall n\in [\alpha],\label{cross_entropy_function_definition}
\end{align}
To verify the bounded-difference condition in~\eqref{Mdiarmid_type_lemma_1}, take two trajectories $s_{k,r}^{(n)}$ and $s_{k,r}'^{(n)}$. We compare their losses by replacing the predicted token and the tokens in its context one at a time. This gives the following telescoping decomposition:
\begin{align}
    &\ell\left( s'^{(n)}_{k,r},\overline{w} \right)-\ell\left( s^{(n)}_{k,r},\overline{w} \right)\label{croos_entropy_loss_func_7}\\
        &=\frac{1}{T}\sum_{t\in[T]}\left[\log p_{\overline{w}}\left(x^{(n)}_{k,r,t}\big|x_{k,r,t-1},\dots, x^{(n)}_{k,r,t-\rho}\right)-\log p_{\overline{w}}\left(x'^{(n)}_{k,r,t}\big|x'^{(n)}_{k,r,t-1}\dots, x'^{(n)}_{k,r,t-\rho}\right)\right]\label{croos_entropy_loss_func_5}\\
        &=\frac{1}{T}\sum_{t\in[T]}\left[\log p_{\overline{w}}\left(x^{(n)}_{k,r,t}|x_{k,r,t-1}\cdots,x^{(n)}_{k,r,t-\rho}\right)-\log p_{\overline{w}}\left(x^{(n)}_{k,r,t}|x^{(n)}_{k,r,t-1},\cdots,x'^{(n)}_{k,r,t-\rho}\right)\right]\\
        &\qquad+\frac{1}{T}\sum_{t\in[T]}\left[\log p_{\overline{w}}\left(x^{(n)}_{k,r,t}|x^{(n)}_{k,r,t-1},\cdots,x'^{(n)}_{k,r,t-\rho}\right)-\log p_{\overline{w}}\left(x^{(n)}_{k,r,t}|x^{(n)}_{k,r,t-1},\cdots,x'^{(n)}_{k,r,t-\rho+1},x'^{(n)}_{k,t-\rho}\right)\right]\\
        &\qquad\qquad\qquad\qquad\qquad\qquad\vdots\\
        &\qquad+\frac{1}{T}\sum_{t\in[T]}\left[\log p_{\overline{w}}\left(x^{(n)}_{k,r,t}|x'^{(n)}_{k,r,t-1}\dots,x'^{(n)}_{k,r,t-\rho}\right)-\log p_{\overline{w}}\left(x'^{(n)}_{k,r,t}| x'^{(n)}_{k,r,t-1}\dots, x'^{(n)}_{k,r,t-\rho}\right)\right]\\
        &=\frac{1}{T}\sum_{t\in[T]}\Bigg[\log p_{\overline{w}}\left(x^{(n)}_{k,r,t}|x'^{(n)}_{k,r,t-1}\dots,x'^{(n)}_{k,r,t-\rho}\right)-\log p_{\overline{w}}\left(x'^{(n)}_{k,r,t}| x'^{(n)}_{k,r,t-1}\dots, x'^{(n)}_{k,r,t-\rho}\right)\\
        &\qquad\qquad\quad+\sum_{i\in[2:T-1]}\Bigg[\log p_{\overline{w}}\left(x^{(n)}_{k,r,t}|x^{(n)}_{k,r,t-1},\cdots,x^{(n)}_{k,r,t-i+1},x'^{(n)}_{k,r,t-i},\cdots x'^{(n)}_{k,r,t-\rho}\right)\label{Difference_loss_telescope_sum_3}\\
        &\qquad\qquad\qquad\qquad\qquad\quad-\log p_{\overline{w}}\left(x^{(n)}_{k,r,t}|x^{(n)}_{k,r,t-1},\cdots,x'^{(n)}_{k,r,t-i+1},x'^{(n)}_{k,r,t-i},\cdots x'^{(n)}_{k,r,t-\rho}\right)\Bigg]\Bigg]\label{Difference_loss_telescope_sum_1}.
        \end{align}
The key observation is that the next-token distribution depends only on the preceding $\rho$ tokens. Therefore, every term involving a replacement outside this context vanishes, and the telescoping expansion reduces to
\begin{align}
&\sum_{t\in[T]}\Bigg[\log p_{\overline{w}}\left(x^{(n)}_{k,r,t}|x'^{(n)}_{k,r,t-1}\dots,x'^{(n)}_{k,r,t-\rho}\right)-\log p_{\overline{w}}\left(x'^{(n)}_{k,r,t}| x'^{(n)}_{k,r,t-1}\dots, x'^{(n)}_{k,r,t-\rho}\right)\nonumber\\
        &\qquad\quad+\sum_{i\in[2:T-1]}\Bigg[\log p_{\overline{w}}\left(x^{(n)}_{k,r,t}|x^{(n)}_{k,r,t-1},\cdots,x^{(n)}_{k,r,t-i+1},x'^{(n)}_{k,r,t-i},\cdots x'^{(n)}_{k,r,t-\rho}\right)\label{Difference_loss_telescope_sum_2}\\
        &\qquad\qquad\qquad\qquad-\log p_{\overline{w}}\left(x^{(n)}_{k,r,t}|x^{(n)}_{k,r,t-1},\cdots,x'^{(n)}_{k,r,t-i+1},x'^{(n)}_{k,r,t-i},\cdots x'^{(n)}_{k,r,t-\rho}\right)\Bigg]\Bigg]\label{telescope_serie_conv_1}\\
    &=\sum_{t=1}^{T}\left[\log p_{\overline{w}}\left(x^{(n)}_{k,r,t}|x'^{(n)}_{k,r,t-1}\dots,x'^{(n)}_{k,r,t-\rho}\right)-\log p_{\overline{w}}\left(x'^{(n)}_{k,r,t}| x'^{(n)}_{k,r,t-1}\dots, x'^{(n)}_{k,r,t-\rho}\right)\right]\nonumber\\
    &
    \qquad+\sum_{t=1}^{T}\sum_{i=2}^{T-1}\Bigg(\log p_{\overline{w}}\left(x^{(n)}_{k,r,t}|x^{(n)}_{k,r,t-1},\cdots,x^{(n)}_{k,r,t-i+1},x'^{(n)}_{k,r,t-i},\cdots x'^{(n)}_{k,r,t-\rho}\right)\label{telescope_serie_conv_3}\\
    &\qquad\qquad\qquad\qquad-\log p_{\overline{w}}\left(x^{(n)}_{k,r,t}|x^{(n)}_{k,r,t-1},\cdots,x'^{(n)}_{k,r,t-i+1},x'^{(n)}_{k,r,t-i},\cdots x'^{(n)}_{k,r,t-\rho}\right)\Bigg)\mathbf{1}_{\{1\leq i\leq\rho\}}\label{telescope_serie_conv_4}\\
    &=\sum_{t=1}^{T}\left[\log p_{\overline{w}}\left(x^{(n)}_{k,r,t}|x'^{(n)}_{k,r,t-1}\dots,x'^{(n)}_{k,r,t-\rho}\right)-\log p_{\overline{w}}\left(x'^{(n)}_{k,r,t}| x'^{(n)}_{k,r,t-1}\dots, x'^{(n)}_{k,r,t-\rho}\right)\right]\mathbf{1}_{\{x^{(n)}_{k,r,t}\neq x'^{(n)}_{k,r,t}\}}\nonumber\\
    &\qquad+\sum_{t=1}^{T}\sum_{i=2}^{\rho}\Bigg(\log p_{\overline{w}}\left(x^{(n)}_{k,r,t}|x^{(n)}_{k,r,t-1},\cdots,x^{(n)}_{k,r,t-i+1},x'^{(n)}_{k,r,t-i},\cdots x'^{(n)}_{k,r,t-\rho}\right)\nonumber\\
    &\qquad\qquad\qquad\qquad-\log p_{\overline{w}}\left(x^{(n)}_{k,r,t}|x^{(n)}_{k,r,t-1},\cdots,x'^{(n)}_{k,r,t-i+1},x'^{(n)}_{k,r,t-i},\cdots x'^{(n)}_{k,r,t-\rho}\right)\Bigg)\mathbf{1}_{\{x^{(n)}_{k,r,i}\neq x'^{(n)}_{k,r,i}\}},\label{telescope_serie_conv_5}
    \end{align}
In particular, the restriction $1\leq i\leq\rho$ in~\eqref{telescope_serie_conv_4}--\eqref{telescope_serie_conv_5} follows directly from the finite context length. We now bound separately the change in the predicted token and the change in an individual context token. Taking the maximum over the two trajectories gives
\begin{align}
        &\log p_{\overline{w}}\left(x^{(n)}_{k,r,t}|x'^{(n)}_{k,r,t-1}\dots,x'^{(n)}_{k,r,t-\rho}\right)-\log p_{\overline{w}}\left(x'^{(n)}_{k,r,t}| x'^{(n)}_{k,r,t-1}\dots, x'^{(n)}_{k,r,t-\rho}\right)\nonumber\\
        &\quad\leq\max\Bigg|\log p_{\overline{w}}\left(x^{(n)}_{k,r,t}|x'^{(n)}_{k,r,t-1}\dots,x'^{(n)}_{k,r,t-\rho}\right)-\log p_{\overline{w}}\left(x'^{(n)}_{k,r,t}| x'^{(n)}_{k,r,t-1}\dots, x'^{(n)}_{k,r,t-\rho}\right)\Bigg|\label{Defnition_max_difference_telescope_2}\\
        &\quad=\left|\log\frac{p_{\overline{w}}\left(x^{(n)}_{k,r,t}|x'^{(n)}_{k,r,t-1}\dots,x'^{(n)}_{k,r,t-\rho}\right)}{p_{\overline{w}}\left(x'^{(n)}_{k,r,t}| x'^{(n)}_{k,r,t-1}\dots, x'^{(n)}_{k,r,t-\rho}\right)}\right|\nonumber\\
        &\quad\coloneqq C_{k,r}\label{Defnition_max_difference_telescope_1}
    \end{align}
For a change in the context, the corresponding maximal difference is bounded in the same way. For each $t\in[T]$ and $i\in[2:\rho]$, we obtain
\begin{align}
        &\log p_{\overline{w}}\left(x^{(n)}_{k,r,t}|x^{(n)}_{k,r,t-1},\cdots,x^{(n)}_{k,r,t-i+1},x'^{(n)}_{k,r,t-i},\cdots x'^{(n)}_{k,r,t-\rho}\right)\label{telescope_serie_conv_10}\\
    &\qquad\qquad-\log p_{\overline{w}}\left(x^{(n)}_{k,r,t}|x^{(n)}_{k,r,t-1},\cdots,x'^{(n)}_{k,r,t-i+1},x'^{(n)}_{k,r,t-i},\cdots x'^{(n)}_{k,r,t-\rho}\right)\nonumber\\
    &\leq \max\Bigg|\log p_{\overline{w}}\left(x^{(n)}_{k,r,t}|x^{(n)}_{k,r,t-1},\cdots,x^{(n)}_{k,r,t-i+1},x'^{(n)}_{k,r,t-i},\cdots x'^{(n)}_{k,r,t-\rho}\right)\label{telescope_serie_conv_11}\\
    &\qquad\qquad-\log p_{\overline{w}}\left(x^{(n)}_{k,r,t}|x^{(n)}_{k,r,t-1},\cdots,x'^{(n)}_{k,r,t-i+1},x'^{(n)}_{k,r,t-i},\cdots x'^{(n)}_{k,r,t-\rho}\right)\Bigg|\label{telescope_serie_conv_12}\\
    &\coloneqq \tilde{C}_{k,r}\nonumber
    \end{align}
Here the maximum is over the same pairs of trajectories, uniformly in the admissible values of $t$ and $i$. Combining these two estimates in~\eqref{telescope_serie_conv_5} yields
\begin{align}
    &\eqref{telescope_serie_conv_5}\leq\left[\max\Bigg|\log p_{\overline{w}}\left(x^{(n)}_{k,r,t}|x'^{(n)}_{k,r,t-1}\dots,x'^{(n)}_{k,r,t-\rho}\right)-\log p_{\overline{w}}\left(x'^{(n)}_{k,r,t}| x'^{(n)}_{k,r,t-1}\dots, x'^{(n)}_{k,r,t-\rho}\right)\Bigg|\right]\sum_{t=1}^{T}\mathbf{1}_{\{x^{(n)}_{k,r,t}\neq x'^{(n)}_{k,r,t}\}}\label{telescope_serie_conv_16}\\
    &\qquad\qquad+\Bigg[\max\Bigg|\log p_{\overline{w}}\left(x^{(n)}_{k,r,t}|x^{(n)}_{k,r,t-1},\cdots,x^{(n)}_{k,r,t-i+1},x'^{(n)}_{k,r,t-i},\cdots x'^{(n)}_{k,r,t-\rho}\right)\label{telescope_serie_conv_6}\\
    &\qquad\qquad\qquad-\log p_{\overline{w}}\left(x^{(n)}_{k,r,t}|x^{(n)}_{k,r,t-1},\cdots,x'^{(n)}_{k,r,t-i+1},x'^{(n)}_{k,r,t-i},\cdots x'^{(n)}_{k,r,t-\rho}\right)\Bigg|\Bigg]\sum_{t=1}^{T}\sum_{i=1}^{\rho}\mathbf{1}_{\{x^{(n)}_{k,r,t-i}\neq x'^{(n)}_{k,r,t-i}\}}\label{telescope_serie_conv_7}\\
&\qquad\quad\leq(C_{k,r}+\rho\tilde{C}_{k,r})\times\sum_{t=1}^{T}\mathbf{1}_{\{x^{(n)}_{k,r,t}\neq x'^{(n)}_{k,r,t}\}},\label{telescope_serie_conv_9}
\end{align}
The last inequality follows from the fact that a discrepancy at a given token can appear in at most $\rho$ subsequent prediction contexts. More precisely,
\begin{align}
    \sum_{t=1}^{T}\sum_{i=1}^{\rho}\mathbf{1}\{x^{(n)}_{k,r,t-i}\neq x'^{(n)}_{k,r,t-i}\}\leq\rho\sum_{t=1}^{T}\mathbf{1}\{x^{(n)}_{k,r,t}\neq x'^{(n)}_{k,r,t}\},
\end{align}
Thus, changing one observation affects its own prediction term and at most $\rho$ context terms. Returning to the difference of trajectory losses, we obtain
\begin{align}
    \ell(s^{(n)}_{k,r},\overline{w})-\ell(s'^{(n)}_{k,r},\overline{w})&\leq\frac{\left(C_{k,r}+\rho \tilde{C}_{k,r}\right)}{T}\times\sum_{t=1}^{T}\mathbf{1}\{x^{(n)}_{k,r,t}\neq x'^{(n)}_{k,r,t}\},\label{difference_worse_bound_2}
\end{align}
This is the bounded-difference property required by~\eqref{Mdiarmid_type_lemma_1}, with a coefficient $c_{k,r,t}$ bounded by
\begin{align}
    &\frac{\left(C_{k,r}+\rho \tilde{C}_{k,r}\right)}{T},\label{difference_moment_max}
\end{align}
for every client $k\in[K]$, communication round $r\in[R]$, and time index $t\in[T]$. We next control the constants appearing in this coefficient. The following elementary bound on the cross-entropy loss will be useful.

\begin{lemma}\label{max_probability_conditional_markov}
For any $w\in\mathcal{W}$, we have
\begin{align}
        \log\frac{1}{p_{{{w}}}(x_{k,t}|x_{k,t-1},\cdots,x_{k,t-\tau})}\leq 2B+\log d.
    \end{align}
\end{lemma}

The proof is given in Appendix~\ref{Proof_max_probability_conditional_markov}. Applying this lemma to the change in the predicted token gives
\begin{align}
    \eqref{Defnition_max_difference_telescope_1}&=\max\left|\log\frac{p_{\overline{w}}\left(x^{(n)}_{k,r,t}|x'^{(n)}_{k,r,t-1}\dots,x'^{(n)}_{k,r,t-\rho}\right)}{p_{\overline{w}}\left(x'^{(n)}_{k,r,t}| x'^{(n)}_{k,r,t-1}\dots, x'^{(n)}_{k,r,t-\rho}\right)}\right|\\
    &\leq\max\left[\log\frac{1}{p_{\overline{w}}\left(x'^{(n)}_{k,r,t}| x'^{(n)}_{k,r,t-1}\dots, x'^{(n)}_{k,r,t-\rho}\right)}\right]\\
    &\leq 2B+\log d,\label{upperbound_negative_loss}
\end{align}
where the maximum is taken over the same admissible token and context pairs as in~\eqref{Defnition_max_difference_telescope_2}. We now turn to the change caused by replacing a token in the context. Writing the ratio of the corresponding conditional probabilities, we obtain
\begin{align}
        \eqref{telescope_serie_conv_12}&=\max\Bigg|\log p_{\overline{w}}\left(x^{(n)}_{k,r,t}|x^{(n)}_{k,r,t-1},\cdots,x^{(n)}_{k,r,t-i+1},x'^{(n)}_{k,r,t-i},\cdots x'^{(n)}_{k,r,t-\rho}\right)\nonumber\\
    &\qquad\qquad\qquad-\log p_{\overline{w}}\left(x^{(n)}_{k,t}|x^{(n)}_{k,t-1},\cdots,x'^{(n)}_{k,r,t-i+1},x'^{(n)}_{k,r,t-i},\cdots x'^{(n)}_{k,r,t-\rho}\right)\Bigg|\\
        &=\max\Bigg|\log\left[\frac{p_{\overline{w}}\left(x^{(n)}_{k,r,t}|x^{(n)}_{k,r,t-1},\cdots,x^{(n)}_{k,r,t-i+1},x'^{(n)}_{k,r,t-i},\cdots x'^{(n)}_{k,r,t-\rho}\right)}{p_{\overline{w}}\left(x^{(n)}_{k,t}|x^{(n)}_{k,r,t-1},\cdots,x'^{(n)}_{k,r,t-i+1},x'^{(n)}_{k,r,t-i},\cdots x'^{(n)}_{k,r,t-\rho}\right)}\right]\Bigg|\label{croos_entropy_loss_func_4}\\
&=\max\Bigg|\log\left[\frac{\big<\sigma\left(g_{\overline{w}}\left(\mathcal{E}(x^{(n)}_{k,r,t-1}),\cdots,\mathcal{E}(x'^{(n)}_{k,r,t-i}),\cdots,\mathcal{E}(x'^{(n)}_{k,r,t-\rho})\right)\right),\mathcal{E}(x^{(n)}_{k,r,t})\big>}{\big<\sigma\left(g_{\overline{w}}\left(\mathcal{E}(x^{(n)}_{k,r,t-1}),\cdots,\mathcal{E}(x'^{(n)}_{k,r,t-i+1}),\cdots,\mathcal{E}(x'^{(n)}_{k,r,t-\rho})\right)\right),\mathcal{E}(x^{(n)}_{k,r,t})\big>}\right]\Bigg|\label{croos_entropy_loss_func_1}
    \end{align}
In~\eqref{croos_entropy_loss_func_1}, we have substituted the predictive distribution from~\eqref{conditional_distribution_def}; its two terms correspond to the softmax outputs evaluated at the two contexts, paired with the same one-hot encoded target token,
\begin{align}
\sigma\left(g_{w}\left(\mathcal{E}(x^{(n)}_{k,r,t-1}), \cdots, \mathcal{E}(x^{(n)}_{k,r,t-\rho})\right)\right),
\end{align}
The two contexts differ only in the token being replaced. Using the softmax representation in~\eqref{soft_max_def}, we can express the resulting log-probability difference in terms of the corresponding logits. For convenience, define
\begin{align}
&g_{w,j} = g_w\left(\mathcal{E}(x^{(n)}_{k,r,t-1}), \cdots,\mathcal{E}(x'^{(n)}_{k,r,t-i}),\cdots,\mathcal{E}(x'^{(n)}_{k,r,t-\rho})\right)_j,\label{element_wise_g_funct_1} \\
&g'_{w,j} = g_w\left(\mathcal{E}(x^{(n)}_{k,r,t-1}), \cdots,\mathcal{E}(x'^{(n)}_{k,r,t-i+1}),\cdots,\mathcal{E}(x'^{(n)}_{k,r,t-\rho})\right)_j,\label{element_wise_g_funct_2}
\end{align}
and the associated softmax coordinates
\begin{align}
&g^{(j)}_{w} = \frac{\exp({g_{w,j}})}{\sum_{j \in [d]} \exp({g_{w,j}})}, \qquad g'^{(j)}_{w} = \frac{\exp({g'_{w,j}})}{\sum_{j \in [d]} \exp({g'_{w,j}})},\label{change_measure_g_func_1}
\end{align}
Here $g_{w,j}$ and $g'_{w,j}$ denote the $j$-th logits computed from the two contexts, while $\mathcal{E}_j(x^{(n)}_{k,r,t})$ is the $j$-th coordinate of the one-hot encoded target token. With this notation, the ratio in~\eqref{croos_entropy_loss_func_1} can be decomposed into a term involving the target-token logits and a term involving the softmax normalizers:
\begin{align}
        \eqref{croos_entropy_loss_func_1} &= \Bigg| \left[\log\left[\frac{\sum_{j\in[d]}g^{(j)}_{\overline{w}}\cdot\mathcal{E}_j(x^{(n)}_{k,r,t})}{\sum_{j\in[d]}g'^{(j)}_{\overline{w}}\cdot\mathcal{E}_j(x^{(n)}_{k,r,t})}\right]\right]\nonumber\\         &=\Bigg|\log\left[\frac{\sum_{j\in[d]}\exp({g_{\overline{w},j}})\cdot\mathcal{E}_j(x^{(n)}_{k,r,t})}{\sum_{j\in[d]}\exp({g'_{\overline{w},j}})\cdot\mathcal{E}_j(x^{(n)}_{k,r,t})}\right] 
         +\log\left[\frac{\sum_{j\in[d]}\exp({g'_{\overline{w},j}})}{\sum_{j\in[d]}\exp({g_{\overline{w},j}})}\right]\Bigg|\nonumber \\
         &\leq\Bigg|\log\left[\frac{\sum_{j\in[d]}\exp({g_{\overline{w},j}})\cdot\mathcal{E}_j(x^{(n)}_{k,r,t})}{\sum_{j\in[d]}\exp({g'_{\overline{w},j}})\cdot\mathcal{E}_j(x^{(n)}_{k,r,t})}\right]\Bigg|
          +\Bigg|\left[\log\left(\frac{\sum_{j\in[d]}\exp({g'_{\overline{w},j}})}{\sum_{j\in[d]}\exp({g_{\overline{w},j}}})\right)\right]\Bigg|\label{simplification_difference_cross_entropy_11}\\
        &\leq\max\Bigg\{\max_{j\in[d]}\log\left[\frac{\exp({g_{\overline{w},j}})\cdot\mathcal{E}_j(x^{(n)}_{k,r,t})}{\exp({g'_{\overline{w},j}})\cdot\mathcal{E}_j(x^{(n)}_{k,r,t})}\right],\max_{j\in[d]}\log\left[\frac{\exp({g'_{\overline{w},j}})\cdot\mathcal{E}_j(x^{(n)}_{k,r,t})}{\exp({g_{\overline{w},j}})\cdot\mathcal{E}_j(x^{(n)}_{k,r,t})}\right]\Bigg\}\label{simplification_difference_croos_entropy_12}\\
        &\quad +\max\Bigg\{\max_{j\in[d]}\log\left[\frac{\exp({g_{\overline{w},j}})}{\exp({g'_{\overline{w},j}})}\right],\max_{j\in[d]}\log\left[\frac{\exp({g'_{\overline{w},j}})}{\exp({g'_{\overline{w},j}})}\right]\Bigg\}\label{simplification_difference_croos_entropy_15}\\
        &=\max_{j\in[d]}\Big|g_{\overline{w},j}-g'_{\overline{w},j}\Big|+\max_{j\in[d]}\left|g_{\overline{w},j}-g'_{\overline{w},j}\right|\nonumber\\
        &\leq 4\max_{j\in[d]}\left|g_{\overline{w},j}\right|\label{simplification_difference_croos_entropy_17}\\
        &=4\|g_{\overline{w}}\|_{\infty}\nonumber\\
        &=4B\label{simplification_difference_croos_entropy_19}
    \end{align}
The first inequality follows from the triangle inequality. To bound the two logarithmic ratios, we use the fact that the ratio of two sums of positive terms is no larger than the largest ratio of corresponding terms. For the softmax normalizers, this gives
\begin{align}
       |x+y|\leq |x|+|y|,\nonumber
   \end{align}
Indeed, for two collections of positive numbers $(a_1,\ldots,a_d)$ and $(b_1,\ldots,b_d)$, we have
\begin{align}
        \Bigg|\log\left(\frac{\sum_{j\in[d]} \exp({g_{\overline{w},j}})}{\sum_{j\in[d]} \exp({g'_{\overline{w},j}})}\right)\Bigg|&=\max\Bigg\{\log\left(\frac{\sum_{j\in[d]} \exp({g_{\overline{w},j}})}{\sum_{j\in[d]} \exp({g'_{\overline{w},j}})}\right),\log\left(\frac{\sum_{j\in[d]} \exp({g'_{\overline{w},i}})}{\sum_{j\in[d]} \exp({g_{\overline{w},j}})}\right)\Bigg\}\nonumber\\
        &\leq \max\Bigg\{\log\left(\max_{j\in[d]} \frac{\exp({g_{\overline{w},j}})}{\exp({g'_{\overline{w},j}})}\right),\log\left(\max_{j\in[d]} \frac{\exp({g'_{\overline{w},j}})}{\exp({g_{\overline{w},j}})}\right)\Bigg\}\label{max_bound_change_g_func_1}\\
        &= \max\Bigg\{\max_{j\in[d]}\log\left( \frac{\exp({g_{\overline{w},j}})}{\exp({g'_{\overline{w},j}})}\right),\max_{j\in[d]}\log\left( \frac{\exp({g'_{\overline{w},j}})}{\exp({g_{\overline{w},j}})}\right)\Bigg\},\label{max_bound_change_g_func_2}\\
        &=\max_{j\in [d]}\Big|g_{\overline{w},j}-g'_{\overline{w},j}\Big|,
    \end{align}
This comparison uses the elementary inequality
\begin{align}
        \frac{a_1+\cdots+a_d}{b_1+\cdots+b_d}\leq\max_{j\in[d]}\frac{a_j}{b_j}
    \end{align}
Together with the bounded-output assumption~\ref{subsubsec:bounded_output}, the preceding estimates bound the context-dependent contribution. Substituting~\eqref{upperbound_negative_loss} and~\eqref{simplification_difference_croos_entropy_19} into~\eqref{difference_worse_bound_2} gives, for every $k\in[K]$, $r\in[R]$, and $t\in[T]$,
\begin{align}
c_{k,r,t}\leq\frac{\left(2B(1+2\rho)+\log d\right)}{T}.\label{bounding_constant_mc_diarmid}
\end{align}

We can therefore apply the Markov-chain McDiarmid inequality in Corollary~\ref{mcdiarmid_inequality_markov_chain} to the centered trajectory loss. Its exponential moment is controlled by the client-specific mixing time as follows:
\begin{align*}
\sum_{n=1}^{N}\mathbb{E}_{\overline{W},W_k}\left[\log\mathbb{E}_{S_{k}}\left[e^{\frac{\lambda_{k,r}}{N}\left(\mathbb{E}_{S_k}\left[\ell(S^{(n)}_{k,r},\overline{W})\right]-\left[\ell(S^{(n)}_{k,r},\overline{W})\right]\right)}\right]\right]\leq\frac{{{\lambda^2_{k,r}}}\|c_k\|^2\tau_{k,\min}}{8N},
\end{align*}
Substituting this estimate into~\eqref{multi_round_generalization_bound_12} and optimizing over $\lambda_{k,r}>0$, we obtain
\begin{align}
    &\mathbb{E}_{S_{[K],[R]},\overline{W}^{(R)}}\left[\gen(S_{[K],[R]}, \overline{W}^{(R)})\right]\\
    &\qquad\leq\frac{1}{KR}\sum_{k=1}^{K}\sum_{r=1}^{R}\frac{1}{\lambda_{k,r}}\left[I(S_{k,r};W^{(r)}_k|\overline{W}^{(r-1)})+\mathbb{E}_{\overline{W}^{(R)}}\left[\log\mathbb{E}_{S_{k}}\left[e^{\phi(S_{k,r},S'_{k,r},\overline{W}^{(R)})}\right]\right]\right]\\
    &\qquad=\frac{1}{KR}\sum_{k=1}^{K}\sum_{r=1}^{R}\frac{1}{\lambda_{k,r}}\left[I(S_{k,r};W^{(r)}_k|\overline{W}^{(r-1)})+\mathbb{E}_{\overline{W}^{(R)}}\left[\log\mathbb{E}_{S_{k,r}}\left[e^{\phi(S_{k,r},S'_{k,r},\overline{W}^{(R)})}\right]\right]\right]\\
    &\qquad\leq\frac{1}{KR}\sum_{k=1}^{K}\sum_{r=1}^{R}\frac{1}{\lambda_{k,r}}\left[I(S_{k,r};W^{(r)}_k|\overline{W}^{(r-1)})+\frac{{\lambda_{k,r}}^{2}\|c_{k,r}\|^{2}\tau_{k,\min}}{8 N_R}\right]\\
    &\qquad=\frac{(2B(1+2\rho)+\log d)}{KR}\sum_{k=1}^{K}\sum_{r=1}^{R}\sqrt{\frac{\tau_{k,\min}I(S_{k,r};W^{(r)}_k|\overline{W}^{(r-1)})}{2N_R T}},\label{substtuting_lambda_k_1}
\end{align}
The optimization is attained by choosing
\begin{align}
    \lambda_{k,r}=\sqrt{\frac{8N_R I(S_{k,r};W^{(r)}_k|\overline{W}^{(r-1)})}{\|c_{k,r}\|^2\tau_{k,\min}}},
\end{align}
Finally, using $N_R=N/R$ gives the bound in Theorem~\ref{multi_round_markovian_general_bound}, completing the proof.

\qed




\section{Proof of Theorem \ref{lossy_inexpectation_bound_federated}}\label{proof_lossy_case_federated}

We follow the argument used for the lossless bound, replacing the final aggregated model by its lossy counterpart. Fix a collection of approximation kernels satisfying the distortion constraints in~\eqref{def:Rdist}, and let $U$ be the auxiliary random variable, independent of the data and of the learning process. For each client $k$ and round $r$, let $S'_{k,r}$ be an independent copy of $S_{k,r}$. Throughout the proof, $\lambda_{k,r}>0$ is arbitrary.

By the definition of the generalization gap and the distortion constraint, we obtain
\begin{align}
\mathbb{E}\big[\gen(S_{[K]},\overline{W}^{(R)})\big]&=\frac{1}{KR}\sum_{k=1}^{K}\sum_{r=1}^{R}
\mathbb{E}\big[\gen(S_{k,r},\overline{W}^{(R)})\big]\nonumber\\
&\leq\frac{1}{KR}\sum_{k=1}^{K}\sum_{r=1}^{R}
\mathbb{E}\big[\gen(S_{k,r},\hat{\overline{W}}_k^{(R,r)})\big]+\epsilon\nonumber\\
&=\frac{1}{KR}\sum_{k=1}^{K}\sum_{r=1}^{R}
\mathbb{E}\Big[\mathcal{L}(\hat{\overline{W}}_k^{(R,r)})
-\hat{\mathcal{L}}(S_{k,r},\hat{\overline{W}}_k^{(R,r)})\Big]+\epsilon\nonumber\\
&=\frac{1}{KR}\sum_{k=1}^{K}\sum_{r=1}^{R}\frac{1}{N_R}
\sum_{n=1}^{N_R}\mathbb{E}\Big[
\mathbb{E}_{S_{k,r}'^{(n)}\sim\mu_k}
\ell(S_{k,r}'^{(n)},\hat{\overline{W}}_k^{(R,r)})
-\ell(S_{k,r}^{(n)},\hat{\overline{W}}_k^{(R,r)})\Big]+\epsilon.
\label{lossy_multi_round_generalization_1}
\end{align}

For notational convenience, introduce the centered trajectory loss
\begin{align}
\phi(S_{k,r}^{(n)},\hat{\overline{W}}_k^{(R,r)})
\coloneqq\lambda_{k,r}\Bigg(
&\mathbb{E}_{S_{k,r}'^{(n)}\sim\mu_k}
\left[\ell(S_{k,r}'^{(n)},\hat{\overline{W}}_k^{(R,r)})\right]-\ell(S_{k,r}^{(n)},\hat{\overline{W}}_k^{(R,r)})\Bigg).
\end{align}
The expectation in~\eqref{lossy_multi_round_generalization_1} may therefore be taken with respect to the joint law of $S_{k,r}$, $\hat{\overline W}_k^{(R,r)}$, $\overline W^{(r-1)}$, $W_{[K]\setminus k}^{(r)}$, and $U$. Applying the Donsker--Varadhan inequality, as in~\eqref{Donsker_varadhan_general_theorem_1}--\eqref{Donsker_varadhan_general_theorem_2}, conditionally on $\big(\overline W^{(r-1)},W_{[K]\setminus k}^{(r)},U\big)$ gives
\begin{align}
&\mathbb{E}\big[\gen(S_{[K]},\overline{W}^{(R)})\big]\leq\frac{1}{KR}\sum_{k=1}^{K}\sum_{r=1}^{R}\frac{1}{\lambda_{k,r}}
\Bigg[
I\Big(S_{k,r};\hat{\overline{W}}_k^{(R,r)}\,\Big|\,
\overline W^{(r-1)},W_{[K]\setminus k}^{(r)},U\Big)\nonumber\\
&\qquad+\mathbb{E}_{\overline W^{(r-1)},W_{[K]\setminus k}^{(r)},U,\hat{\overline W}_k^{(R,r)}}
\log\mathbb{E}_{S_{k,r}}
\exp\!\left\{\frac{1}{N_R}\sum_{n=1}^{N_R}
\phi(S_{k,r}^{(n)},\hat{\overline W}_k^{(R,r)})\right\}\Bigg]+\epsilon.
\label{lossy_multi_round_generalization_3}
\end{align}
Here the reference law $P_{S_{k,r}}$ is justified by the independence of the current dataset from $\big(\overline W^{(r-1)},W_{[K]\setminus k}^{(r)},U\big)$. Equivalently, the divergence term in the conditional Donsker--Varadhan step is
\begin{align}
&\mathbb{E}\left[D\Big(P_{S_{k,r}\mid\hat{\overline W}_k^{(R,r)},\overline W^{(r-1)},W_{[K]\setminus k}^{(r)},U}\Big\|P_{S_{k,r}}\Big)\right]\nonumber\\
&\quad=I\Big(S_{k,r};\hat{\overline W}_k^{(R,r)},\overline W^{(r-1)},W_{[K]\setminus k}^{(r)},U\Big).\nonumber
\end{align}
Since $S_{k,r}$ is independent of $\big(\overline W^{(r-1)},W_{[K]\setminus k}^{(r)},U\big)$, the chain rule reduces this information term to
\begin{align}
&I\Big(S_{k,r};\hat{\overline W}_k^{(R,r)},\overline W^{(r-1)},W_{[K]\setminus k}^{(r)},U\Big)\nonumber\\
&\quad=I\Big(S_{k,r};\hat{\overline W}_k^{(R,r)}\,\Big|\,
\overline W^{(r-1)},W_{[K]\setminus k}^{(r)},U\Big).
\end{align}

Moreover, the $N_R$ trajectories collected within a client and round are independent under the reference law. Consequently, the logarithmic moment-generating function in~\eqref{lossy_multi_round_generalization_3} factorizes across trajectories, yielding
\begin{align}
\mathbb{E}\big[\gen(S_{[K]},\overline{W}^{(R)})\big]&\leq\frac{1}{KR}\sum_{k=1}^{K}\sum_{r=1}^{R}\frac{1}{\lambda_{k,r}}
\Bigg[I\Big(S_{k,r};\hat{\overline W}_k^{(R,r)}\,\Big|\,
\overline W^{(r-1)},W_{[K]\setminus k}^{(r)},U\Big)\nonumber\\
&\qquad+\mathbb{E}_{\overline W^{(r-1)},W_{[K]\setminus k}^{(r)},U,\hat{\overline W}_k^{(R,r)}}
\log\prod_{n=1}^{N_R}\mathbb{E}_{S_{k,r}^{(n)}}
\exp\!\left\{\frac{\phi(S_{k,r}^{(n)},\hat{\overline W}_k^{(R,r)})}{N_R}\right\}\Bigg]+\epsilon.
\label{lossy_multi_round_generalization_4}
\end{align}
Using $\log\prod_n a_n=\sum_n\log a_n$ gives the equivalent form
\begin{align}
&\mathbb{E}\big[\gen(S_{[K]},\overline{W}^{(R)})\big]\leq\frac{1}{KR}\sum_{k=1}^{K}\sum_{r=1}^{R}\frac{1}{\lambda_{k,r}}
\Bigg[I\Big(S_{k,r};\hat{\overline W}_k^{(R,r)}\,\Big|\,
\overline W^{(r-1)},W_{[K]\setminus k}^{(r)},U\Big)\nonumber\\
&\qquad+\sum_{n=1}^{N_R}\mathbb{E}_{\overline W^{(r-1)},W_{[K]\setminus k}^{(r)},U,\hat{\overline W}_k^{(R,r)}}
\log\mathbb{E}_{S_{k,r}^{(n)}}
\exp\!\left\{\frac{\phi(S_{k,r}^{(n)},\hat{\overline W}_k^{(R,r)})}{N_R}\right\}\Bigg]+\epsilon.
\label{lossy_multi_round_generalization_6}
\end{align}

It remains to control the moment-generating function. For each fixed realization of $\hat{\overline W}_k^{(R,r)}$, the same bounded-difference argument used in the proof of Theorem~\ref{multi_round_markovian_general_bound} applies to the trajectory loss. Provided the approximate models satisfy the same bounded-output assumption, the Markov-chain McDiarmid inequality gives
\begin{align}
&\sum_{n=1}^{N_R}\mathbb{E}\log\mathbb{E}_{S_{k,r}^{(n)}}
\left[\exp\!\left\{\frac{\phi(S_{k,r}^{(n)},\hat{\overline W}_k^{(R,r)})}{N_R}\right\}\right]\leq\frac{\lambda_{k,r}^2\|c_{k,r}\|^2\tau_{k,\min}}{8N_R},
\end{align}
where $\|c_{k,r}\|^2\leq\big(2B(1+2\rho)+\log d\big)^2/T$. Inserting this estimate into~\eqref{lossy_multi_round_generalization_6} and minimizing each summand over $\lambda_{k,r}>0$, with
\begin{align}
\lambda_{k,r}=
\sqrt{\frac{8N_R\,I\big(S_{k,r};\hat{\overline W}_k^{(R,r)}\mid\overline W^{(r-1)},W_{[K]\setminus k}^{(r)},U\big)}{\|c_{k,r}\|^2\tau_{k,\min}}},
\end{align}
we obtain
\begin{align}
&\mathbb{E}\big[\gen(S_{[K]},\overline{W}^{(R)})\big]\nonumber\\
&\leq\frac{2B(1+2\rho)+\log d}{KR}
\sum_{k=1}^{K}\sum_{r=1}^{R}
\sqrt{\frac{\tau_{k,\min}}{2N_RT}
I\Big(S_{k,r};\hat{\overline W}_k^{(R,r)}\,\Big|\,
\overline W^{(r-1)},W_{[K]\setminus k}^{(r)},U\Big)}+\epsilon.
\label{lossy_multi_round_generalization_7}
\end{align}
Finally, using $N_R=N/R$ and taking the infimum over the admissible approximation kernels gives the asserted rate-distortion bound, provided the rate-distortion function is defined with the same conditioning on $U$ and the admissible kernels satisfy the stated distortion constraints. This completes the proof.
\qed




\section{Proof of Theorem \ref{Theorem_Multi_round_LNTP}}\label{Proof_Theorem_Multi_round_LNTP}

Let assume that assumptions in appendix \ref{assumption_sgd_problem_ntp} holds then
For ease of notation, let denote the variables  

\begin{align}
m&=\left\lceil 63\left(\frac{C\cdot t\cdot{\rho}}{K\epsilon_1}\right)^2\log\left(dNKT\right)\right\rceil,\label{Paramaters_dimension_multi_round_1}\\
c_1&=q^{(R-r)}_{e,b}\sqrt{\frac{K^2\epsilon^2_1}{C^2t^2}+1},\:\nu=\frac{t}{2c_1}\\
c_2&=\sqrt{\frac{K^2\epsilon_1^2}{C^2t^2}+1},\:\epsilon_1=\frac{\theta}{2}\label{Paramaters_dimension_multi_round_3}
\end{align}
where $t\geq q^{(R-r)}_{e,b}$.

We use an extension of Theorem \ref{lossy_inexpectation_bound_federated}, that it can be derive easily with the same manner. To prove this Theorem, let $\epsilon_{k,r}$ for all $k\in[K]$ and $r\in[R]$ such that satisfies the equation \eqref{def:Rdist}. Fix a $k\in[K]$ and $r\in[R]$. We upper bound the rate-distortion term $ R_{\mathcal{D}}(k,r,\epsilon)$, defined in \eqref{general_bound_lossy_mutual_1}. To this end, let $V^{(r)}_k$ denotes all randomness in $r$-th round for client $k$, i.e., given $s^{(r)}_k$, all mini-batches $\{\mathcal{B}_{k,r,t}\}_t$ will be fixed.
We define a proper 
\begin{align}
    P_{\hat{\overline{W}}^{(R,r)}_k|S_{k,r},\mathfrak{B}_{k,r}},\label{auxiliary_algorithm_condition_distribution}
\end{align}

that satisfy condition \ref{def:Rdist}, where

\begin{align}
    \mathfrak{B}_{k,r}=\left(V^{(r)}_{k},\overline{W}^{(r-1)},W^{(r)}_{[K]/k},S^{[r+1:R]}_{[K]},V^{[r+1:R]}_{[K]}\right).\label{Auxiliary_mini_batch_multi_round}
\end{align}

For simplicity, let

\begin{align}
    \mathfrak{U}_{k,r}=\left(W^{(r)}_{[K]/k},S^{[r+1:R]}_{[K]},V^{[r+1:R]}_{[K]}\right).\label{Auxiliary_simplicity_mini_batch_multi_round}
\end{align}

To define $P_{\hat{\overline{W}}^{(R,r)}_k|S_{k,r},\mathfrak{B}_{k,r}}$, first we define $P_{\hat{\overline{W}}^{(R,r)}_k|W^{(r)}_k,\mathfrak{U}_{k,r}}$. Then, we let
\begin{align} P_{\hat{\overline{W}}^{(R,r)}_k|W^{(r)}_k,\mathfrak{B}_{k,r},S_{k,r}}=P_{\hat{\overline{W}}^{(R,r)}_k|W^{(r)}_k,\mathfrak{U}_{k,r}}.
\end{align}

Now, we have

\begin{align}
    P_{\hat{\overline{W}}^{(R,r)}_k|S_{k,r},\mathfrak{B}_{k,r}}=\mathbb{E}_{W^{(r)}_k\sim P_{W^{(r)}_k|S_{k,r},V^{(r)}_k,\overline{W}^{(r-1)}}}\left[P_{\hat{\overline{W}}^{(R,r)}|W^{(r)}_k,\mathfrak{U}_{k,r}}\right].
\end{align}

Note that $P_{\hat{\overline{W}}^{(r)}_k|S_{k,r},\mathfrak{B}_{k,r}}=P_{W^{(r)}_k|S_{k,r},V^{(r)}_k,\overline{W}^{(r-1)}}$. We proceed to define $P_{\hat{\overline{W}}^{(R,r)}_k|W^{(r)}_k,\mathfrak{U}_{k,r}}$. Consider an integer $m\in\mathbb{N}$. Fix an $\mathfrak{U}_{k,r}$, $i.e.$, fix $w^{(r)}_{[K]/k},\left(s^{[r+1:R]}_{[K]},v^{[r+1:R]}_{[K]}\right)$  and consequently all mini-batches

\begin{align}
    \{\mathfrak{B}_{k,r',t}\}_{k\in[K],r'\in[r+1:R],t\in[\rho]}.
\end{align}

Assume the aggregated model at round $(r)$ to be
\begin{align}
    \overbar{w}_k^{(r),\star} = \frac{1}{K}\sum_{j\neq k} w_{j}^{(r)}.\label{aggrigated_model_in_r-th_round}
\end{align}
Then, with this aggregated model at step $(r)$, and the considered fixed $\mathfrak{U}_{k,r}$, if $r<R$, consider all the induced $\mathsf{D}^{\star}_{(\overbar{w}^{(r'-1)},\{\mathcal{B}_{j,r',t}\}_t)}$, for $j\in[K]$, $r'\in[r+1:R]$. Denote  
\begin{align}
\mathsf{D}^{\star}_{r'}\coloneqq & \frac{1}{K}\sum_{k\in[K]}\mathsf{D}^{\star}_{(\overbar{w}^{(r'-1)},\{\mathcal{B}_{k,r',t}\}_t)},\quad r'\in[r+1:R],\nonumber\\
\mathsf{D}^{\star}\coloneqq & \prod_{r'\in[r+1:R]} \mathsf{D}^{\star}_{r'}.
\end{align}
Denote also the resulting global model as $\overbar{w}_k^{(R),\star}$.  If $r=R$, let $\mathsf{D}^{\star}=\mathrm{I}_d$, where $\mathrm{I}_d$ is the identity matrix.

It is easy to verify that 
\begin{align}
    &\mathbb{E}_{\vc{S},\vc{V},\vc{W},\hat{\overbar{W}}_k^{(R,r)}}\left[\operatorname{gen}_{k,\theta}(S_{k,r},\overbar{W}^{(R)}) - \operatorname{gen}_{k,0}(S_{k,r},\hat{\overbar{W}}_k^{(R,r)})\right]\nonumber\\
    &= \mathbb{E}_{P_{S_{k,r},W_k^{(r)},\mathfrak{U}_{k,r},\overbar{W}^{(R)}}} \mathbb{E}_{p_{\hat{\overbar{W}}_k^{(R,r)}|W_k^{(r)},\mathfrak{U}_{k,r}}}\left[\operatorname{gen}_{k,\theta}(S_{k,r},\overbar{W}^{(R)}) - \operatorname{gen}_{k,0}(S_{k,r},\hat{\overbar{W}}_k^{(R,r)})\right] \nonumber \\
    &= \mathbb{E}_{P_{W_k^{(r)},\mathfrak{U}_{k,r},\overbar{W}^{(R)}}} \mathbb{E}_{p_{\hat{\overbar{W}}_k^{(R,r)}|W_k^{(r)},\mathfrak{U}_{k,r}}}
    \mathbb{E}_{Z_k\sim \mu_k}\left[\mathcal{L}_{k,-\theta/2}(\overbar{W}^{(R)}) - {\mathcal{L}}_{k,0}(\hat{\overbar{W}}_k^{(R,r)})\right] \nonumber \\
    &\quad- \mathbb{E}_{P_{S_{k,r},W_k^{(r)},\mathfrak{U}_{k,r},\overbar{W}^{(R)}}} \mathbb{E}_{p_{\hat{\overbar{W}}_k^{(R,r)}|W_k^{(r)},\mathfrak{U}_{k,r}}}\left[\hat{\mathcal{L}}_{k,{\theta/2}}(S_{k,r},{\overline{W}}^{(R)})-\hat{\mathcal{L}}_{k,0}(S_{k,r},\hat{\overline{W}}^{(R,r)}_k)\right] \nonumber \\
    &= 
    \mathbb{E}_{S_{k,r},\overline{W}^{(R)}}\left[\frac{1}{T}\sum_{t \in [T]}
        \mathbf{1}\left\{ \log\frac{C_1(\overline{W}^{(R)},t,\vec{X}_{k,r})}{C_2(\overline{W}^{(R)},t,\vec{X}_{k,r})} < {-\theta/2} \right\}\right]\\
    &\quad-
    \mathbb{E}_{S_{k,r},\hat{\overline{W}}^{(R,r)}_k}\left[\frac{1}{T}\sum_{t \in [T]}
        \mathbf{1}\left\{\log\frac{C_1(\hat{\overline{W}}^{(R,r)}_k,t,\vec{X}_{k,r})}{C_2(\hat{\overline{W}}^{(R,r)}_k,t,\vec{X}_{k,r})} < {0} \right\}\right]\\
        &+\mathbb{E}_{S_{k,r},\hat{\overline{W}}^{(R,r)}_k}\Bigg[\frac{1}{N_R}\sum_{n\in[N_R]}\Bigg[\frac{1}{T}\sum_{t \in [T]}
    \mathbf{1}\left\{ \log\frac{C_{1,n}(\hat{\overline{W}}^{(R,r)}_k,t,\vec{X}^{(n)}_{k,r})}{C_{2,n}(\hat{\overline{W}}^{(R,r)}_k,t,\vec{X}^{(n)}_{k,r})}<0 \right\}\\
    &\quad-\mathbb{E}_{S_{k,r},{\overline{W}}^{(R)}}\Bigg[\frac{1}{N_R}\sum_{n\in[N_R]}\Bigg[\frac{1}{T}\sum_{t \in [T]}
    \mathbf{1}\left\{ \log\frac{C_{1,n}({\overline{W}}^{(R)},t,\vec{X}^{(n)}_{k,r})}{C_{2,n}({\overline{W}}^{(R)},t,\vec{X}^{(n)}_{k,r})}<\theta/2 \right\}\Bigg]\Bigg]\\
    &\leq
    \mathbb{E}_{S_{k,r},\overline{W}^{(R)}}\Bigg[\frac{1}{T}\sum_{t \in [T]}
        \mathbf{1}\left\{\left|\log\frac{C_1(\overline{W}^{(R)},t,\vec{X}_{k,r})}{C_2({\overline{W}}^{(R)},t,\vec{X}_{k,r})}-\log\frac{C_1(\hat{\overline{W}}^{(R,r)}_k,t,\vec{X}_{k,r})}{C_2(\hat{\overline{W}}_k^{(R,r)},t,\vec{X}_{k,r})}\right|>\frac{\theta}{2}\right\}\\
    &\quad+\mathbb{E}_{S_{k,r},\hat{\overline{W}}^{(R,r)}_k}\Bigg[\frac{1}{N_R}\sum_{n\in[N_R]}\Bigg[\frac{1}{T}\sum_{t \in [T]}
    \mathbf{1}\left\{\left|\log\frac{C_{1,n}(\hat{\overline{W}}^{(R,r)}_k,t,\vec{X}^{(n)}_{k,r})}{C_{2,n}(\hat{\overline{W}}^{(R,r)}_k,t,\vec{X}^{(n)}_{k,r})}-\log\frac{C_{1,n}({\overline{W}}^{(R)},t,\vec{X}^{(n)}_{k,r})}{C_{2,n}({\overline{W}}^{(R)},t,\vec{X}^{(n)}_{k,r})}\right|>\frac{\theta}{2}\right\}\\
    &\quad\coloneqq\epsilon_A\label{lossy_pop_emp_difference_1}
\end{align}

Let 
\begin{align}
C_{1}(w,t,\vec{\mathbf{x}}) &= p_{w}\left(x_{t}|x_{t-1}, \dots, x_{t-\rho}\right),\label{defition_maximal_ratio_1} \\
C_{2}(w,t,\vec{\mathbf{x}}) &= \max_{x'_{t} \neq x_{t}}
p_{w}\left({x}'_{t}|x_{t-1}, \dots, x_{t-\rho}\right).\label{defition_maximal_ratio_2}
\end{align}

Given the model $w$, we done the corresponding matrices by $M_j^{w}$, where $j\in[\rho]$. Moreover, denote  
\begin{align}
    u_{i,t} \coloneqq & \mathcal{E}(x_{i,t}), \nonumber \\
    \ell_{i,t} \coloneqq& \ell, \text{ if } u_{i,t} = e_{\ell}.
\end{align}
In other words, $\ell_{i,t}$ is the order of $x_{i,t}$ in the dictionary. Define similar notation of $\ell'_{i,t}$. Denote also the element $(a_1,a_2)\in [d] \times [d]$ of $M_j^{w}$ by $M_j^{w}[a_1,a_2]$. 

Then, we can write
\begin{align}
    C_{i,1}(w,t,\vec{\mathbf{x}}) &= \frac{e^{\sum_{j\in[\rho]} M_j^{w}[\ell_{i,t-j},\ell_{i,t}]}}{\sum_{\ell\in[d]}e^{\sum_{j\in[\rho]} M_j^{w}[\ell_{i,t-j},\ell]}}, \qquad
    C_{i,2}(w,t,\vec{\mathbf{x}}) = \frac{\max_{\ell'_{i,t} \neq \ell_{i,t} }e^{\sum_{j\in[\rho]} M_j^{w}[\ell_{i,t-j},\ell'_{i,t}]}}{\sum_{\ell\in[d]}e^{\sum_{j\in[\rho]} M_j^{w}[\ell_{i,t-j},\ell]}}. \nonumber
\end{align}
Hence,
\begin{align}
    \log \frac{ C_{i,1}(w,t,\vec{\mathbf{x}})}{ C_{i,2}(w,t,\vec{\mathbf{x}})} \coloneqq \sum\limits_{j\in[\rho]} M_j^{w}[\ell_{i,t-j},\ell_{i,t}]- \max_{\ell'_{i,t} \neq \ell_{i,t} } \sum\limits_{j\in[\rho]} M_j^{w}[\ell_{i,t-j},\ell'_{i,t}].
\end{align}
To express this term in the form of an inner product, given a model $w$, define
\begin{align}
    V(w,r) \triangleq & \Big(M_1^{w}[1,r],M_1^{w}[2,r],\ldots,M_1^{w}[d,r],\ldots,M_{\rho}^{w}[1,r],M_{\rho}^{w}[2,r],\ldots,M_{\rho}^{w}[d,r] \Big) \in \mathbb{R}^{\rho d},\nonumber \\
    V(x_{i,t-1:t-\rho}) \triangleq & \Big(u_{i,t-1},\ldots,u_{i,t-\rho}\Big) \in \mathbb{R}^{\rho d}.\label{Auxiliary_vector_definition_LNTP}
\end{align}
With this definition, we have
\begin{align}
    \log \frac{ C_{i,1}(w,t,\vec{\mathbf{x}})}{ C_{i,2}(w,t,\vec{\mathbf{x}})} = \left\langle   V(w,\ell_{i,t}) , V(x_{i,t-1:t-\rho}) \right\rangle- \max_{\ell'_{i,t} \neq \ell_{i,t}} \left\langle   V(w,\ell'_{i,t}) , V(x_{i,t-1:t-\rho}) \right\rangle.
\end{align}
 
 Thus,

\begin{align}
D_{i,k,r,t}&\coloneqq\left|\log\frac{C_{i,1}(\overline{W}^{(R)},t,\vec{X}_{k,r})}{C_{i,2}({\overline{W}}^{(R)},t,\vec{X}_{k,r})}-\log\frac{C_1(\hat{\overline{W}}^{(R,r)}_k,t,\vec{X}_{k,r})}{C_2(\hat{\overline{W}}_k^{(R,r)},t,\vec{X}_{k,r})}\right|\\
&\leq \left| \left\langle   V(\overline{W}^{(R)},\ell_{i,t}) - V(\hat{\overline{W}}^{(R,r)}_k,\ell_{i,t})  , V(X_{i,k,r,t-1:t-\rho}) \right\rangle \right|\nonumber\\
&\quad +\left| \max_{\ell'_{i,t} \neq \ell_{i,t}} \left\langle   V({\overline{W}}^{(R)},\ell'_{i,t}) , V(X_{i,k,r,t-1:t-\rho}) \right\rangle - \max_{\ell'_{i,t} \neq \ell_{i,t}} \left\langle   V(\hat{\overline{W}}_k^{(R,r)},\ell'_{i,t}) , V(X_{i,k,r,t-1:t-\rho}) \right\rangle \right| \nonumber\\
&\leq\left| \left\langle   V(\overline{W}^{(R)},\ell_{i,t}) - V(\hat{\overline{W}}^{(R,r)}_k,\ell_{i,t})  , V(X_{i,k,r,t-1:t-\rho}) \right\rangle \right|\nonumber\\
&\quad +\max_{\ell'_{i,t} \neq \ell_{i,t}}\left|  \left\langle   V({\overline{W}}^{(R)},\ell'_{i,t}) , V(X_{i,k,r,t-1:t-\rho}) \right\rangle - \left\langle   V(\hat{\overline{W}}_k^{(R,r)},\ell'_{i,t}) , V(X_{i,k,r,t-1:t-\rho}) \right\rangle \right| \nonumber\\
&=\left| \left\langle   V(\overline{W}^{(R)},\ell_{i,t}) - V(\hat{\overline{W}}^{(R,r)}_k,\ell_{i,t})  , V(X_{i,k,r,t-1:t-\rho}) \right\rangle \right|\nonumber\\
&\quad +\max_{\ell'_{i,t} \neq \ell_{i,t}}\left|  \left\langle   V({\overline{W}}^{(R)},\ell'_{i,t})-V(\hat{\overline{W}}_k^{(R,r)},\ell'_{i,t}) , V(X_{i,k,r,t-1:t-\rho}) \right\rangle \right| \nonumber\\
&=\left|\left\langle V(\overline{W}^{(R)}-\hat{\overline{W}}^{(R,r)}_k,\ell_{i,t}), V(X_{i,k,r,t-1:t-\rho}) \right\rangle \right|\nonumber\\
&\quad +\max_{\ell'_{i,t} \neq \ell_{i,t}}\left|  \left\langle   V({\overline{W}}^{(R)}-\hat{\overline{W}}_k^{(R,r)},\ell'_{i,t}), V(X_{i,k,r,t-1:t-\rho}) \right\rangle \right| \nonumber\\
&\leq 2\max_{\ell_{i,t}}\left|  \left\langle   V({\overline{W}}^{(R)}-\hat{\overline{W}}_k^{(R,r)},\ell_{i,t}), V(X_{i,k,r,t-1:t-\rho}) \right\rangle \right|.
    \end{align}
 Note that $V(\mathsf{A}\mathsf{D}^\ast w^{(r)}_k,\ell)=\mathsf{A}V(\mathsf{D}^\ast w^{(r)}_k,\ell)$ and $\sum_{\ell\in[d]}\left\|V\left( w^{(r)}_k,\ell\right)\right\|^2=\|w^{(r)}_k\|^2$ where $A \in \mathbb{R}^{m \times \rho d}$ whose elements are distributed in an i.i.d. manner according to $\mathcal{N}(0,1/m)$. For any $\ell\in[d]$ and $k\in[K]$ denote the event $\mathcal{J}^{(r)}_{\ell,k}$, that $\left\|  \mathsf{A} \, V\left(\mathsf{D}^\ast w^{(r)}_k,\ell\right)\right\| > c_1  \left\|V\left( w^{(r)}_k,\ell\right)\right\|$, where $c_1 \in \mathbb{R}_+$, i.e.

\begin{align}
    \mathcal{J}^{(r)}_{\ell,k} \coloneqq \mathbbm{1}\left\{\left\| \, \mathsf{A}V\left(\mathsf{D}^\ast w^{(r)}_k,\ell\right)\right\| > c_1 \left\|V\left(w^{(r)}_k,\ell\right)\right\|\right\}.
\end{align}

To define $\hat{W}^{(r)}_k$, we define $V(\hat{W}^{(r)}_k,\ell)$ for every $\ell\in[d]$. Fix some $\epsilon_1 \geq 0$ and denote
\begin{align}
    [d]_{w^{(r)}_k} = \left\{\ell\in[d]\colon \|V(w^{(r)}_k,\ell)\| \geq \epsilon_1 /2\rho\right\}.
\end{align} 

First, let
\begin{align}
    U^{(r)}_{k,\ell} = \begin{cases}
        V\left(\mathsf{A}\mathsf{D}^\ast w^{(r)}_k,\ell\right), & \text{ if }  \left\|  A \, V\left(\mathsf{D}^\ast w^{(r)}_k,\ell\right)\right\| \leq c_1 \left\|V\left(w^{(r)}_k,\ell\right)\right\|, \\
        \mathbf{0}_m, & \text{ otherwise.} 
    \end{cases}
\end{align}
Then, define
\begin{align}
    \hat{W}^{(r)}_{k,\ell} = \begin{cases}
         \mathbf{0}_{m}&, \text{ if } \ell \notin [d]_{w^{(r)}_k}, \\
         U^{(r)}_{k,\ell} + N^{(r)}_{k,\ell},& \text{ otherwise.}
     \end{cases},
\end{align}
where $N^{(r)}_{k,\ell} \sim \text{Uniform}\left(\mathcal{B}_m(0,\nu)\right)$. Finally, let
\begin{align}
     V(\hat{W}^{(r)}_k,\ell) &= A^\top \hat{W}^{(r)}_{k,\ell},\\
     V(\hat{\overline{W}}^{(R,r)}_k,\ell) &= V\left(\overline{W}_k^{(R),\ast},\ell\right)+\frac{1}{K}A^\top \hat{W}^{(r)}_{k,\ell}
\end{align}
Note that for all $\ell$ such that $ \ell \notin [d]_{w^{(R)}_k}$ we have,
\begin{align}
    &\mathbb{P}\left(\left|  \left\langle   V(\overline{W}^{(R)}-\hat{\overline{W}}^{(R,r)}_k,\ell), V(X_{i,k,r,t-1:t-\rho}) \right\rangle  \right| \geq \frac{\epsilon_1}{2} \right)\nonumber\\
    &\quad=\mathbb{P}\left(\frac{1}{K}\left|  \left\langle   V({w}_k^{(R)},\ell), V(X_{i,k,r,t-1:t-\rho}) \right\rangle  \right| \geq \frac{\epsilon_1}{2}\right)\nonumber\\
    &\quad=0.
\end{align}

Hence,

\begin{align}
   &\mathbb{P}\left(D_{i,k,r,t}\geq\epsilon_1\right)\\
   &\quad\leq\mathbb{P}\left(2\max_{\ell_{i,t}}\left|  \left\langle   V({\overline{W}}^{(R)}-\hat{\overline{W}}_k^{(R,r)},\ell_{i,t}), V(X_{i,k,r,t-1:t-\rho}) \right\rangle \right|\geq\epsilon_1\right)\nonumber\\
   &\quad=\mathbb{P}\left(\max_{\ell_{i,t}}\left|  \left\langle   V({\overline{W}}^{(R)}-\hat{\overline{W}}_k^{(R,r)},\ell_{i,t}), V(X_{i,k,r,t-1:t-\rho}) \right\rangle \right|\geq\frac{\epsilon_1}{2}\right)\nonumber\\
   &\quad\leq\sum_{\ell\in[d]}\mathbb{P}\left(\left|  \left\langle   V({\overline{W}}^{(R)}-\hat{\overline{W}}_k^{(R,r)},\ell), V(X_{i,k,r,t-1:t-\rho}) \right\rangle \right|\geq\frac{\epsilon_1}{2}\right)\nonumber\\
   &\quad=\sum_{\ell\in[d]_{{w}_k^{(R)}}}\mathbb{P}\left(\left|  \left\langle   V({\overline{W}}^{(R)}-\hat{\overline{W}}_k^{(R,r)},\ell), V(X_{i,k,r,t-1:t-\rho}) \right\rangle \right|\geq\frac{\epsilon_1}{2}\right)\nonumber\\
   &\quad\leq\sum_{\ell\in[d]_{{w}_k^{(R)}}}\mathbb{P}\left(\left|\left\langle V\left(\overline{W}^{(R)}-\big(\overline{W}_k^{(R),\ast}+\frac{1}{K}\mathsf{D}^\ast w^{(r)}_k\big),\ell\right),V\left(X_{i,k,r,t-1:t-\rho}\right)\right\rangle\right|\geq\frac{\epsilon_1}{6}\right)\label{splitting_distortion_epsilon_1}\\
   &\qquad+\sum_{\ell\in[d]_{{w}_k^{(R)}}}\mathbb{P}\left(\frac{1}{K}\left|\left\langle V\left(\mathsf{D}^\ast w^{(r)}_k,\ell\right),V\left(X_{i,k,r,t-1:t-\rho}\right)\right\rangle-\left\langle\mathsf{A}V\left(\mathsf{D}^\ast w^{(r)}_k,\ell\right),\mathsf{A}V\left(X_{i,k,r,t-1:t-\rho}\right)\right\rangle\right|\geq\frac{\epsilon_1}{6}\right)\nonumber\\
   &\qquad+\sum_{\ell\in[d]_{{w}_k^{(R)}}}\mathbb{P}\left(\frac{1}{K}\left|\left\langle \mathsf{A}V\left(\mathsf{D}^\ast w^{(r)}_k,\ell\right)-\hat{W}^{(r)}_{k,\ell},\mathsf{A}V\left(X_{i,k,r,t-1:t-\rho}\right)\right\rangle\right|\geq\frac{\epsilon_1}{6}\right)\nonumber\\
   &\quad=\circled{I}+\circled{II}+\circled{III}\label{splitting_distortion_epsilon_2}
\end{align}

where

\begin{itemize}
    \item Equations~(\ref{splitting_distortion_epsilon_1})--(\ref{splitting_distortion_epsilon_2}) follow from the triangle inequality, i.e.,
    
\begin{align}
    &\left\langle V\left(\overline{W}^{(R)}-\hat{\overline{W}}^{(R,r)}_k,\ell\right), V\left(X_{i,k,r,t-1:t-\rho}\right) \right\rangle\\
    &=\left\langle V\left(\overline{W}^{(R)}-\big(\overline{W}_k^{(R),\ast}+\frac{1}{K}\mathsf{D}^\ast w^{(r)}_k\big),\ell\right),V\left(X_{i,k,r,t-1:t-\rho}\right)\right\rangle\\
    &\qquad+\frac{1}{K}\left\langle V\left(\mathsf{D}^\ast w^{(r)}_k,\ell\right),V\left(X_{i,k,r,t-1:t-\rho}\right)\right\rangle-\left\langle\mathsf{A}V\left(\mathsf{D}^\ast w^{(r)}_k,\ell\right),\mathsf{A}V\left(X_{i,k,r,t-1:t-\rho}\right)\right\rangle\\
    &\qquad+\frac{1}{K}\left\langle \mathsf{A}V\left(\mathsf{D}^\ast w^{(r)}_k,\ell\right)-\hat{W}^{(r)}_{k,\ell},\mathsf{A}V\left(X_{i,k,r,t-1:t-\rho}\right)\right\rangle\\
    &\leq\left|\left\langle V\left(\overline{W}^{(R)}-\big(\overline{W}_k^{(R),\ast}+\frac{1}{K}\mathsf{D}^\ast w^{(r)}_k\big),\ell\right),V\left(X_{i,k,r,t-1:t-\rho}\right)\right\rangle\right|\\
    &\qquad+\left|\frac{1}{K}\left\langle V\left(\mathsf{D}^\ast w^{(r)}_k,\ell\right),V\left(X_{i,k,r,t-1:t-\rho}\right)\right\rangle-\left\langle\mathsf{A}V\left(\mathsf{D}^\ast w^{(r)}_k,\ell\right),\mathsf{A}V\left(X_{i,k,r,t-1:t-\rho}\right)\right\rangle\right|\\
    &\qquad+\left|\frac{1}{K}\left\langle \mathsf{A}V\left(\mathsf{D}^\ast w^{(r)}_k,\ell\right)-\hat{W}^{(r)}_{k,\ell},\mathsf{A}V\left(X_{i,k,r,t-1:t-\rho}\right)\right\rangle\right|.
\end{align}

Moreover, using the fact that for every pair of random variables $(X_i)\in\mathbb{R}$ and constants $(c_i)\in\mathbb{R}$, it holds that
\begin{align}
\mathbb{P}\left(\sum_{i}X_i \geq\sum_{i} c_i\right)\leq\sum_{i}\mathbb{P}\left(X_i \geq c_i\right),
\end{align}
the result follows.

\end{itemize}

\begin{itemize}
    \item \textbf{Bounding \circled{I}}: For ease of notation let denote variables 
    \begin{align}
        \Delta_\ell&\coloneqq\left\langle V\left(\overline{W}^{(R)}-\big(\overline{W}_k^{(R),\ast}+\frac{1}{K}\mathsf{D}^\ast w^{(r)}_k\big),\ell\right),V\left(X_{i,k,r,t-1:t-\rho}\right)\right\rangle\\
        \tilde{\Delta}&\coloneqq\overline{W}^{(R)}-\big(\overline{W}_k^{(R),\ast}+\frac{1}{K}\mathsf{D}^\ast w^{(r)}_k\big).
    \end{align}

The Cauchy–Schwarz inequality implies that 
\begin{align}
    \left|\Delta_\ell\right|&\leq\|V(X_{i,k,r,t-1:t-\rho})\|_2\times\left\| V\left(\overline{W}^{(R)}-\big(\overline{W}_k^{(R),\ast}+\frac{1}{K}\mathsf{D}^\ast w^{(r)}_k\big),\ell\right)\right\|_2\\
    &\leq \sqrt{\rho}\times\left\|V\left(\overline{W}^{(R)}-\big(\overline{W}_k^{(R),\ast}+\frac{1}{K}\mathsf{D}^\ast w^{(r)}_k\big),\ell\right)\right\|_2\\
    &\leq\sqrt{\rho}\times\|\tilde{\Delta}\|.
\end{align}

   We show that, with probability one,
\[
\|\widetilde{\Delta}\|_2<\frac{\epsilon_1}{6\sqrt{\rho}},
\]
where
\[
\|\widetilde{\Delta}\|_2
=
\left\|
V\left(
\overbar{W}^{(R)}
-
\left(
\overbar{W}_k^{(R),\ast}
+\frac{1}{K}\mathsf{D}^\ast w_k^{(r)}
\right),
\ell
\right)
\right\|_2.
\]
For $r=R$, we have $\mathsf{D}^{\ast}=\mathrm{I}_d$ and the corresponding term is zero. Hence, for $R>1$, using $\|V(u,\ell)\|_2\leq\|u\|$, it suffices to show that, for every $r<R$,
\[
\left\|
\overbar{w}^{(R)}
-
\left(
\overbar{w}_k^{(R),\ast}
+\frac{1}{K}\mathsf{D}^{\ast}w_k^{(r)}
\right)
\right\|
<
\frac{\epsilon_1}{6\sqrt{\rho}}.
\]
We establish this whenever either of the two conditions in Assumption~\ref{a:KValue} holds.

For $r'\in[r+1:R]$, define
\begin{align}
\xi_{r'}
\coloneqq
\overbar{w}^{(r')}
-
\left(
\overbar{w}_k^{(r'),\ast}
+
\mathsf{D}_{r'}^{\ast}
\left(
\overbar{w}^{(r'-1)}
-
\overbar{w}_k^{(r'-1),\ast}
\right)
\right).
\label{def:xi}
\end{align}

\begin{itemize}[leftmargin=1cm]

\item[1.] \textbf{If Condition~1 holds:}
Assume \eqref{eq:svmKcond1} holds and fix a corresponding $\beta\in(0,1]$. We first show simultaneously by induction that, for $r'=r+1,\ldots,R$,
\begin{align}
\|\xi_{r'}\|
&\leq
\frac{\alpha C^2\big((1+\beta)q_{e,b}^2\big)^{r'-r-1}}{K^2},
\label{eq:inductionBase1}\\
\left\|
\overbar{w}^{(r'-1)}
-
\overbar{w}_k^{(r'-1),\ast}
\right\|^2
&\leq
\frac{C^2\big((1+\beta)q_{e,b}^2\big)^{r'-r-1}}{K^2}.
\label{eq:inductionBase2}
\end{align}

For $r'=r+1$,
\[
\overbar{w}^{(r)}-\overbar{w}_k^{(r),\ast}
=
\frac{1}{K}w_k^{(r)},
\]
and thus
\[
\left\|
\overbar{w}^{(r)}
-
\overbar{w}_k^{(r),\ast}
\right\|^2
\leq
\frac{C^2}{K^2}.
\]
Moreover, by Assumption~2,
\begin{align}
\|\xi_{r+1}\|
\leq
\alpha
\left\|
\overbar{w}^{(r)}
-
\overbar{w}_k^{(r),\ast}
\right\|^2
=
\frac{\alpha}{K^2}\|w_k^{(r)}\|^2
\leq
\frac{\alpha C^2}{K^2}.
\label{eq:r+1}
\end{align}
Hence, \eqref{eq:inductionBase1} and \eqref{eq:inductionBase2} hold for the induction base.

Now assume that \eqref{eq:inductionBase1} and \eqref{eq:inductionBase2} hold for $r'=r+1,\ldots,j$, where $j<R$. From the definition of $\xi_j$, the triangle inequality, Assumption~2, and $\|\mathsf{D}_j^\ast\|_{\mathrm{op}}\leq q_{e,b}$, we have
\begin{align}
\left\|
\overbar{w}^{(j)}
-
\overbar{w}_k^{(j),\ast}
\right\|
&\leq
\left(
q_{e,b}
+
\alpha
\left\|
\overbar{w}^{(j-1)}
-
\overbar{w}_k^{(j-1),\ast}
\right\|
\right)
\left\|
\overbar{w}^{(j-1)}
-
\overbar{w}_k^{(j-1),\ast}
\right\|.
\nonumber
\end{align}
By \eqref{eq:inductionBase2},
\[
\left\|
\overbar{w}^{(j-1)}
-
\overbar{w}_k^{(j-1),\ast}
\right\|
\leq
\frac{
C\big((1+\beta)q_{e,b}^2\big)^{(j-r-1)/2}
}{K}.
\]
The first requirement in \eqref{eq:svmKcond1} therefore gives
\[
\alpha
\left\|
\overbar{w}^{(j-1)}
-
\overbar{w}_k^{(j-1),\ast}
\right\|
<
q_{e,b}\big(\sqrt{1+\beta}-1\big).
\]
Consequently,
\begin{align}
\left\|
\overbar{w}^{(j)}
-
\overbar{w}_k^{(j),\ast}
\right\|^2
&<
(1+\beta)q_{e,b}^2
\left\|
\overbar{w}^{(j-1)}
-
\overbar{w}_k^{(j-1),\ast}
\right\|^2
\nonumber\\
&\leq
\frac{
C^2\big((1+\beta)q_{e,b}^2\big)^{j-r}
}{K^2}.
\label{eq:inductionStep2}
\end{align}
Hence, \eqref{eq:inductionBase2} holds for $r'=j+1$. Furthermore, by Assumption~2 and \eqref{eq:inductionStep2},
\begin{align}
\|\xi_{j+1}\|
&\leq
\alpha
\left\|
\overbar{w}^{(j)}
-
\overbar{w}_k^{(j),\ast}
\right\|^2
\leq
\frac{
\alpha C^2
\big((1+\beta)q_{e,b}^2\big)^{j-r}
}{K^2}.
\label{eq:inductionStep1}
\end{align}
This completes the induction.

Now, for any $r'\in[r+1:R-1]$, denote
\[
\overbar{\mathsf{D}}_{r'}^{\ast}
\coloneqq
\prod_{j=r'+1}^{R}\mathsf{D}^{\ast}_{j},
\]
and let $\overbar{\mathsf{D}}_R^{\ast}=\mathrm{I}_d$. Then,
\begin{align}
&\left\|
\overbar{w}^{(R)}
-
\left(
\overbar{w}_k^{(R),\ast}
+\frac{1}{K}\mathsf{D}^{\ast}w_k^{(r)}
\right)
\right\|
\nonumber\\
&\quad\leq
\sum_{r'=r+1}^{R}
q_{e,b}^{R-r'}\|\xi_{r'}\|
\nonumber\\
&\quad\leq
\frac{\alpha C^2}{K^2}
\sum_{r'=r+1}^{R}
q_{e,b}^{R-r'}
\big((1+\beta)q_{e,b}^2\big)^{r'-r-1}
\nonumber\\
&\quad=
\frac{
\alpha C^2
\left(
q_{e,b}^{R-r}
-
\big((1+\beta)q_{e,b}^2\big)^{R-r}
\right)
}{
K^2
\left(
q_{e,b}
-
(1+\beta)q_{e,b}^2
\right)
}.
\label{eq:Cond1FinalStep}
\end{align}
By the second requirement in \eqref{eq:svmKcond1}, the right-hand side is strictly smaller than $\epsilon_1/(6\sqrt{\rho})$. Hence,
\[
\|\widetilde{\Delta}\|_2<\frac{\epsilon_1}{6\sqrt{\rho}}.
\]

\item[2.] \textbf{If Condition~2 in Appendix~\ref{assumption_sgd_problem_ntp} holds:}
Assume \eqref{eq:svmKcond2} holds. We show simultaneously by induction that, for $r'=r+1,\ldots,R$,
\begin{align}
\|\xi_{r'}\|
&\leq
\frac{
\alpha C^2(1-\alpha\nu)^{r'-r-1}
}{K^2},
\label{eq:C2eq1}\\
\left\|
\overbar{w}^{(r'-1)}
-
\overbar{w}_k^{(r'-1),\ast}
\right\|^2
&\leq
\frac{
C^2(1-\alpha\nu)^{r'-r-1}
}{K^2}.
\label{eq:C2eq2}
\end{align}
The induction base follows exactly as in \eqref{eq:r+1}.

Assume now that \eqref{eq:C2eq1} and \eqref{eq:C2eq2} hold up to some $j<R$. As before,
\[
\left\|
\overbar{w}^{(j)}
-
\overbar{w}_k^{(j),\ast}
\right\|
\leq
\left(
q_{e,b}
+
\alpha
\left\|
\overbar{w}^{(j-1)}
-
\overbar{w}_k^{(j-1),\ast}
\right\|
\right)
\left\|
\overbar{w}^{(j-1)}
-
\overbar{w}_k^{(j-1),\ast}
\right\|.
\]
Since $0\leq\nu<(1-q_{e,b}^2)/\alpha$, we have $1-\alpha\nu>q_{e,b}^2$ and $1-\alpha\nu\leq1$. Hence, by \eqref{eq:C2eq2},
\[
\left\|
\overbar{w}^{(j-1)}
-
\overbar{w}_k^{(j-1),\ast}
\right\|
\leq
\frac{C}{K}.
\]
The first requirement in \eqref{eq:svmKcond2} yields
\[
\frac{\alpha C}{K}
<
\sqrt{1-\alpha\nu}-q_{e,b}.
\]
It follows that
\begin{align}
\left\|
\overbar{w}^{(j)}
-
\overbar{w}_k^{(j),\ast}
\right\|^2
&<
(1-\alpha\nu)
\left\|
\overbar{w}^{(j-1)}
-
\overbar{w}_k^{(j-1),\ast}
\right\|^2
\nonumber\\
&\leq
\frac{
C^2(1-\alpha\nu)^{j-r}
}{K^2}.
\end{align}
Thus, \eqref{eq:C2eq2} holds for $r'=j+1$. Moreover,
\begin{align}
\|\xi_{j+1}\|
&\leq
\alpha
\left\|
\overbar{w}^{(j)}
-
\overbar{w}_k^{(j),\ast}
\right\|^2
\leq
\frac{
\alpha C^2(1-\alpha\nu)^{j-r}
}{K^2}.
\label{eq:C3}
\end{align}
This proves \eqref{eq:C2eq1} and completes the induction.

Finally,
\begin{align}
&\left\|
\overbar{w}^{(R)}
-
\left(
\overbar{w}_k^{(R),\ast}
+\frac{1}{K}\mathsf{D}^{\ast}w_k^{(r)}
\right)
\right\|
\label{eq:Cond2FinalStep_1}\\
&\quad\leq
\sum_{r'=r+1}^{R}
q_{e,b}^{R-r'}\|\xi_{r'}\|
\nonumber\\
&\quad\leq
\frac{\alpha C^2}{K^2}
\sum_{r'=r+1}^{R}
q_{e,b}^{R-r'}
(1-\alpha\nu)^{r'-r-1}
\nonumber\\
&\quad=
\frac{
\alpha C^2
\left(
q_{e,b}^{R-r}
-
(1-\alpha\nu)^{R-r}
\right)
}{
K^2
\left(
q_{e,b}
-
(1-\alpha\nu)
\right)
}.
\label{eq:Cond2FinalStep}
\end{align}
By the second requirement in \eqref{eq:svmKcond2}, the right-hand side is strictly smaller than $\epsilon_1/(6\sqrt{\rho})$. Therefore,
\[
\|\widetilde{\Delta}\|_2<\frac{\epsilon_1}{6\sqrt{\rho}}.
\]
\end{itemize}

\item\textbf{Bounding \circled{II}}: Note that $\left\|\mathsf{D}^\ast w^{(r)}_k\right\|\leq C\times q^{R-r}_{e,b}$  due to the fact that spectral norm of all $\mathsf{D}^\ast_r$ are bounded by $q_{e,b}$ due to
Assumption 2. Then, for every $w^{(r)}_k$ and $\ell\in[d]$,

\begin{align}
    &\mathbb{P}_{\mathsf{A}}\left(\frac{1}{K}\left|\left\langle V\left(\mathsf{D}^\ast w^{(r)}_k,\ell\right),V\left(X_{i,k,r,t-1:t-\rho}\right)\right\rangle-\left\langle\mathsf{A}V\left(\mathsf{D}^\ast w^{(r)}_k,\ell\right),\mathsf{A}V\left(X_{i,k,r,t-1:t-\rho}\right)\right\rangle\right|\geq\frac{\epsilon_1}{6}\right)\nonumber\\
    &\quad\leq 4e^{-\frac{m}{7}\left(\frac{K\epsilon_1}{6\sqrt{\rho}Cq^{R-r}_{e,b}}\right)^2}\nonumber\\
    &\quad\leq\mathcal{O}\left(\frac{1}{dNKT}\right)
\end{align}
where the last step is due to \cite[Lemma 8, Part 2]{gronlund2020near}

\item\textbf{Bounding \circled{III}}:
Fix some $c_1,c_2>0.$ Now, to bound this, we have

\begin{align}
    &\mathbb{P}_{\mathsf{A}}\left(\frac{1}{K}\left|\left\langle\mathsf{A}V\left(\mathsf{D}^\ast w^{(r)}_k,\ell\right)-\hat{W}^{(r)}_{k,\ell},\mathsf{A}V\left(X_{i,k,r,t-1:t-\rho}\right)\right\rangle\right|\geq\frac{\epsilon_1}{6}\right)\\
    &\quad\leq\mathbb{P}_{\mathsf{A}}\left(\left\|\mathsf{A}V\left(X_{i,k,r,t-1:t-\rho}\right)\right\|\geq c_2\sqrt{\rho}\right)+\mathbb{P}_{\mathsf{A}}\left(\left\|  A \, V\left(\mathsf{D}^\ast w^{(r)}_k,\ell\right)\right\| \geq c_1 \left\|V\left(w^{(r)}_k,\ell\right)\right\| \right)\\
    &\qquad+\mathbb{P}_{\hat{W}^{(r)}_{k,\ell}\big|w^{(r)}_{k}}\mathbb{P}_{\mathsf{A}}\left(\frac{1}{K}\left|\left\langle\mathsf{A}V\left(\mathsf{D}^\ast w^{(r)}_k,\ell\right)-\hat{W}^{(r)}_{k,\ell},\mathsf{A}V\left(X_{i,k,r,t-1:t-\rho}\right)\right\rangle\right|\geq\frac{\epsilon_1}{6}\Big|\text{Cond}_1,\text{Cond}_2\right)\nonumber\\
    &\quad=\mathbb{P}_{\mathsf{A}}\left(\left\|\mathsf{A}V\left(X_{i,k,r,t-1:t-\rho}\right)\right\|\geq c_2\sqrt{\rho}\right)+\mathbb{P}_{\mathsf{A}}\left(\left\|  A\mathsf{D}^\ast \, V\left( w^{(r)}_k,\ell\right)\right\| \geq c_1 \left\|V\left(w^{(r)}_k,\ell\right)\right\| \right)\\
    &\qquad+\mathbb{P}_{\mathsf{A},\hat{W}^{(r)}_{k,\ell}}\left(\frac{1}{K}\left|\left\langle\mathsf{A}\mathsf{D}^\ast V\left( w^{(r)}_k,\ell\right)-\hat{W}^{(r)}_{k,\ell},\mathsf{A}V\left(X_{i,k,r,t-1:t-\rho}\right)\right\rangle\right|\geq\frac{\epsilon_1}{6}\Big|\text{Cond}_1,\text{Cond}_2\right)\nonumber\\
    &\quad=\mathbb{P}_{\mathsf{A}}\left(\left\|\mathsf{A}V\left(X_{i,k,r,t-1:t-\rho}\right)\right\|\geq c_2\sqrt{\rho}\right)+\mathbb{P}_{\mathsf{A}}\left(\left\|  A \,\mathsf{D}^\ast V\left( w^{(r)}_k,\ell\right)\right\| \geq c_1 \left\|V\left(w^{(r)}_k,\ell\right)\right\| \right)\nonumber\\
    &\qquad+\mathbb{P}_{\mathsf{A},U\sim\unif(\mathcal{B}_m(0,\nu))}\left(\frac{1}{K}\left|\left\langle U,\mathsf{A}V\left(X_{i,k,r,t-1:t-\rho}\right)\right\rangle\right|\geq\frac{\epsilon_1}{6}\Big|\text{Cond}_1\right)\nonumber\\
    &\quad\leq e^{-{m}\left( c_2^2-1-2\log( c_2)\right)}+e^{-m\left(q^{2(r-R)}_{e,b}c^2_1-1-2\log\left(q^{(r-R)}_{e,b}c_1\right)\right)}+\frac{\sqrt{2(m+1)}}{\sqrt{\pi}}e^{-\left(\frac{m+1}{2}\right)\left(\frac{K\epsilon_1}{6\sqrt{\rho} c_2\nu C}\right)^2}\\
    &\quad\leq\mathcal{O}\left(\frac{1}{dNKT}\right)
\end{align}

where 

 \begin{itemize}
     \item $\text{Cond}_1$ and $\text{Cond}_2$ are $\left\|\mathsf{A}V\left(X_{i,k,r,t-1:t-\rho}\right)\right\|< c_2\sqrt{\rho}$ and $\left\|  A \, V\left(\mathsf{D}^\ast w^{(r)}_k,\ell\right)\right\| < c_1 \left\|V\left(w^{(r)}_k,\ell\right)\right\|$ respectively,
     \item and the last step since
     \begin{align}
        &\mathbb{P}_{\mathsf{A}}\left(\left\|  A \,\mathsf{D}^\ast V\left( w^{(r)}_k,\ell\right)\right\| \geq c_1 \left\|V\left(w^{(r)}_k,\ell\right)\right\| \right)=\mathbb{P}_{\mathsf{A}}\left(\left\|  A \,\mathsf{D}^\ast{\tilde{V}\left( w^{(r)}_k,\ell\right)}\right\| \geq c_1  \right),
     \end{align}
     where $\tilde{V}\left( w^{(r)}_k,\ell\right)=\frac{{V}\left( w^{(r)}_k,\ell\right)}{\left\|{V}\left( w^{(r)}_k,\ell\right)\right\|}$, combining with this fact that $\left\|\mathsf{D}^\ast \tilde{V}\left( w^{(r)}_k,\ell\right)\right\|\leq q^{R-r}_{e,b}$.
 \end{itemize}

\end{itemize}

In the last step, we compute the upper bound on  the mutual information term, for this purpose let $\mathbf{T}^{(r)}_{k,[d]}=\left({T}^{(r)}_{k,1},\cdots,{T}^{(r)}_{k,d}\right)\in\{0,1\}^{d}$, where ${T}^{(r)}_{k,\ell}=\mathrm{1}\left\{\ell\in[d]_{w^{(r)}_k}\right\}$, $\ell\in[d]$. Hence, if $T^{(r)}_{k,\ell}=0$, then 
\begin{align}
V(\hat{\overline{W}}^{(R,r)}_k,\ell)=V\left(\overline{W}^{(R),\ast},\ell\right).
\end{align}
For simplicity and with the abuse of notation, denote
\begin{align}
    [d]_{\mathbf{T}^{(r)}_{k,[d]}} =\left\{\ell\in[d]\colon T^{(r)}_{k,\ell} = 1\right\}.
\end{align}
With this definition, we have

\begin{align}
&\mathbb{I}^A\left(S_{k,r};\hat{\overline{W}}^{(R,r)}_k|\mathfrak{V}_{k,r}\right)\nonumber\\
&\quad\leq \mathbb{I}^A\left({W}_k^{(r)};\hat{\overline{W}}^{(R,r)}_k\Big|\mathfrak{V}_{k,r}\right)\label{Bound_MI_Binary_1}\\
    &\quad= \mathbb{I}^A\left({W}_k^{(r)};\hat{\overbar{W}}_k^{(R,r)}\Big|\mathfrak{U}_{k,r}\right)\nonumber\\
    &\quad= \mathbb{I}^A\left(\left\{V(W^{(r)}_k,\ell)\right\}_{\ell\in[d]};\left\{V(\hat{\overline{W}}^{(R,r)}_k,\ell)\right\}_{\ell\in[d]}\Big|\mathfrak{U}_{k,r}\right)\nonumber\\
    &\quad\leq \mathbb{I}^A\left(\left\{V(W^{(r)}_k,\ell)\right\}_{\ell\in[d]};\left\{\hat{W}^{(r)}_{k,\ell}\right\}_{\ell\in[d]}\Big|\mathfrak{U}_{k,r}\right)\label{Bound_MI_Binary_2}\\
    &\quad\leq \mathbb{I}^A\left(\left\{V(W^{(r)}_k,\ell)\right\}_{\ell\in[d]};\left\{\hat{W}^{(r)}_{k,\ell}\right\}_{\ell\in[d]},\mathbf{T}^{(r)}_{k,[d]}\Big|\mathfrak{U}_{k,r}\right)\nonumber\\
    &\quad=\mathbb{I}^A\left(\left\{V(W^{(r)}_k,\ell)\right\}_{\ell\in[d]};\left\{\hat{W}^{(r)}_{k,\ell}\right\}_{\ell\in[d]}\Big|\mathbf{T}^{(r)}_{k,[d]},\mathfrak{U}_{k,r}\right)+\mathbb{I}^A\left(\left\{V(W^{(r)}_k,\ell)\right\}_{\ell\in[d]};\mathbf{T}^{(r)}_{k,[d]}\Big|\mathfrak{U}_{k,r}\right)\nonumber\\
    &\quad\leq \mathbb{I}^A\left(\left\{V(W^{(r)}_k,\ell)\right\}_{\ell\in[d]};\left\{\hat{W}^{(r)}_{k,\ell}\right\}_{\ell\in[d]}\Big|\mathbf{T}^{(r)}_{k,[d]},\mathfrak{U}_{k,r}\right)+H\left(\mathbf{T}^{(r)}_{k,[d]}\right)\nonumber\\
    &\quad=\mathbb{E}_{\mathbf{T}^{(r)}_{k,[d]}}\left[I^{A,\mathbf{T}^{(r)}_{k,[d]}}\left(\left\{V(W^{(r)}_k,\ell)\right\}_{\ell\in[d]};\left\{\hat{W}^{(r)}_{k,\ell}\right\}_{\ell\in[d]}\Big|\mathfrak{U}_{k,r}\right)\right]+ H\left(\mathbf{T}^{(r)}_{k,[d]}\right)\nonumber\\
    &\quad=\mathbb{E}_{\mathbf{T}^{(r)}_{k,[d]}}\left[h^{A,\mathbf{T}^{(r)}_{k,[d]}}\left(\left\{\hat{W}^{(r)}_{k,\ell}\right\}_{\ell\in[d]_{\mathbf{T}^{(r)}_{k,[d]}}}\Big|\mathfrak{U}_{k,r}\right)-h^{A,\mathbf{T}^{(r)}_{k,[d]}}\left(\left\{\hat{W}^{(r)}_{k,\ell}\right\}_{\ell\in[d]_{\mathbf{T}^{(r)}_{k,[d]}}}\Big|\left\{V(W^{(r)}_k,\ell)\right\}_{\ell\in[d]},\mathfrak{U}_{k,r}\right)\right]\nonumber\\
    &\qquad\qquad\qquad+H\left(\mathbf{T}^{(r)}_{k,[d]}\right)\nonumber\\
    &\quad\leq\mathbb{E}_{\mathbf{T}^{(r)}_{k,[d]}}\left[\sum_{\ell\in[d]_{\mathbf{T}^{(r)}_{k,[d]}}} h^{A,\mathbf{T}^{(r)}_{k,[d]}}\left(\left\{\hat{W}^{(r)}_{k,\ell}\right\}\Big|\mathfrak{U}_{k,r}\right)-h^{A,\mathbf{T}^{(r)}_{k,[d]}}\left(\left\{\hat{W}^{(r)}_{k,\ell}\right\}_{\ell\in[d]_{\mathbf{T}^{(r)}_{k,[d]}}}\Big|\left\{V(W^{(r)}_k,\ell)\right\}_{\ell\in[d]},\mathfrak{U}_{k,r}\right)\right]\nonumber\\
    &\qquad\qquad\qquad+H\left(\mathbf{T}^{(r)}_{k,[d]}\right)\nonumber\\
    &\quad=\mathbb{E}_{\mathbf{T}^{(r)}_{k,[d]}}\left[\sum_{\ell\in[d]_{\mathbf{T}^{(r)}_{k,[d]}}} \left(h^{A,\mathbf{T}^{(r)}_{k,[d]}}\left(\hat{W}^{(r)}_{k,\ell}\Big|\mathfrak{U}_{k,r}\right)-h^{A,\mathbf{T}^{(r)}_{k,[d]}}\left(\hat{W}^{(r)}_{k,\ell}\Big|V(W^{(r)}_k,\ell),\mathfrak{U}_{k,r}\right)\right)\right]+H\left(\mathbf{T}^{(r)}_{k,[d]}\right)\nonumber\\
    &\quad\leq \mathbb{E}_{\mathbf{T}^{(r)}_{k,[d]}}\left[\sum_{\ell\in[d]} T^{(r)}_{k,\ell}\right]\log\left(\text{Volume}\left(\mathcal{B}_m(c_1C+\nu)\right)\right)-\mathbb{E}_{\mathbf{T}^{(r)}_{k,[d]}}\left[\sum_{\ell\in\mathbf{T}^{(r)}_{k,[d]}}h^{A,\mathbf{T}^{(r)}_{k,[d]}}\left(\hat{W}^{(r)}_{k,\ell}\Big|V(W^{(r)}_k,\ell),\mathfrak{U}_{k,r}\right)\right]\nonumber\\
    &\qquad\qquad\qquad+H\left(\mathbf{T}^{(r)}_{k,[d]}\right)\nonumber\\
    &\quad=\mathbb{E}_{\mathbf{T}^{(r)}_{k,[d]}}\left[\sum_{\ell\in[d]} T^{(r)}_{k,\ell}\right]\log\left(\text{Volume}\left(\mathcal{B}_m(c_1C+\nu)\right)\right)-\mathbb{E}_{\mathbf{T}^{(r)}_{k,[d]}}\left[\sum_{\ell\in[d]} T^{(r)}_{k,\ell}\right]\log\left(\text{Volume}\left(\mathcal{B}_m(\nu)\right)\right)+H\left(\mathbf{T}^{(r)}_{k,[d]}\right)\nonumber\\
    &\quad=m\mathbb{E}_{\mathbf{T}^{(r)}_{k,[d]}}\left[\sum_{\ell\in[d]} T^{(r)}_{k,\ell}\right]\log\left(\frac{c_1C+\nu}{\nu}\right)+H\left(\mathbf{T}^{(r)}_{k,[d]}\right)\label{Bound_MI_Binary_3}\\
\end{align}

Here, the information quantities are evaluated
conditionally on the realization of the independent
Gaussian projection matrix $\mathsf{A}$, as indicated
by the superscript $A$. The inequalities
\eqref{Bound_MI_Binary_1} and
\eqref{Bound_MI_Binary_2} follow from the
data processing inequality.

It remains to upper bound ‌$H\left(\mathbf{T}^{(r)}_{k,[d]}\right)$, Denote $N =\min\left(d,\left(\frac{4C \rho}{\theta}\right)^2\right)$. The vector $\mathbf{T}^{(r)}_{k,[d]}$  can have at most $N$ entries equal to one. Hence,

\begin{itemize}
    \item If $N<d/2$, the number of possible choices for $\mathbf{T}^{(r)}_{k,[d]}$ is at most
\begin{align}
     \left|\mathbf{T}^{(r)}_{k,[d]}\right|\triangleq &\binom{d}{0}+\binom{d}{1}+\cdots+\binom{d}{N}\leq  e^{d\, h_b(N/d)}.\nonumber
\end{align}

Since 

\begin{align}
N
=
\min\left\{
d,
\left(\frac{4C\rho}{\theta}\right)^2
\right\},
\nonumber
\end{align}

and $N<d/2$, we have
\begin{align}
N
=
\left(\frac{4C\rho}{\theta}\right)^2.
\nonumber
\end{align}

Hence,
\begin{align}
H\left(\mathbf{T}^{(r)}_{k,[d]}\right)
&\leq
\log\left(\left|\mathbf{T}^{(r)}_{k,[d]}\right|\right)
\nonumber\\
&\leq
d\,
h_b\left(
\frac{1}{d}
\left(\frac{4C\rho}{\theta}\right)^2
\right)
\nonumber\\
&=
\left(\frac{4C\rho}{\theta}\right)^2
\left[
\frac{d}{
\left(\frac{4C\rho}{\theta}\right)^2}
h_b\left(
\frac{
\left(\frac{4C\rho}{\theta}\right)^2
}{d}
\right)
\right]
\nonumber\\
&\leq
\left(\frac{4C\rho}{\theta}\right)^2
\left[
1+
\log\left(1+
\frac{d}{
\left(\frac{4C\rho}{\theta}\right)^2}
\right)
\right].
\nonumber
\end{align}

where we used
\begin{align}
x h_b(1/x)
&=
\log(x)
-
x\left(1-\frac{1}{x}\right)
\log\left(1-\frac{1}{x}\right)
\nonumber\\
&\leq
\log(x)
-
x\left(1-\frac{1}{x}\right)
\left(
1-\frac{1}{1-\frac{1}{x}}
\right)
\nonumber\\
&\leq
\log(x)+1,
\nonumber
\end{align}
by exploiting the inequality
$\log(u)\geq 1-1/u$.

\item If $N\geq d/2$, then
\begin{align}
H\left(\mathbf{T}^{(r)}_{k,[d]}\right)
\leq d
\leq
2N
\leq
2\left(\frac{4C\rho}{\theta}\right)^2.
\nonumber
\end{align}

Therefore, by recalling equation~\eqref{Bound_MI_Binary_3}, we have
\begin{align}
\ref{Bound_MI_Binary_3}
&\leq
m
\min\left\{
d,
\left(\frac{4C\rho}{\theta}\right)^2
\right\}
\log\left(
\frac{c_1C+\nu}{\nu}
\right)
\nonumber\\
&\quad+
\mathcal{O}\left(
\left(\frac{4C\rho}{\theta}\right)^2
\left[
1+
\log\left(1+
\frac{d}{
\left(\frac{4C\rho}{\theta}\right)^2}
\right)
\right]
\right).
\nonumber
\end{align}

This completes the proof.

\end{itemize}






\section{Proof of Theorem \ref{Theorem_Multi_round_SANTP} }\label{Proof_Theorem_Multi_round_SANTP}

The proof of Theorem \ref{Theorem_Multi_round_SANTP} follows the same lossy approach as the proof of Theorem \ref{Theorem_Multi_round_LNTP}. 

At the beginning, let us consider the parameters

\begin{align}
m_3&=\left\lceil 448\left(\frac{C\cdot t}{K\theta}\right)^2\log\left( dNKT\right)\right\rceil,\label{Paramaters_dimension_multi_SANTP_round_1}\\
m_4&=\left\lceil 448\left(\frac{C^2\cdot t\cdot\sqrt{\rho}}{K\theta}\right)^2\log\left(\rho dNKT\right)\right\rceil,\nonumber\\
c_{4,1}=c_{3,1}&=q^{(R-r)}_{e,b}\sqrt{\frac{K^2\theta^2}{C^2t^2}+1},\quad\nu_{4}=\nu_{3}=\frac{t}{2c_1},\nonumber\\
c_{4,2}=c_{3,2}&=\sqrt{\frac{K^2\theta^2}{C^2t^2}+1},\label{Paramaters_dimension_multi_SANTP_round_3}
\end{align}
where $t\geq q^{(R-r)}_{e,b}$, these parameters will be used in the remainder of the proof.

We use the extension of Theorem \ref{lossy_inexpectation_bound_federated} to derive an upper bound on the rate--distortion term $R_{\mathcal{D}}(\cdot)$. Since the attention mechanism depends on the product of the transpose of the Key matrix and the Query matrix, i.e.,
$M^{\top}_{k,2}\!\!\cdot M_{k,1}$, as well as the matrix $M_{k,3}$, we introduce the matrix $M_{k,4}=M^{\top}_{k,2} M_{k,1}$.
Accordingly, we consider the equivalent parameterization $W_k=\left(M_{k,3},M_{k,4}\right)$
for defining the attention problem.  In this setting, the local model for client $k$ at round $r$ is defined as
$W^{(r)}_k=\left(M_{k,3}^{(r)}, M_{k,4}^{(r)}\right)$, while the aggregated model at the $r$-th round is denoted by ${\overline{W}}^{(r)}=\left(\overline{M}^{(r)}_{3},\overline{M}^{(r)}_{4}\right)$. 

We define an auxiliary conditional probability distribution, analogous to \eqref{auxiliary_algorithm_condition_distribution}, that satisfies Condition \ref{def:Rdist} by considering the mini-batches ${\mathcal{B}\{k,r,t}\}$ and the auxiliary variable $\mathfrak{U}_{k,r}$, defined similarly to \eqref{Auxiliary_mini_batch_multi_round} and \eqref{Auxiliary_simplicity_mini_batch_multi_round}, respectively.

Therefore, we can rewrite $g_{M_{k,[2]}}(\vec{x}_k)$ in Theorem~\ref{Theorem_Multi_round_SANTP} as follows:
\begin{align}
    g_{M_{k,4}}(\vec{x}_k)\coloneqq g_{M_{k,[2]}}(\vec{x}_k)
    &\coloneqq\sigma\left(\frac{1}{\sqrt{\rho}}\big\langle M_{k,1}
    \mathcal{E}(x_{k,t-1}),M_{k,2}\mathcal{E}(x_{k,t-j})\big\rangle\right)_{j=1}^{\rho}\nonumber\\
    &=\sigma\left(\frac{1}{\sqrt{\rho}}\big\langle M^{\top}_{k,2}M_{k,1}
    \mathcal{E}(x_{k,t-1}),\mathcal{E}(x_{k,t-j})\big\rangle\right)_{j=1}^{\rho}\nonumber\\
    &=\sigma\left(\frac{1}{\sqrt{\rho}}\big\langle M_{k,4}
    \mathcal{E}(x_{k,t-1}),\mathcal{E}(x_{k,t-j})\big\rangle\right)_{j=1}^{\rho}\nonumber
\end{align}

Let $C_1\left({w_k}^{(r)},t,\vec{\mathbf{x}}_{k,r}\right)$ and $C_2\left({w_k}^{(r)},t,\vec{\mathbf{x}}_{k,r}\right)$ be defined as in \eqref{defition_maximal_ratio_1} and \eqref{defition_maximal_ratio_2}, respectively. This definition implies that, given the model $W^{(r)}_k$, we have

\begin{align}
    p_{W^{(r)}_k}\left(x_{k,r,t}\Big|x_{k,r,t-1},\cdots,x_{k,r,t-\tau}\right)=\frac{\exp\left[\left\langle \mathcal{E}\left(x_{k,r,t}\right),M^{(r)}_{k,3}\sum_{j=1}^{\rho}\left(g_{M^{(r)}_{k,4}}\left(\vec{x}_{k,r}\right)\right)_j\cdot\mathcal{E}\left(x_{k,r,t-j}\right)\right\rangle\right]}{\sum_{\overline{x}_{k,r,t}}\exp\left[\left\langle \mathcal{E}\left(\overline{x}_{k,r,t}\right),M^{(r)}_{k,3}\sum_{j=1}^{\rho}\left(g_{M^{(r)}_{k,4}}\left(\vec{x}_{k,r}\right)\right)_j\cdot\mathcal{E}\left(x_{k,r,t-j}\right)\right\rangle\right]},\nonumber
\end{align}

and

\begin{align}
    C_{1}(w^{(r)}_k,t,\vec{\mathbf{x}}_{k,r})
    &=\frac{\exp\left[\left\langle \mathcal{E}\left(x_{k,r,t}\right),M^{(r)}_{k,3}\sum_{j=1}^{\rho}\left(g_{M^{(r)}_{k,4}}\left(\vec{x}_{k,r}\right)\right)_j\cdot\mathcal{E}\left(x_{k,r,t-j}\right)\right\rangle\right]}{\sum_{\overline{x}_{k,r,t}}\exp\left[\left\langle \mathcal{E}\left(\overline{x}_{k,r,t}\right),M^{(r)}_{k,3}\sum_{j=1}^{\rho}\left(g_{M^{(r)}_{k,4}}\left(\vec{x}_{k,r}\right)\right)_j\cdot\mathcal{E}\left(x_{k,r,t-j}\right)\right\rangle\right]}\nonumber\\
    C_2(w^{(r)}_k,t,\vec{\mathbf{x}}_k)
    &=\frac{\max_{x'_{k,r,t}\neq x_{k,r,t}}\exp\left[\left\langle \mathcal{E}\left(x'_{k,r,t}\right),M^{(r)}_{k,3}\sum_{j=1}^{\rho}\left(g_{M^{(r)}_{k,4}}\left(\vec{x}_{k,r}\right)\right)_j\cdot\mathcal{E}\left(x_{k,r,t-j}\right)\right\rangle\right]}{\sum_{\overline{x}_{k,r,t}}\exp\left[\left\langle \mathcal{E}\left(\overline{x}_{k,r,t}\right),M^{(r)}_{k,3}\sum_{j=1}^{\rho}\left(g_{M^{(r)}_{k,4}}\left(\vec{x}_{k,r}\right)\right)_j\cdot\mathcal{E}\left(x_{k,r,t-j}\right)\right\rangle\right]}\nonumber
\end{align}

 which implies that
 
\begin{align}
    \log\frac{C_1\left(w^{(r)}_k,t,\vec{\mathbf{x}}_{k,r}\right)}{C_2\left(w^{(r)}_k,t,\vec{\mathbf{x}}_{k,r}\right)}&=\left[\left\langle \mathcal{E}\left(x_{k,r,t}\right),M^{(r)}_{k,3}\sum_{j=1}^{\rho}\left(g_{M^{(r)}_{k,4}}\left(\vec{x}_{k,r}\right)\right)_j\cdot\mathcal{E}\left(x_{k,r,t-j}\right)\right\rangle\right]\nonumber\\
    &\quad-\max_{x'_{k,r,t}\neq x_{k,r,t}}\left[\left\langle \mathcal{E}\left(x'_{k,r,t}\right),M_{k,3}\sum_{j=1}^{\rho}\left(g_{M^{(r)}_{k,4}}\left(\vec{x}_{k,r}\right)\right)_j\cdot\mathcal{E}\left(x_{k,r,t-j}\right)\right\rangle\right]\nonumber
\end{align}

For each client $k \in [K]$ and model $i\in\{3,4\}$ at round $r$, we additionally define the variables $\overline{M}_{k,i}^{(r),\ast}$ and $\hat{\overline{M}}_{k,i}^{(R,r)}$ such that $\mathsf{A}_i$ be independent for any $i\in\{3,4\}$. The variable $\overline{M}_{k,i}^{(r),\ast}$ is defined analogously to \eqref{aggrigated_model_in_r-th_round}, while

\begin{align}
    &\hat{\overline{M}}_{k,i}^{(R,r)}
=\overline{M}_{k,i}^{(R),\ast} +\frac{1}{K}\mathsf{A}_i^{\top} \hat{M}_{k,i}^{(r)},\nonumber\\
&\hat{\overline{W}}^{(R,r)}_{k,1}=\left(\hat{\overline{M}}_{k,3}^{(R,r)},\overline{M}_{4}^{(R)}\right),\nonumber\\
&\hat{\overline{W}}^{(R,r)}_k=\left(\hat{\overline{M}}_{k,3}^{(R,r)},\hat{\overline{M}}_{k,4}^{(R,r)}\right).\nonumber
\end{align}

Using the same approach as equation \eqref{lossy_pop_emp_difference_1} we have
\begin{align}
&\mathbb{E}_{\vc{S},\vc{V},\vc{W},\hat{\overbar{W}}_k^{(R,r)}}\left[\operatorname{gen}_{k,\theta}(S_{k,r},{\overbar{W}}^{(R)}) - \operatorname{gen}_{k,0}(S_{k,r},\hat{\overbar{W}}_k^{(R,r)})\right]\nonumber\\
&\leq
    \mathbb{E}_{S_{k,r},\overline{W}^{(R)}}\Bigg[\frac{1}{T}\sum_{t \in [T]}
        \mathbf{1}\left\{\left|\log\frac{C_1(\overline{W}^{(R)},t,\vec{X}_{k,r})}{C_2({\overline{W}}^{(R)},t,\vec{X}_{k,r})}-\log\frac{C_1(\hat{\overline{W}}^{(R,r)}_{k,1},t,\vec{X}_{k,r})}{C_2(\hat{\overline{W}}_{k,1}^{(R,r)},t,\vec{X}_{k,r})}\right|>\frac{\theta}{2}\right\}\nonumber\\
    &\quad+\mathbb{E}_{S_{k,r},\hat{\overline{W}}^{(R,r)}_k}\Bigg[\frac{1}{N_R}\sum_{n\in[N_R]}\Bigg[\frac{1}{T}\sum_{t \in [T]}
    \mathbf{1}\left\{\left|\log\frac{C_{1,n}(\hat{\overline{W}}^{(R,r)}_k,t,\vec{X}^{(n)}_{k,r})}{C_{2,n}(\hat{\overline{W}}^{(R,r)}_k,t,\vec{X}^{(n)}_{k,r})}-\log\frac{C_{1,n}({\overline{W}}^{(R)},t,\vec{X}^{(n)}_{k,r})}{C_{2,n}({\overline{W}}^{(R)},t,\vec{X}^{(n)}_{k,r})}\right|>\frac{\theta}{2}\right\}\nonumber\\
    &\leq\epsilon_{\mathsf{A}_{[3]},k,1}+\epsilon_{\mathsf{A}_{[3]},k,2}\nonumber\\
    &=\epsilon_{\mathsf{A}_{[3]},k}\nonumber,
\end{align}

applying the triangle inequality implies that

\begin{align}
    \epsilon_{\mathsf{A}_{[3]},k,1}&=\mathbb{E}_{S_{k,r},\overline{W}^{(R)}}\Bigg[\frac{1}{T}\sum_{t \in [T]}
        \mathbf{1}\left\{\left|\log\frac{C_1(\overline{W}^{(R)},t,\vec{X}_{k,r})}{C_2({\overline{W}}^{(R)},t,\vec{X}_{k,r})}-\log\frac{C_1(\hat{\overline{W}}^{(R,r)}_{k,1},t,\vec{X}_{k,r})}{C_2(\hat{\overline{W}}_{k,1}^{(R,r)},t,\vec{X}_{k,r})}\right|>\frac{\theta}{4}\right\}\nonumber\\
        &\quad+\mathbb{E}_{S_{k,r},\overline{W}^{(R)}}\Bigg[\frac{1}{T}\sum_{t \in [T]}
        \mathbf{1}\left\{\left|\log\frac{C_1(\hat{\overline{W}}_{k,1}^{(R,r)},t,\vec{X}_{k,r})}{C_2(\hat{\overline{W}}_{k,1}^{(R,r)},t,\vec{X}_{k,r})}-\log\frac{C_1(\hat{\overline{W}}^{(R,r)}_{k},t,\vec{X}_{k,r})}{C_2(\hat{\overline{W}}_{k}^{(R,r)},t,\vec{X}_{k,r})}\right|>\frac{\theta}{4}\right\}\nonumber\\
        &=\circled{I}^1+\circled{II}^1\label{Definition_summarizing_losses_1}
\end{align}

and

\begin{align}
 \epsilon_{\mathsf{A}_{[3]},k,2}&=\mathbb{E}_{S_{k,r},\hat{\overline{W}}^{(R,r)}_k}\Bigg[\frac{1}{N_R}\sum_{n\in[N_R]}\Bigg[\frac{1}{T}\sum_{t \in [T]}
    \mathbf{1}\left\{\left|\log\frac{C_{1,n}({\overline{W}}^{(R)},t,\vec{X}^{(n)}_{k,r})}{C_{2,n}({\overline{W}}^{(R)},t,\vec{X}^{(n)}_{k,r})}-\log\frac{C_{1,n}(\hat{\overline{W}}^{(R,r)}_{k,1},t,\vec{X}^{(n)}_{k,r})}{C_{2,n}(\hat{\overline{W}}^{(R,r)}_{k,1},t,\vec{X}^{(n)}_{k,r})}\right|>\frac{\theta}{4}\right\}\\
    &\quad+\mathbb{E}_{S_{k,r},\hat{\overline{W}}^{(R,r)}_k}\Bigg[\frac{1}{N_R}\sum_{n\in[N_R]}\Bigg[\frac{1}{T}\sum_{t \in [T]}
    \mathbf{1}\left\{\left|\log\frac{C_{1,n}(\hat{\overline{W}}^{(R,r)}_{k,1},t,\vec{X}^{(n)}_{k,r})}{C_{2,n}(\hat{\overline{W}}^{(R,r)}_{k,1},t,\vec{X}^{(n)}_{k,r})}-\log\frac{C_{1,n}(\hat{\overline{W}}^{(R,r)}_{k},t,\vec{X}^{(n)}_{k,r})}{C_{2,n}(\hat{\overline{W}}^{(R,r)}_{k},t,\vec{X}^{(n)}_{k,r})}\right|>\frac{\theta}{4}\right\}\\
    &=\circled{I}^2+\circled{II}^2\label{Definition_summarizing_losses_2}
\end{align}

then we show that 

\begin{align}
    \mathbb{E}_{\mathsf{A}_{[3]}}\left[\epsilon_{\mathsf{A}_{[3]},k}\right]\leq\mathcal{O}\left(\frac{1}{NKT}\right).\nonumber
\end{align}

Similar to Equation~\eqref{Auxiliary_vector_definition_LNTP}, for any
$w^{(r)}_k=\left(M^{(r)}_{k,3},M^{(r)}_{k,4}\right)$ and $i\in\{3,4\}$, we define

\begin{align}
    V\left[\ell,M^{(r)}_{k,i}\right]&\coloneqq\left(M^{(r)}_{k,i}[\ell,1],\cdots,M^{(r)}_{k,i}[\ell,d]\right)\in\mathbb{R}^{d}\nonumber\\
    V\left[M^{(r)}_{k,i},\ell\right]&\coloneqq\left(M^{(r)}_{k,i}[1,\ell],\cdots,M^{(r)}_{k,i}[d,\ell]\right)\in\mathbb{R}^{d}\label{Denote_row_column_def}
\end{align}

\begin{itemize}

\item\textbf{Bounding $\circled{I}^1$ and $\circled{I}^2$}: We have 

\begin{align}
    \circled{I}^1&=\mathbb{E}_{S_{k,r},\overline{W}^{(R)},\hat{\overline{W}}^{(R,r)}_{k,1}}\Bigg[\frac{1}{T}\sum_{t \in [T]}
        \mathbf{1}\left\{\left|\log\frac{C_1(\overline{W}^{(R)},t,\vec{X}_{k,r})}{C_2({\overline{W}}^{(R)},t,\vec{X}_{k,r})}-\log\frac{C_1(\hat{\overline{W}}^{(R,r)}_{k,1},t,\vec{X}_{k,r})}{C_2(\hat{\overline{W}}_{k,1}^{(R,r)},t,\vec{X}_{k,r})}\right|>\frac{\theta}{4}\right\}\nonumber
\end{align}

this implies that

\begin{align}
&D^{(1)}_{k,r,t-1:t-\rho}\coloneqq\left|\log\frac{C_1(\overline{W}^{(R)},t,\vec{X}_{k,r})}{C_2({\overline{W}}^{(R)},t,\vec{X}_{k,r})}-\log\frac{C_1(\hat{\overline{W}}^{(R,r)}_{k,1},t,\vec{X}_{k,r})}{C_2(\hat{\overline{W}}_{k,1}^{(R,r)},t,\vec{X}_{k,r})}\right|\nonumber\\
        &\leq\Bigg|\left\langle \mathcal{E}\left(x_{k,r,t}\right),\sum_{j=1}^{\rho}\left(g_{\overline{M}^{(R)}_{4}}\left(\vec{x}_{k,r}\right)\right)_j\!\!\!\cdot\overline{M}^{(R)}_{3}\cdot\mathcal{E}\left(x_{k,r,t-j}\right)\right\rangle\nonumber\\
        &\qquad -\max_{x'_{k,r,t}\neq x_{k,r,t}}\left[\left\langle \mathcal{E}\left(x'_{k,r,t}\right),\sum_{j=1}^{\rho}\left(g_{\overline{M}^{(R)}_{4}}\left(\vec{x}_{k,r}\right)\right)_j\!\!\!\cdot\overline{M}^{(R)}_{k,3}\cdot\mathcal{E}\left(x_{k,r,t-j}\right)\right\rangle\right]\nonumber\\
        &\qquad-\left\langle \mathcal{E}\left(x_{k,r,t}\right),\sum_{j=1}^{\rho}\left(g_{\hat{\overline{M}}^{(R,r)}_{k,4}}\left(\vec{x}_{k,r}\right)\right)_j\!\!\!\cdot\overline{M}^{(R)}_{3}\cdot\mathcal{E}\left(x_{k,r,t-j}\right)\right\rangle\nonumber\\
        &\qquad +\max_{x'_{k,r,t}\neq x_{k,r,t}}\left[\left\langle \mathcal{E}\left(x'_{k,r,t}\right),\sum_{j=1}^{\rho}\left(g_{\hat{\overline{M}}^{(R,r)}_{k,4}}\left(\vec{x}_{k,r}\right)\right)_j\!\!\!\cdot\overline{M}^{(R)}_{3}\cdot\mathcal{E}\left(x_{k,r,t-j}\right)\right\rangle\right]\Bigg|\nonumber\\
        &\leq\Bigg|\left\langle \mathcal{E}\left(x_{k,r,t}\right),\sum_{j=1}^{\rho}\left(g_{\overline{M}^{(R)}_{4}}\left(\vec{x}_{k,r}\right)-g_{\hat{\overline{M}}^{(R,r)}_{k,4}}\left(\vec{x}_{k,r}\right)\right)_j\!\!\!\cdot\overline{M}^{(R)}_{3}\cdot\mathcal{E}\left(x_{k,r,t-j}\right)\right\rangle\Bigg|\nonumber\\
        &\qquad+\max_{x'_{k,r,t}\neq x_{k,r,t}}\Bigg|\left[\left\langle \mathcal{E}\left(x'_{k,r,t}\right),\sum_{j=1}^{\rho}\left(g_{\overline{M}^{(R)}_{4}}\left(\vec{x}_{k,r}\right)-g_{\hat{\overline{M}}^{(R,r)}_{k,4}}\left(\vec{x}_{k,r}\right)\right)_j\!\!\!\cdot\overline{M}^{(R)}_{3}\cdot\mathcal{E}\left(x_{k,r,t-j}\right)\right\rangle\right]\Bigg|\nonumber\\
        &\leq 2\max_{\overline{x}_{k,r,t}}\left|\left\langle \mathcal{E}\left(\overline{x}_{k,r,t}\right),\sum_{j=1}^{\rho}\left(g_{\overline{M}^{(R)}_{4}}\left(\vec{x}_{k,r}\right)-g_{\hat{\overline{M}}^{(R,r)}_{k,4}}\left(\vec{x}_{k,r}\right)\right)_j\!\!\!\cdot\overline{M}^{(R)}_{3}\cdot\mathcal{E}\left(x_{k,r,t-j}\right)\right\rangle\right|\label{Difference_softmaxes_loss}
\end{align}

Let define

\begin{align}
\vec{\Delta} g^{(R)}(\vec{x}_{k,r})&=g_{\overline{M}^{(R)}_{4}}\left(\vec{x}_{k,r}\right)-g_{\hat{\overline{M}}^{(R,r)}_{k,4}}\left(\vec{x}_{k,r}\right),\nonumber
\end{align}

thus we have

\begin{align}
    \eqref{Difference_softmaxes_loss}&=2\max_{\overline{x}_{k,r,t}}\left|\left\langle \mathcal{E}\left(\overline{x}_{k,r,t}\right),\sum_{j=1}^{\rho}\left(\vec{\Delta} g^{(R)}(\vec{x}_{k,r})\right)_j\!\!\!\cdot\overline{M}^{(R)}_{3}\cdot\mathcal{E}\left(x_{k,r,t-j}\right)\right\rangle\right|\nonumber
    \end{align}

therefore by recalling $\circled{I}^1$ and $\circled{I}^2$ in equations \eqref{Definition_summarizing_losses_1} and \eqref{Definition_summarizing_losses_2} respectively, we obtain
    
    \begin{align}
    \circled{I}^1&=\frac{1}{T}\sum_{t=1}^{T}\mathbb{P}\left(D^{(1)}_{k,r,t-1:t-\rho}>\frac{\theta}{4}\right)\nonumber\\
    &\leq \frac{1}{T}\sum_{t=1}^{T}\mathbb{P}\left(2\max_{\overline{x}_{k,r,t}}\left|\left\langle \mathcal{E}\left(\overline{x}_{k,r,t}\right),\sum_{j=1}^{\rho}\left(\vec{\Delta} g^{(R)}(\vec{x}_{k,r})\right)_j\!\!\!\cdot\overline{M}^{(R)}_{3}\cdot\mathcal{E}\left(x_{k,r,t-j}\right)\right\rangle\right|\geq\frac{\theta}{4}\right)\label{boudning_probability_loss_1}\\
    &=\frac{1}{T}\sum_{t=1}^{T}\mathbb{P}\left(\max_{\mathcal{E}\left(\overline{x}_{k,r,t}\right)}\left|\left\langle \mathcal{E}\left(\overline{x}_{k,r,t}\right),\sum_{j=1}^{\rho}\left(\vec{\Delta} g^{(R)}(\vec{x}_{k,r})\right)_j\!\!\!\cdot\overline{M}^{(R)}_{3}\cdot\mathcal{E}\left(x_{k,r,t-j}\right)\right\rangle\right|\geq\frac{\theta}{8}\right)\nonumber\\
    &\leq\frac{1}{T}\sum_{t=1}^{T}\sum_{\mathcal{E}\left(\overline{x}_{k,r,t}\right)}\mathbb{P}\left(\left|\left\langle \mathcal{E}\left(\overline{x}_{k,r,t}\right),\sum_{j=1}^{\rho}\left(\vec{\Delta} g^{(R)}(\vec{x}_{k,r})\right)_j\!\!\!\cdot\overline{M}^{(R)}_{3}\cdot\mathcal{E}\left(x_{k,r,t-j}\right)\right\rangle\right|\geq\frac{\theta}{8}\right)\label{boudning_probability_loss_2}\\
    &=\frac{1}{T}\sum_{t=1}^{T}\sum_{\mathcal{E}\left(\overline{x}_{k,r,t}\right)}\mathbb{P}\left(\left|\left\langle{\overline{M}_{3}^{(R)}}^{\top}\cdot \mathcal{E}\left(\overline{x}_{k,r,t}\right),\sum_{j=1}^{\rho}\left(\vec{\Delta} g^{(R)}(\vec{x}_{k,r})\right)_j\cdot\mathcal{E}\left(x_{k,r,t-j}\right)\right\rangle\right|\geq\frac{\theta}{8}\right)\\
    &=\frac{1}{T}\sum_{t=1}^{T}\sum_{\ell\in[d]}\mathbb{P}\left(\left|\left\langle V\left[{\overline{M}_{3}^{(R)}}^{\top},\ell\right],\sum_{j=1}^{\rho}\left(\vec{\Delta} g^{(R)}_1(\vec{x}_{k,r})\right)_j\cdot\mathcal{E}\left(x_{k,r,t-j}\right)\right\rangle\right|\geq\frac{\theta}{8}\right)\label{boudning_probability_loss_3}
    \end{align}

where 
\begin{itemize}
    \item Equation \eqref{boudning_probability_loss_1} follows from derived upper bound in \eqref{Difference_softmaxes_loss} on $D^{(1)}_{k,r,t-1:t-\rho}$.

    \item Equation \eqref{boudning_probability_loss_2} is due to union bound.
\end{itemize}

For $i\in\{3,4\}$ we have $V\left[\mathsf{A_{1}}\mathsf{D}^\ast M^{(r)}_{k,i},\ell\right]=\mathsf{A}_i V\left[\mathsf{D}^\ast M^{(r)}_{k,i},\ell\right]$ and $\sum_{\ell\in[d]}\left\|V\left[M^{(r)}_{k,i},\ell\right]\right\|^2= \left\|M^{(r)}_{k,i}\right\|^2$ where $\mathsf{A_{i}}\in\mathbb{R}^{m_i\times d}$ whose elements are distributed in an i.i.d. manner according to $\mathcal{N}\left(0,1/{m_i}\right)$. For any $\ell\in[d], r\in[d], i\in\{3,4\}$ and $k\in[K]$ define the event $\mathcal{J}^{(r)}_{\ell,k}$ as follwoing where $c_{i,1}\in\mathbb{R}_+$, i.e.

\begin{align}
    \mathcal{J}^{(r)}_{k,\ell,i} \coloneqq \mathbbm{1}\left\{\left\|\mathsf{A}_iV\left[\left[\mathsf{D}^\ast M^{(r)}_{k,i}\right]^\top,\ell\right]\right\|\geq c_{i,1}\left\|V\left[\left[M^{(r)}_{k,i}\right]^\top,\ell\right]\right\|\right\}.
\end{align}

For $i\in\{3,4\}$, define
\begin{align}
Q^{(r)}_{k,i} &\coloneqq \left[M^{(r)}_{k,i}\right]^\top,
\nonumber
\end{align}

To define $\hat{M}^{(r)}_{k,i}$, we define $V\left[\hat{M}^{(r)}_{k,i},\ell\right]$ for every $\ell\in[d]$, for $i\in\{3,4\}$ let consider
\begin{align}
    &[d]_{{M}^{(r)}_{k,i}} = \left\{\ell\in[d]:\left\|V\left[{Q}^{(r)}_{k,i},\ell\right]\right\| \geq\frac{\theta}{8\sqrt{\rho}C}\right\}\nonumber
\end{align} 

First, let
\begin{align}
    U^{(r)}_{k,\ell,i} = \begin{cases}
        V\left[\mathsf{A}_i\left[\mathsf{D}^\ast M^{(r)}_{k,i}\right]^\top,\ell\right], & \text{ if }  \left\| \mathsf{A}_i \, V\left[\left[\mathsf{D}^\ast M^{(r)}_{k,i}\right]^\top,\ell\right]\right\| \leq c_{i,1} \left\|V\left[Q^{(r)}_{k,i},\ell\right]\right\|, \\
        \mathbf{0}_{m_i}, & \text{ otherwise.} 
    \end{cases}
\end{align}
Then, define
\begin{align}
    \hat{M}^{(r)}_{k,\ell,i} = \begin{cases}
         \mathbf{0}_{m_i}&, \text{ if } \ell \notin [d]_{M^{(r)}_{k,i}}, \\
         U^{(r)}_{k,\ell,i} + N^{(r)}_{k,\ell,i},& \text{ otherwise.}
     \end{cases},
\end{align}
where $N^{(r)}_{k,\ell,i} \sim \text{Uniform}\left(\mathcal{B}_{m_i}(0,\nu_i)\right)$. Finally, let
\begin{align}
     V\left[\hat{M}^{(r)}_{k,i},\ell\right] &= \mathsf{A}_1^\top \hat{M}^{(r)}_{k,\ell,i},\\
     V\left[\hat{\overline{M}}^{(R,r)}_{k,i},\ell\right] &= V\left[\overline{M}_{k,i}^{(R),\ast},\ell\right]+\frac{1}{K}\mathsf{A}_i^\top \hat{M}^{(r)}_{k,\ell,i}
\end{align}

Define the reconstructed matrices $\hat{Q}^{(r)}_{k,i}$ by
\begin{align}
V[\hat{Q}^{(r)}_{k,i},\ell]
=
\mathsf{A}_i^\top
\hat{M}^{(r)}_{k,\ell,i},
\qquad
i\in\{3,4\}.
\nonumber
\end{align}
Finally, set
\begin{align}
\hat{M}^{(r)}_{k,i}
&\coloneqq
\left[\hat{Q}^{(r)}_{k,i}\right]^{\top}.
\end{align}

Note that for all $\ell\in[d]$ and $t\in[T]$ we have 
    \begin{align}
    &\left|\left\langle V\left[{\overline{M}_{3}^{(R)}}^{\top},\ell\right],\sum_{j=1}^{\rho}\left(\vec{\Delta} g^{(R)}(\vec{x}_{k,r})\right)_j\cdot\mathcal{E}\left(x_{k,r,t-j}\right)\right\rangle\right|\nonumber\\
    &\quad\leq  \left\|V\left[{\overline{M}_{3}^{(R)}}^{\top},\ell\right]\right\|_{\infty} \times \sum_{j=1}^{\rho}\left|\left(\vec{\Delta} g^{(R)}(\vec{x}_{k,r})\right)_j\right|\label{Diffference_lipchittc_querry_loss_3}\\
    &\quad =\left\|V\left[{\overline{M}_{3}^{(R)}}^{\top},\ell\right]\right\|_{\infty} \times \left\|\vec{\Delta} g^{(R)}(\vec{x}_{k,r})\right\|_{1}\nonumber\\
    &\quad\leq\left\|V\left[{\overline{M}_{3}^{(R)}}^{\top},\ell\right]\right\|_{\infty}\times\sum_{j=1}^{\rho}\left| g_{\overline{M}^{(R)}_{4}}\left(\vec{x}_{k,r}\right)-g_{\hat{\overline{M}}^{(R,r)}_{k,4}}\left(\vec{x}_{k,r}\right)\right|\nonumber\\    &\quad\leq\left\|V\left[{\overline{M}_{3}^{(R)}}^{\top},\ell\right]\right\|_{\infty} \times \frac{1}{\sqrt{\rho}}\sum_{j=1}^{\rho}\left|\left\langle \left(\overline{M}^{(R)}_{4}-\hat{\overline{M}}^{(R,r)}_{k,4}\right)\cdot\mathcal{E}(x_{k,r,t-1}),\mathcal{E}({x}_{k,r,t-j})\right\rangle\right|\label{Diffference_lipchittc_querry_loss_6}\\
    &\quad\leq C\sqrt{\rho} \times\max_{\substack{j\in[\rho]\\x_{k,r,t-1} \\ x_{k,r,t-j}}}\left|\left\langle \left(\overline{M}^{(R)}_{4}-\hat{\overline{M}}^{(R,r)}_{k,4}\right)\cdot\mathcal{E}(x_{k,r,t-1}),\mathcal{E}({x}_{k,r,t-j})\right\rangle\right|\nonumber
    \end{align}
    where 

\begin{itemize}
     \item Equation~\eqref{Diffference_lipchittc_querry_loss_3} follows from Hölder's inequality:
\begin{align}
&\left|\left\langle V\left[\left[{\overline{M}_{3}^{(R)}}\right]^\top,\ell\right],\sum_{j=1}^{\rho}\left(\vec{\Delta} g^{(R)}(\vec{x}_{k,r})\right)_j\cdot\mathcal{E}\left(x_{k,r,t-j}\right)\right\rangle\right|\nonumber\\
&\qquad\leq  \left\|V\left[\left[{\overline{M}_{3}^{(R)}}\right]^\top,\ell\right]\right\|_{\infty} \times \sum_{j=1}^{\rho}\left|\left(\vec{\Delta} g^{(R)}(\vec{x}_{k,r})\right)_j\right|.\nonumber
\end{align}

 \item Equation~\eqref{Diffference_lipchittc_querry_loss_6} follows from the 1-Lipschitz continuity of the softmax function with respect to the $\ell_1$ norm, i.e.,
\begin{align}
    \|\sigma(\vec{x})-\sigma(\vec{y})\|_1
    \le
    \|\vec{x}-\vec{y}\|_1,
    \qquad
    \forall\,\vec{x},\vec{y}\in\mathbb{R}^\rho.\nonumber
\end{align}

\item The last inequality follows from the assumption that $\left\|\overline{M}^{(R)}_3\right\|\leq C$.
\end{itemize}

This implies that
\begin{align}
    \eqref{boudning_probability_loss_3}
    &\leq\sum_{\ell\in[d]}\mathbb{P}\!\left(\max_{\substack{j\in[\rho]\\x_{k,r,t-1}\\x_{k,r,t-j}}}
    \left|\left\langle
    \left(\overline{M}^{(R)}_{4}-\hat{\overline{M}}^{(R,r)}_{k,4}\right)\cdot\mathcal{E}(x_{k,r,t-1}), \mathcal{E}(x_{k,r,t-j})
    \right\rangle\right|\geq\frac{\theta}{8C\sqrt{\rho}}\right)\label{Union_bound_type_bound_1}\\
    &\leq d\times\sum_{j=1}^{\rho}\sum_{x',x''}\mathbb{P}\!\left(\left|\left\langle\left(\overline{M}^{(R)}_{4}-\hat{\overline{M}}^{(R,r)}_{k,4}\right)\cdot\mathcal{E}(x'), \mathcal{E}(x'')
    \right\rangle\right|\geq\frac{\theta}{8C\sqrt{\rho}}\right)\nonumber\\
    &= {\rho}d\times\!\!\sum_{\ell''=1}^{d}\sum_{\ell'=1}^{d}\mathbb{P}\left(\left|\left\langle V\left[\left(\overline{M}^{(R)}_{4}-\hat{\overline{M}}^{(R,r)}_{k,4}\right),\ell'\right],V\left[\mathbf{I},\ell''\right]\right|\geq\frac{\theta}{8C\sqrt{\rho}}\right\rangle\right)\label{Union_bound_type_bound_2}\\
    &={\rho}d\times\sum_{\ell''=1}^{d}\sum_{\ell'\in[d]_{{M}^{(r)}_{k,4}}}\mathbb{P}\left(\left|\left\langle V\left[\left(\overline{M}^{(R)}_{4}-\hat{\overline{M}}^{(R,r)}_{k,4}\right),\ell'\right],V\left[\mathbf{I},\ell''\right]\right\rangle\right|\geq\frac{\theta}{8C\sqrt{\rho}}\right)  \label{Union_bound_type_bound_3}
\end{align}

where

\begin{itemize}
    \item Equation \eqref{Union_bound_type_bound_1} last inequality follows from the union bound.

 \item Equation~\eqref{Union_bound_type_bound_2} follows by rewriting the notation according to the definitions in~\eqref{Denote_row_column_def}.

\item Equation \eqref{Union_bound_type_bound_3} holds because for any $(\ell',\ell'')$ with $\ell''\in[d],\ell'\notin [d]_{{M}^{(r)}_{k,4}}$, we have 

\begin{align}
    \mathbb{P}\left(\left|V\left[\left(\overline{M}^{(R)}_{4}-\hat{\overline{M}}^{(R,r)}_{k,4}\right),\ell'\right],V\left[\mathbf{I},\ell''\right]\right|\geq\frac{\theta}{8C\sqrt{\rho}}\right)=0.\nonumber
\end{align}

\end{itemize}

Hence by rewriting the equation \eqref{Union_bound_type_bound_3}, we have

\begin{align}
\eqref{Union_bound_type_bound_3}&={\rho}d\times\sum_{\ell''=1}^{d}\sum_{\ell'\in[d]_{{M}^{(r)}_{k,4}}}\mathbb{P}\left(\left|\left\langle V\left[\left(\overline{M}^{(R)}_{4}-\hat{\overline{M}}^{(R,r)}_{k,4}\right),\ell'\right],V\left[\mathbf{I},\ell''\right]\right\rangle\right|\geq\frac{\theta}{8C\sqrt{\rho}}\right)\nonumber\\  
&\leq\rho d\times\sum_{\ell''=1}^{d}\sum_{\ell'\in[d]_{{M}^{(r)}_{k,4}}}\mathbb{P}\left(\left|\left\langle V\left[\overline{M}^{(R)}_{4}-\left({\overline{M}}_{k,4}^{(R,\ast)}+\frac{1}{K}\mathsf{D}^\ast M^{(r)}_{k,4}\right),\ell'\right],V\left[\mathbf{I},\ell''\right]\right\rangle\right|\geq\frac{\theta}{24C\sqrt{\rho}}\right)\nonumber\\
&\quad+\rho d\times\sum_{\ell''=1}^{d}\sum_{\ell'\in[d]_{{M}^{(r)}_{k,1}}}\mathbb{P}\left(\frac{1}{K}\left|\left\langle V\left[\left(\mathsf{D}^\ast M^{(r)}_{k,4}-\mathsf{A}_4^\top \mathsf{A}_4\mathsf{D}^\ast M^{(r)}_{k,4}\right),\ell'\right],V\left[\mathbf{I},\ell''\right]\right\rangle\right|\geq\frac{\theta}{24C\sqrt{\rho}}\right)\nonumber\\
&\quad+\rho d\times\sum_{\ell''=1}^{d}\sum_{\ell'\in[d]_{{M}^{(r)}_{k,1}}}\mathbb{P}\left(\frac{1}{K}\left|\left\langle V\left[\left(\mathsf{A}_4^\top \mathsf{A}_4\mathsf{D}^\ast M^{(r)}_{k,4}-\hat{M}^{(r)}_{k,4}\right),\ell'\right],V\left[\mathbf{I},\ell''\right]\right\rangle\right|\geq\frac{\theta}{24C\sqrt{\rho}}\right)\label{limited_union}\\
 &=\circled{T}_{1,1}+\circled{T}_{1,2}+\circled{T}_{1,3},\nonumber
\end{align}

where Equation~\eqref{limited_union} follows from the elementary inequality
\begin{align}
    \mathbb{P}\!\left(\sum_{i=1}^{3}X_i \geq a\right)
    \leq
    \sum_{i=1}^{3}\mathbb{P}\!\left(X_i \geq \frac{a}{3}\right).\label{union_prob_randomvariable}
\end{align}

\begin{itemize}
    \item \textbf{Bounding $\circled{T}_{1,1}$:}We show that, under Condition~\ref{eq:svmKcond1}, for every pair $(\ell',\ell'')$ satisfying $\ell'\in[d]$ and $\ell''\in[d]_{M^{(r)}_{k,3}}$, we have

    \begin{align}
        \mathbb{P}\left(\left|\left\langle V\left[\overline{M}^{(R)}_{4}-\left({\overline{M}}_{k,4}^{(R,\ast)}+\frac{1}{K}\mathsf{D}^\ast M^{(r)}_{k,4}\right),\ell'\right],V\left[\mathbf{I},\ell''\right]\right\rangle\right|<\frac{\theta}{24C\sqrt{\rho}}\right)=1\nonumber.
    \end{align}

To this end, we show that for any pair $(\ell',\ell'')$,
\begin{align}
    &\left|\left\langle V\left[\overline{M}^{(R)}_{4}-\left({\overline{M}}_{k,4}^{(R,\ast)}+\frac{1}{K}\mathsf{D}^\ast M^{(r)}_{k,4}\right),\ell'\right],V\left[\mathbf{I},\ell''\right]\right\rangle\right|\nonumber\\
    &\leq \left\|V\left[\overline{M}^{(R)}_{4}-\left({\overline{M}}_{k,4}^{(R,\ast)}+\frac{1}{K}\mathsf{D}^\ast M^{(r)}_{k,4}\right),\ell'\right]\right\|\times\left\|V\left[\mathbf{I},\ell''\right]\right\|\label{Inequality_deterministic_bound_1}\\
    &\leq \left\|V\left[\overline{M}^{(R)}_{1}-\left({\overline{M}}_{k,4}^{(R,\ast)}+\frac{1}{K}\mathsf{D}^\ast M^{(r)}_{k,4}\right),\ell'\right]\right\|\label{Inequality_deterministic_bound_2}\\
    &\leq \left\|\overline{M}^{(R)}_{4}-\left({\overline{M}}_{k,4}^{(R,\ast)}+\frac{1}{K}\mathsf{D}^\ast M^{(r)}_{k,4}\right)\right\|\nonumber
    \end{align}
where 

\begin{itemize}
\item Equation~\eqref{Inequality_deterministic_bound_1} follows from the Cauchy--Schwarz inequality.

\item Equation~\eqref{Inequality_deterministic_bound_2} follows from the fact that, for any $\ell\in[d]$,
\[
\left\|V\left[\mathbf{I},\ell\right]\right\|
\leq 1.
\]
    \end{itemize}

\item[1.] \textbf{If Condition~1 holds:} By repeating the same steps that led to Equation~\eqref{eq:Cond1FinalStep}, we obtain

\begin{align}
    &\left\|\overline{M}^{(R)}_{4}-\left({\overline{M}}_{k,4}^{(R,\ast)}+\frac{1}{K}\mathsf{D}^\ast M^{(r)}_{k,4}\right)\right\|\leq
\frac{
\alpha C^2
\left(
q_{e,b}^{R-r}
-
\big((1+\beta)q_{e,b}^2\big)^{R-r}
\right)
}{
K^2
\left(
q_{e,b}
-
(1+\beta)q_{e,b}^2
\right)
}.\nonumber
\end{align}

This bound holds for any pair $(\ell',\ell'')$, which implies

    \begin{align}
        \left|\left\langle V\left[\overline{M}^{(R)}_{4}-\left({\overline{M}}_{k,4}^{(R,\ast)}+\frac{1}{K}\mathsf{D}^\ast M^{(r)}_{k,4}\right),\ell'\right],V\left[\mathbf{I},\ell''\right]\right\rangle\right|
        &\leq\left\|\overline{M}^{(R)}_{4}-\left({\overline{M}}_{k,4}^{(R,\ast)}+\frac{1}{K}\mathsf{D}^\ast M^{(r)}_{k,4}\right)\right\|\nonumber\\
        &\leq 
\frac{
\alpha C^2
\left(
q_{e,b}^{R-r}
-
\big((1+\beta)q_{e,b}^2\big)^{R-r}
\right)
}{
K^2
\left(
q_{e,b}
-
(1+\beta)q_{e,b}^2
\right)
},\nonumber
    \end{align}
    
where, under Condition~1, namely,
\begin{align}
\max_{r\in[R-1]}
    \left[\frac{
    48\alpha\sqrt{\rho}\,C^3
    \left(
    q_{e,b}^{r}
    -
    \big((1+\beta)q_{e,b}^2\big)^{r}
    \right)
    }{
    \theta
    \left(
    q_{e,b}
    -
    (1+\beta)q_{e,b}^2
    \right)
    }\right]
< K^2,
\nonumber
\end{align}
this term is bounded above by $\frac{\theta}{48C\sqrt{\rho}}$.

\item[2.] \textbf{If Condition~2 holds:}By repeating the same steps as in
Equations~\eqref{eq:Cond2FinalStep_1}--\eqref{eq:Cond2FinalStep}, we obtain
\begin{align}
    &\left\|\overline{M}^{(R)}_{k,4}-\left({\overline{M}}_{k,4}^{(R,\ast)}+\frac{1}{K}\mathsf{D}^\ast M^{(r)}_{k,4}\right)\right\|
    \leq
    \frac{
\alpha C^2
\left(
q_{e,b}^{R-r}
-
(1-\alpha\nu)^{R-r}
\right)
}{
K^2
\left(
q_{e,b}
-
(1-\alpha\nu)
\right)
},
    \nonumber
\end{align}

consequently,

\begin{align}
    &\left|\left\langle V\left[\overline{M}^{(R)}_{4}-\left({\overline{M}}_{k,4}^{(R,\ast)}+\frac{1}{K}\mathsf{D}^\ast M^{(r)}_{k,4}\right),\ell'\right],V\left[\mathbf{I},\ell''\right]\right\rangle\right|\leq
    \frac{
\alpha C^2
\left(
q_{e,b}^{R-r}
-
(1-\alpha\nu)^{R-r}
\right)
}{
K^2
\left(
q_{e,b}
-
(1-\alpha\nu)
\right)
},
    \nonumber
\end{align}

Under Condition~2, namely,

\begin{align}
    \max_{r\in[R-1]}\left[\frac{48
\alpha\sqrt{\rho} C^3
\left(
q_{e,b}^{R-r}
-
(1-\alpha\nu)^{R-r}
\right)
}{
\theta\left(q_{e,b}
-
\left(1-\alpha\nu\right)\right)
}\right]
    < K^2,
    \nonumber
\end{align}

the above term is bounded above by $\frac{\theta}{48C\sqrt{\rho}}$, which completes this part of the proof.

    \item\textbf{Bounding $\circled{T}_{1,2}$:}

By rewriting the term $\circled{T}_{1,2}$ we have,

\begin{align}
    &\!\!\!\!\!\!\!\!\!\!\!\!\!\!\!\!\!\!\!\!\!\rho d\times\sum_{\ell''=1}^{d}\sum_{\ell'\in[d]_{{M}^{(r)}_{k,1}}}\!\!\!\!\!\mathbb{P}\left(\frac{1}{K}\left|\left\langle V\left[\left(\mathsf{D}^\ast M^{(r)}_{k,4}-\mathsf{A}_4^\top \mathsf{A}_4\mathsf{D}^\ast M^{(r)}_{k,4}\right),\ell'\right],V\left[\mathbf{I},\ell''\right]\right\rangle\right|\geq\frac{\theta}{24C\sqrt{\rho}}\right)\nonumber\\
    &\!\!\!\!\!\!\!\!\!\!\!\!\!\!\!\!\!\!\!\!\!=\rho d\times\sum_{\ell''=1}^{d}\sum_{\ell'\in[d]_{{M}^{(r)}_{k,1}}}\!\!\!\!\!\mathbb{P}\left(\left|\left\langle V\left[\left(\mathsf{D}^\ast M^{(r)}_{k,4}-\mathsf{A}_4^\top \mathsf{A}_4\mathsf{D}^\ast M^{(r)}_{k,4}\right),\ell'\right],V\left[\mathbf{I},\ell''\right]\right\rangle\right|\geq\frac{K\theta}{24C\sqrt{\rho}}\right)\nonumber\\
    &\!\!\!\!\!\!\!\!\!\!\!\!\!\!\!\!\!\!\!\!\!=\rho d\times\sum_{\ell''=1}^{d}\sum_{\ell'\in[d]_{{M}^{(r)}_{k,4}}}\!\!\!\!\!\mathbb{P}\left(\left|\left\langle V\left[\mathsf{D}^\ast M^{(r)}_{k,4},\ell'\right],V\left[\mathbf{I},\ell''\right]\right\rangle-\left\langle AV\left[\mathsf{D}^\ast M^{(r)}_{k,4},\ell'\right],AV\left[\mathbf{I},\ell''\right]\right\rangle\right|\geq\frac{K\theta}{24C\sqrt{\rho}}\right)\label{semifina_term}
    \end{align}

where last inequality follows ineaulity 

\begin{align}
    \mathbb{P}\left(\sum_{i}X_i\geq\sum_{i} a_i\right)\leq \sum_{i=1}\mathbb{P}\left(X_i\geq a_i\right)\nonumber.
\end{align}

 Note that $\left\|\mathsf{D}^\ast M^{(r)}_{k,i}\right\|\leq C q^{R-r}_{e,b}, i\in[3]$ due to the fact that the spectral norms of all $\mathsf{D}^\ast_r$ are bounded by $q_{e,b}$ under Assumption~2. Then, by applying \cite[Lemma~8, Part~2]{gronlund2020near}, for any pair $(\ell',\ell'')$ and any $j\in[\rho]$, we have
\begin{align}
&\mathbb{P}\Bigg(
\Big|
\left\langle V\left[\mathsf{D}^\ast M^{(r)}_{k,4},\ell'\right],V\left[\mathbf{I},\ell''\right]\right\rangle-
\left\langle \mathsf{A}_4V\left[\mathsf{D}^\ast M^{(r)}_{k,4},\ell'\right],\mathsf{A}_4V\left[\mathbf{I},\ell''\right]\right\rangle
\geq \frac{K\theta}{24C\sqrt{\rho}}
\Bigg)\nonumber \\
&\qquad\leq 4\exp\!\left(
-\frac{m_1}{7}
\left(\frac{K\theta}{24 C^2q^{R-r}_{e,b}\sqrt{\rho}\,}\right)^2
\right)
\nonumber \\
&\qquad\leq \mathcal{O}\!\left(\frac{1}{\rho d^3NKT}\right).
\nonumber
\end{align}

which implies that
\begin{align}
\eqref{semifina_term} \leq \mathcal{O}\!\left(\frac{1}{NKT}\right),
\nonumber
\end{align}

and consequently,
\begin{align}
\circled{T}_{1,2} \leq \mathcal{O}\!\left(\frac{1}{NKT}\right).
\end{align}

    \item \textbf{Bounding $\circled{T}_{1,3}$:}
For any pair $(\ell',\ell'')$ such that $\ell'\in[d]_{{M}^{(r)}_{k,4}}$ and $\ell''\in[d]$, we have

\begin{align}
    &\!\!\!\!\!\!\!\!\!\!\!\!\!\!\!\!\!\!\!\!\mathbb{P}\left(\frac{1}{K}\left|\left\langle V\left[\left(\mathsf{A}_1^\top \mathsf{A}_1\mathsf{D}^\ast M^{(r)}_{k,4}-\hat{M}^{(r)}_{k,4}\right),\ell'\right],V\left[\mathbf{I},\ell''\right]\right\rangle\right|\geq\frac{\theta}{24C\sqrt{\rho}}\right)\nonumber\\
    &\!\!\!\!\!\!\!\!\!\!\!\!\!\!\!\!\!\!\!\!=\mathbb{P}\left(\left|\left\langle V\left[\mathsf{A}_4\mathsf{D}^\ast M^{(r)}_{k,4},\ell'\right]-\hat{M}^{(r)}_{k,\ell',4},\mathsf{A}_4V\left[\mathbf{I},\ell''\right]\right\rangle\right|\geq\frac{K\theta}{24C\sqrt{\rho}}\right)\nonumber\\
    &\!\!\!\!\!\!\!\!\!\!\!\!\!\!\!\!\!\!\!\!\leq \mathbb{P}_{\mathsf{A}_4}\left(\left\|\mathsf{A}_4V\left[\mathbf{I},\ell''\right]\right\|\geq c_{4,2}\right) +\mathbb{P}\left(\left\| \mathsf{A}_4 \, V\left[\mathsf{D}^\ast M^{(r)}_{k,4},\ell'\right]\right\| \geq c_{4,1} \left\|V\left[M^{(r)}_{k,4},\ell'\right]\right\|\right) \nonumber\\
    &\!\!\!\!\!\!\!\!\!\!\!\!\!\!\!\!\!\!\!\!\quad+\mathbb{E}_{\hat{M}^{(r)}_{k,\ell',4}\sim\unif\left(\mathcal{B}_{m_4}(\cdot,\nu_1)\right)}\mathbb{E}_{\mathsf{A}_4}\left[\mathbb{P}\left(\left|\left\langle V\left[\mathsf{A}_4\mathsf{D}^\ast M^{(r)}_{k,4},\ell'\right]-\hat{M}^{(r)}_{k,\ell',4},\mathsf{A}_4V\left[\mathbf{I},\ell''\right]\right\rangle\right|\geq\frac{K\theta}{24C\sqrt{\rho}}\Big| \text{Cond}_1,\text{Cond}_2\right)\right]\nonumber\\
    &\!\!\!\!\!\!\!\!\!\!\!\!\!\!\!\!\!\!\!\!=\mathbb{P}_{\mathsf{A}_1}\left(\left\|\mathsf{A}_4V\left[\mathbf{I},\ell''\right]\right\|\geq c_{4,2}\right) +\mathbb{P}\left(\left\| \mathsf{A}_4 \, V\left[\mathsf{D}^\ast M^{(r)}_{k,4},\ell'\right]\right\| \geq c_{4,1} \left\|V\left[M^{(r)}_{k,4},\ell'\right]\right\|\right)\nonumber\\
    &\!\!\!\!\!\!\!\!\!\!\!\!\!\!\!\!\!\!\!\!\quad+\mathbb{E}_{U\sim\unif\left(\mathcal{B}_{m_4}(0,\nu_4)\right)}\mathbb{E}_{\mathsf{A}_4}\left[\mathbb{P}\left(\left|\left\langle U,\mathsf{A}_4V\left[\mathbf{I},\ell''\right]\right\rangle\right|\geq\frac{K\theta}{24C\sqrt{\rho}}\Big| \text{Cond}_1\right)\right]\nonumber\\
    &\!\!\!\!\!\!\!\!\!\!\!\!\!\!\!\!\!\!\!\!\leq e^{-0.21m_4\left(c^2_{4,2}-1-2\log(c_{4,2})\right)}+e^{-m_4\left(q^{2(r-R)}_{e,b}c^2_{4,1}-1-2\log\left(q^{(R-r)}_{e,b}c_{4,1}\right)\right)}+\frac{\sqrt{2(m_4+1)}}{\sqrt{\pi}}e^{-\left(\frac{m_4+1}{2}\right)\left(\frac{K\theta}{24\sqrt{\rho} c_{4,2}\nu_4 C}\right)^2}\label{Simlification_thitrm_term_3}\\
    &\!\!\!\!\!\!\!\!\!\!\!\!\!\!\!\!\!\!\!\!\leq\mathcal{O}\left(\frac{1}{\rho NKTd^3}\right)\nonumber
\end{align}

where
\begin{itemize}
    \item $\text{Cond}_1$ and $\text{Cond}_2$ denote the events
    \begin{align}
        \left\|\mathsf{A}_4V\left[\mathbf{I},\ell''\right]\right\|\leq c_{4,2}
    \quad\text{and}\quad
    \left\| \mathsf{A}_4 \, V\left[\mathsf{D}^\ast M^{(r)}_{k,4},\ell'\right]\right\| \leq c_{4,1} \left\|V\left[M^{(r)}_{k,4},\ell'\right]\right\|,\nonumber
    \end{align}
        respectively.

    \item Equation~\eqref{Simlification_thitrm_term_3} follows from
    \cite[Lemma~8, Part~2]{gronlund2020near}.
\end{itemize}

Therefore, this implies that
\begin{align}
    \circled{I}^1,\circled{I}^2
    &\leq
    \circled{T}_{1,1}
    +\circled{T}_{1,2}
    +\circled{T}_{1,3}
    \nonumber\\
    &\leq
    \mathcal{O}\!\left(\frac{1}{NKT}\right).\nonumber
\end{align}

This completes the proof of the bounds on $\circled{I}^1$ and $\circled{I}^2$ in \eqref{Definition_summarizing_losses_1} and \eqref{Definition_summarizing_losses_2}.



\item\textbf{Bounding $\circled{II}^1$ and $\circled{II}^2$:}To bound the terms $\circled{II}^1$ and $\circled{II}^2$, we have

\begin{align}
    \circled{II}^1=\mathbb{E}_{S_{k,r},\overline{W}^{(R)}}\Bigg[\frac{1}{T}\sum_{t \in [T]}
        \mathbf{1}\left\{\left|\log\frac{C_1(\hat{\overline{W}}_{k,1}^{(R,r)},t,\vec{X}_{k,r})}{C_2(\hat{\overline{W}}_{k,1}^{(R,r)},t,\vec{X}_{k,r})}-\log\frac{C_1(\hat{\overline{W}}^{(R,r)}_{k},t,\vec{X}_{k,r})}{C_2(\hat{\overline{W}}_{k}^{(R,r)},t,\vec{X}_{k,r})}\right|>\frac{\theta}{4}\right\},\nonumber
\end{align}

By considering
\begin{align}
    &D^{(2)}_{k,r,t-1:t-\rho}
    \coloneqq
    \left|
    \log\frac{C_1(\hat{\overline{W}}_{k,1}^{(R,r)},t,\vec{X}_{k,r})}
    {C_2(\hat{\overline{W}}_{k,1}^{(R,r)},t,\vec{X}_{k,r})}
    -
    \log\frac{C_1(\hat{\overline{W}}^{(R,r)}_{k},t,\vec{X}_{k,r})}
    {C_2(\hat{\overline{W}}_{k}^{(R,r)},t,\vec{X}_{k,r})}
    \right|,
    \nonumber
\end{align}
and following the same argument as in Equation~\eqref{Difference_softmaxes_loss}, we have
\begin{align}
    &D^{(2)}_{k,r,t-1:t-\rho}
    \leq
    2\max_{\overline{x}_{k,r,t}}
    \left|
    \left\langle
    \mathcal{E}\left(\overline{x}_{k,r,t}\right),
    \sum_{j=1}^{\rho}
    \left(g_{\hat{\overline{M}}^{(R)}_{k,4}}
    \left(\vec{x}_{k,r}\right)\right)_j
    \left(\overline{M}^{(R)}_{3}
    -
    \hat{\overline{M}}^{(R,r)}_{k,3}\right)
    \mathcal{E}\left(x_{k,r,t-j}\right)
    \right\rangle
    \right|,
    \label{Difference_softmaxes_loss_3}
\end{align}
which implies that
\begin{align}
    \!\!\!\!\!\!\!\!\!\!\!\!\!\!\!\!\!\!\!\!\!\!\!\!\!\!\!\!\circled{II}^1
    &=
    \frac{1}{T}\sum_{t=1}^{T}
    \mathbb{P}\left(D^{(1)}_{k,r,t-1:t-\rho}>\frac{\theta}{4}\right)
    \nonumber\\
    &\!\!\!\!\!\!\!\!\!\!\!\!\!\!\!\!\!\!\!\!\!\!\!\!\!\!\!\!\leq
    \frac{1}{T}\sum_{t=1}^{T}\sum_{\ell\in[d]}
    \mathbb{P}\left(
    \left|
    \left\langle
    V\left[
    \left({\overline{M}_{3}^{(R)}}-{\hat{\overline{M}}}_{k,3}^{(R,r)}\right)^{\top},\ell
    \right],
    \sum_{j=1}^{\rho}
    \left(g_{\hat{\overline{M}}^{(R)}_{k,4}}(\vec{x}_{k,r})\right)_j
    \mathcal{E}(x_{k,r,t-j})
    \right\rangle
    \right|
    \geq \frac{\theta}{8}
    \right)
    \nonumber\\
    &\!\!\!\!\!\!\!\!\!\!\!\!\!\!\!\!\!\!\!\!\!\!\!\!\!\!\!\!=
    \frac{1}{T}\sum_{t=1}^{T}
    \sum_{\ell\in[d]_{{M}^{(r)}_{k,3}}}
    \mathbb{P}\left(
    \left|
    \left\langle
    V\left[
    \left({\overline{M}_{3}^{(R)}}-{\hat{\overline{M}}}_{k,3}^{(R,r)}\right)^{\top},\ell
    \right],
    \sum_{j=1}^{\rho}
    \left(g_{\hat{\overline{M}}^{(R)}_{k,4}}(\vec{x}_{k,r})\right)_j
    \mathcal{E}(x_{k,r,t-j})
    \right\rangle
    \right|
    \geq \frac{\theta}{8}
    \right).\nonumber
\end{align}

where the last equality follows from the fact that, for any
$\ell\notin[d]_{{M}^{(r)}_{k,3}}$,
\begin{align}
    \mathbb{P}\left(
    \left|
    \left\langle
    V\left[
    \left({\overline{M}_{3}^{(R)}}-{\hat{\overline{M}}}_{k,3}^{(R,r)}\right)^{\top},\ell
    \right],
    \sum_{j=1}^{\rho}
    \left(g_{\hat{\overline{M}}^{(R)}_{k,4}}(\vec{x}_{k,r})\right)_j
    \mathcal{E}(x_{k,r,t-j})
    \right\rangle
    \right|
    \geq\frac{\theta}{8}
    \right)
    =0.
\end{align}

This follows because
\begin{align}
&\left|
\left\langle
V\left[
\left({\overline{M}_{3}^{(R)}}-{\hat{\overline{M}}}_{k,3}^{(R,r)}\right)^{\top},\ell
\right],
\sum_{j=1}^{\rho}
\left(g_{\hat{\overline{M}}^{(R)}_{k,4}}(\vec{x}_{k,r})\right)_j
\mathcal{E}(x_{k,r,t-j})
\right\rangle
\right|
\nonumber\\
&\leq
\left\|
V\left[
\left({\overline{M}_{3}^{(R)}}-{\hat{\overline{M}}}_{k,3}^{(R,r)}\right)^{\top},\ell
\right]
\right\|_{\infty}
\left\|
\sum_{j=1}^{\rho}
\left(g_{\hat{\overline{M}}^{(R)}_{k,4}}(\vec{x}_{k,r})\right)_j
\mathcal{E}(x_{k,r,t-j})
\right\|_1
\label{Holder_upperbound_1}\\
&\leq
\left\|
V\left[
\left({\overline{M}_{3}^{(R)}}-{\hat{\overline{M}}}_{k,3}^{(R,r)}\right)^{\top},\ell
\right]
\right\|_{\infty}
\label{Holder_upperbound_2}\\
&\leq\left\|V\left[
\left({\overline{M}_{3}^{(R)}}-{\hat{\overline{M}}}_{k,3}^{(R,r)}\right)^{\top},\ell
\right]\right\|\nonumber\\
&\leq \frac{\theta}{8}\nonumber
\end{align}

where
\begin{itemize}
    \item The inequality in Equation~\eqref{Holder_upperbound_1} follows from Hölder's inequality.

    \item Equation~\eqref{Holder_upperbound_2} follows from the fact that
    $\left(g_{\hat{\overline{M}}^{(R)}_{k,4}}(\vec{x}_{k,r})\right)_j$
    are elements of probability distributions, which implies
    \begin{align}
    \left\|
    \sum_{j=1}^{\rho}
    \left(g_{\hat{\overline{M}}^{(R)}_{k,4}}(\vec{x}_{k,r})\right)_j
    \mathcal{E}(x_{k,r,t-j})
    \right\|_1
    \leq 1.\nonumber
    \end{align}
\end{itemize}

Therefore, it suffices to show that, uniformly over any $t\in[T]$ and 
$\ell\in[d]_{{M}^{(r)}_{k,3}}$,
\begin{align}
&\mathbb{P}\left(\left|\left\langle V\left[
\left({\overline{M}_{3}^{(R)}}-{\hat{\overline{M}}}_{k,3}^{(R,r)}\right)^{\top},\ell
\right],
\sum_{j=1}^{\rho}\left(g_{\hat{\overline{M}}^{(R)}_{k,4}}(\vec{x}_{k,r})\right)_j
\mathcal{E}(x_{k,r,t-j})\right\rangle\right|\geq\frac{\theta}{8}\right)\\
&\leq\mathcal{O}\left(\frac{1}{dNKT}\right).\nonumber
\end{align}
For any $t\in[T]$ and $\ell\in[d]$, sames as Equation \eqref{union_prob_randomvariable} we have

\begin{align}
&\!\!\!\!\!\!\!\!\!\!\!\!\!\!\!\!\!\!\!\!\!\!\!\!\!\!\!\!\mathbb{P}\left(\left|\left\langle V\left[
\left({\overline{M}_{3}^{(R)}}-{\hat{\overline{M}}}_{k,3}^{(R,r)}\right)^{\top},\ell
\right],
\sum_{j=1}^{\rho}\left(g_{\hat{\overline{M}}^{(R)}_{k,4}}(\vec{x}_{k,r})\right)_j
\mathcal{E}(x_{k,r,t-j})\right\rangle\right|\geq \frac{\theta}{8}\right)\nonumber\\
&\!\!\!\!\!\!\!\!\!\!\!\!\!\!\!\!\!\!\!\!\!\!\!\!\!\!\!\!\leq \mathbb{P}\left(\left|\left\langle V\left[\left({\overline{M}_{3}^{(R)}}-{\hat{\overline{M}}}_{k,3}^{(R,r)}\right)^{\top},\ell\right],
\sum_{j=1}^{\rho}\left(g_{\hat{\overline{M}}^{(R)}_{k,4}}(\vec{x}_{k,r})\right)_j
\mathcal{E}(x_{k,r,t-j})\right\rangle\right|\geq \frac{\theta}{8}\right)\nonumber\\
&\!\!\!\!\!\!\!\!\!\!\!\!\!\!\!\!\!\!\!\!\!\!\!\!\!\!\!\!\leq\mathbb{P}\left(\left|\left\langle V\left[\left[{\overline{M}_{3}^{(R)}}-\left({\overline{M}}_{k,3}^{(R,\ast)}+\frac{1}{K}\mathsf{D}^\ast M^{(r)}_{k,3}\right)\right]^{\top},\ell\right],
\sum_{j=1}^{\rho}\left(g_{\hat{\overline{M}}^{(R)}_{k,4}}(\vec{x}_{k,r})\right)_j
\mathcal{E}(x_{k,r,t-j})\right\rangle\right|\geq\frac{\theta}{24}\right)\nonumber\\
&\!\!\!\!\!\!\!\!\!\!\!\!\!\!\!\!\!\!\!\!\!\!\!\!\!\!\!\!\quad+\mathbb{P}\left(\frac{1}{K}\left|\left\langle V\left[\left(\mathsf{D}^\ast M^{(r)}_{k,3}-\mathsf{A}_3^\top \mathsf{A}_3\mathsf{D}^\ast M^{(r)}_{k,3}\right)^{\top},\ell\right],
\sum_{j=1}^{\rho}\left(g_{\hat{\overline{M}}^{(R)}_{k,4}}(\vec{x}_{k,r})\right)_j
\mathcal{E}(x_{k,r,t-j})\right\rangle\right|\geq\frac{\theta}{24}\right)\nonumber\\
&\!\!\!\!\!\!\!\!\!\!\!\!\!\!\!\!\!\!\!\!\!\!\!\!\!\!\!\!\quad+\mathbb{P}\left(\frac{1}{K}\left|\left\langle V\left[\left(\mathsf{A}_3^\top \mathsf{A}_3\mathsf{D}^\ast M^{(r)}_{k,3}-\hat{M}^{(r)}_{k,3}\right)^{\top},\ell\right],
\sum_{j=1}^{\rho}\left(g_{\hat{\overline{M}}^{(R)}_{k,4}}(\vec{x}_{k,r})\right)_j
\mathcal{E}(x_{k,r,t-j})\right\rangle\right|\geq\frac{\theta}{24}\right)\nonumber\\
&\!\!\!\!\!\!\!\!\!\!\!\!\!\!\!\!\!\!\!\!\!\!\!\!\!\!\!\!=\circled{T}_{3,1}+\circled{T}_{3,2}+\circled{T}_{3,3},\nonumber
\end{align}

\begin{itemize}
 \item \textbf{Bounding $\circled{T}_{3,1}$:}
To bound the term $\circled{T}_{3,1}$, we repeat the same argument used to bound $\circled{T}_{1,1}$. In particular, for any $t\in[T]$ and any $\ell\notin[d]_{M^{(r)}_{k,3}}$, we have
\begin{align}
    \!\!\!\!\!\!\!\!\!\!\!\!\!\!\!\!\!\!\!\!\!\!\!\!\!\!\!\!\mathbb{P}\left(
    \left|
    \left\langle
    V\left[
    \left(
    \overline{M}_{3}^{(R)}-\left(\overline{M}_{k,3}^{(R,\ast)}
    +\frac{1}{K}\mathsf{D}^\ast M^{(r)}_{k,3} \right)\right)^{\top},\ell\right],\sum_{j=1}^{\rho}
    \left(g_{\hat{\overline{M}}^{(R)}_{k,4}}(\vec{x}_{k,r})\right)_j
    \mathcal{E}(x_{k,r,t-j})\right\rangle\right|\geq\frac{\theta}{16}\right)=0.
    \nonumber
\end{align}
Indeed,
\begin{align}
    &\left|
    \left\langle
    V\left[
    \left(
    \overline{M}_{3}^{(R)}
    -\left(
    \overline{M}_{k,3}^{(R,\ast)}
    +\frac{1}{K}\mathsf{D}^\ast M^{(r)}_{k,3}
    \right)
    \right)^{\top},
    \ell
    \right],
    \sum_{j=1}^{\rho}
    \left(g_{\hat{\overline{M}}^{(R)}_{k,4}}(\vec{x}_{k,r})\right)_j
    \mathcal{E}(x_{k,r,t-j})
    \right\rangle
    \right|
    \nonumber\\
    &\le
    \left\|
    V\left[
    \left(
    \overline{M}_{3}^{(R)}
    -\left(
    \overline{M}_{k,3}^{(R,\ast)}
    +\frac{1}{K}\mathsf{D}^\ast M^{(r)}_{k,3}
    \right)
    \right)^{\top},
    \ell
    \right]
    \right\|
    \left\|
    \sum_{j=1}^{\rho}
    \left(g_{\hat{\overline{M}}^{(R)}_{k,4}}(\vec{x}_{k,r})\right)_j
    \mathcal{E}(x_{k,r,t-j})
    \right\|
    \nonumber\\
    &\le
    \left\|
    V\left[
    \left(
    \overline{M}_{3}^{(R)}
    -\left(
    \overline{M}_{k,3}^{(R,\ast)}
    +\frac{1}{K}\mathsf{D}^\ast M^{(r)}_{k,3}
    \right)
    \right)^{\top},
    \ell
    \right]
    \right\|
    \nonumber\\
    &\le
    \frac{\theta}{24C\sqrt{\rho}}
    \le
    \frac{\theta}{24}.
    \nonumber
\end{align}
The penultimate inequality follows from Conditions~1 and~2. The final inequality follows from the facts that $B\ge1$ and $\rho\ge1$, and hence $B\sqrt{\rho}\ge1$.

\item \textbf{Bounding $\circled{T}_{3,2}$:}
Define
\begin{align}
    \vec{g}_{\hat{\overline{M}}^{(R)}_{k,4}}
    \coloneqq
    \sum_{j=1}^{\rho}
    \left(g_{\hat{\overline{M}}^{(R)}_{k,4}}(\vec{x}_{k,r})\right)_j
    \mathcal{E}(x_{k,r,t-j}).\label{Denoting_Vector_attention}
\end{align}
Then,
\begin{align}
    &\mathbb{P}\left(
    \frac{1}{K}
    \left|
    \left\langle
    V\!\left[
    \left(
    \mathsf{D}^\ast M^{(r)}_{k,3}
    -\mathsf{A}_3^\top\mathsf{A}_3\mathsf{D}^\ast M^{(r)}_{k,3}
    \right)^{\!\top},
    \ell
    \right],
    \vec{g}_{\hat{\overline{M}}^{(R)}_{k,4}}
    \right\rangle
    \right|
    \geq\frac{\theta}{24}
    \right)
    \nonumber\\
    &=
    \mathbb{P}\left(
    \left|
    \left\langle
    V\!\left[
    \left(\mathsf{D}^\ast M^{(r)}_{k,3}\right)^{\!\top},
    \ell
    \right],
    \vec{g}_{\hat{\overline{M}}^{(R)}_{k,4}}
    \right\rangle
    -
    \left\langle
    \mathsf{A}_3
    V\!\left[
    \left(\mathsf{D}^\ast M^{(r)}_{k,3}\right)^{\!\top},
    \ell
    \right],
    \mathsf{A}_3
    \vec{g}_{\hat{\overline{M}}^{(R)}_{k,4}}
    \right\rangle
    \right|
    \geq\frac{K\theta}{24}
    \right)
    \nonumber\\
    &\leq
    4\exp\!\left(
    -\frac{m_3}{7}
    \left(
    \frac{K\theta}
    {24 C q_{e,b}^{R-r}}
    \right)^2
    \right)
    \label{multiplile_concentration_1}\\
    &\leq
    \mathcal{O}\!\left(\frac{1}{dNKT}\right).
    \label{multiplile_concentration_2}
\end{align}

where:
\begin{itemize}
    \item Equation~\eqref{multiplile_concentration_1} follows by applying \cite[Lemma~8, Part~2]{gronlund2020near} together with the bound
    \(\|\mathsf{D}^\ast M^{(r)}_{k,3}\|
    \leq C q_{e,b}^{R-r}\).

    \item Equation~\eqref{multiplile_concentration_2} follows from the choice of the parameter \(m_3\).
\end{itemize}

 \item   \textbf{Bounding $\circled{T}_{3,3}$:}

 By recalling the denotation in \eqref{Denoting_Vector_attention} We have

\begin{align}
    &\!\!\!\!\!\!\!\!\!\!\!\!\!\!\!\!\!\!\!\!\!\!\!\!\!\!\!\!\!\!\!\!\!\!\!\!\!\!\!\!\mathbb{P}\left(\frac{1}{K}\left|\left\langle V\left[\left(\mathsf{A}_3^\top \mathsf{A}_3\mathsf{D}^\ast M^{(r)}_{k,3}-\hat{M}^{(r)}_{k,3}\right)^{\top},\ell\right],
\sum_{j=1}^{\rho}\left(g_{\hat{\overline{M}}^{(R)}_{k,4}}(\vec{x}_{k,r})\right)_j
\mathcal{E}(x_{k,r,t-j})\right\rangle\right|\geq\frac{\theta}{24}\right)\nonumber\\
&\!\!\!\!\!\!\!\!\!\!\!\!\!\!\!\!\!\!\!\!\!\!\!\!\!\!\!\!\!\!\!\!\!\!\!\!\!\!\!\!=\mathbb{P}\left(\left|\left\langle V\left[\left(\mathsf{A}_3^\top \mathsf{A}_3\mathsf{D}^\ast M^{(r)}_{k,3}-\hat{M}^{(r)}_{k,3}\right)^{\top},\ell\right],
\vec{g}_{\hat{\overline{M}}^{(R)}_{k,4}}\right\rangle\right|\geq\frac{K\theta}{24}\right)\nonumber \\
&\!\!\!\!\!\!\!\!\!\!\!\!\!\!\!\!\!\!\!\!\!\!\!\!\!\!\!\!\!\!\!\!\!\!\!\!\!\!\!\!=\mathbb{P}\left(\left|\left\langle V\left[\left(\mathsf{A}_3\mathsf{D}^\ast M^{(r)}_{k,3}\right)^{\top},\ell\right]-\hat{M}^{(r)}_{k,\ell,3},
\mathsf{A}_3\vec{g}_{\hat{\overline{M}}^{(R)}_{k,4}}\right\rangle\right|\geq\frac{K\theta}{24}\right)\nonumber\\
&\!\!\!\!\!\!\!\!\!\!\!\!\!\!\!\!\!\!\!\!\!\!\!\!\!\!\!\!\!\!\!\!\!\!\!\!\!\!\!\!\leq\mathbb{P}_{\mathsf{A}_3}\left(\left\|\mathsf{A}_3\vec{g}_{\hat{\overline{M}}^{(R)}_{k,4}}\right\|\geq c_{3,2}\right)+\mathbb{P}\left(\left\|V\left[\left(\mathsf{A}_3\mathsf{D}^\ast M^{(r)}_{k,3}\right)^{\top},\ell\right]\right\|\geq c_{3,1}\left\|V\left[\left( M^{(r)}_{k,3}\right)^{\top},\ell\right]\right\|\right)\nonumber\\
&\!\!\!\!\!\!\!\!\!\!\!\!\!\!\!\!\!\!\!\!\!\!\!\!\!\!\!\!\!\!\!\!\!\!\!\!\!\!\!\!\quad+\mathbb{E}_{U\sim\unif\left(\mathcal{B}_{m_3}(0,\nu_3)\right)}\mathbb{E}_{\mathsf{A}_1}\left[\mathbb{P}\left(\left|\left\langle U,\mathsf{A}_3\vec{g}_{\hat{\overline{M}}^{(R)}_{k,[2]}}\right\rangle\right|\geq\frac{K\theta}{24C}\Big| \text{Cond}_1\right)\right]\nonumber\\
&\!\!\!\!\!\!\!\!\!\!\!\!\!\!\!\!\!\!\!\!\!\!\!\!\!\!\!\!\!\!\!\!\!\!\!\!\!\!\!\!\leq e^{-0.21m_3\left(c^2_{3,2}-1-2\log(c_{3,2})\right)}+e^{-m_3\left(q^{2(r-R)}_{e,b}c^2_{3,1}-1-2\log\left(q^{(R-r)}_{e,b}c_{3,1}\right)\right)}+\frac{\sqrt{2(m_3+1)}}{\sqrt{\pi}}e^{-\left(\frac{m_3+1}{2}\right)\left(\frac{K\theta}{24 c_{3,2}\nu_1 C}\right)^2}\label{Simlification_third_thitrm_term_3}\\
&\!\!\!\!\!\!\!\!\!\!\!\!\!\!\!\!\!\!\!\!\!\!\!\!\!\!\!\!\!\!\!\!\!\!\!\!\!\!\!\!\leq\mathcal{O}\left(\frac{1}{dNKT}\right)\nonumber
\end{align}

where 

\begin{itemize}
    \item $\text{Cond}_1$ and $\text{Cond}_2$ denote the events
    \begin{align}
       \left\|\mathsf{A}_3\vec{g}_{\hat{\overline{M}}^{(R)}_{k,[2]}}\right\|\leq c_{3,2}
    \quad\text{and}\quad
    \left\|V\left[\left(\mathsf{A}_3\mathsf{D}^\ast M^{(r)}_{k,3}\right)^{\top},\ell\right]\right\|\leq c_{3,1}\left\|V\left[\left( M^{(r)}_{k,3}\right)^{\top},\ell\right]\right\|,
    \end{align}
        respectively.

        \item Equation~\eqref{Simlification_third_thitrm_term_3} follows from
    \cite[Lemma~8, Part~2]{gronlund2020near}.
\end{itemize}

Therefore,
\begin{align}
    \circled{III}^1,\ \circled{III}^2
    &\leq
    \circled{T}_{3,1}
    +\circled{T}_{3,2}
    +\circled{T}_{3,3}
    \nonumber\\
    &\leq
    \mathcal{O}\!\left(\frac{1}{NKT}\right).
    \nonumber
\end{align}
By combining the bounds for $\circled{I}^1,\circled{I}^2$, $\circled{II}^1,\circled{II}^2$, and $\circled{III}^1,\circled{III}^2$, we conclude that the distortion term is bounded by $\mathcal{O}\!\left(\frac{1}{NKT}\right)$.

\end{itemize}

\end{itemize}

\end{itemize}

In the final step, it remains to derive an upper bound on the mutual information term. To this end, define
\[
\mathbf{T}^{(r)}_{k,\{3,4\},[d]}
\coloneqq
\bigl(
\mathbf{T}^{(r)}_{k,3,[d]},
\mathbf{T}^{(r)}_{k,4,[d]}
\bigr).
\]
For each \(i\in\left\{3,4\right\}\), let
\[
\mathbf{T}^{(r)}_{k,i,[d]}
\coloneqq
\bigl(
T^{(r)}_{k,i,1},\ldots,T^{(r)}_{k,i,d}
\bigr)
\in\{0,1\}^d,
\]
where, for every \(\ell\in[d]\),
\[
T^{(r)}_{k,i,\ell}
\coloneqq
\mathbf{1}\{\ell\in[d]_{M^{(r)}_{k,i}}\}.
\] Hence, if $T_{k,i,\ell}=0$, then 
\begin{align}
V(\hat{\overline{M}}^{(R,r)}_{k,i},\ell)=V\left(\overline{M}^{(R),\ast}_{k,i},\ell\right).
\end{align}
For simplicity and with the abuse of notation, denote
\begin{align}
    [d]_{\mathbf{T}^{(r)}_{k,i,[d]}} =\left\{\ell\in[d]\colon T^{(r)}_{k,i,\ell} = 1\right\}.
\end{align}

Therefore we have

\begin{align}
   &\mathbb{I}^{\mathsf{A}_{\{3,4\}}}\left(S_{k,r};\hat{\overline{W}}^{(R,r)}_k|\mathfrak{V}_{k,r}\right)\nonumber\\
   &\quad\leq\mathbb{I}^{\mathsf{A}_{\{3,4\}}}\left({W}_k^{(r)};\hat{\overline{W}}^{(R,r)}_k\Big|\mathfrak{V}_{k,r}\right)\label{Bound_MI_Binary_SANTP_1}\\
    &\quad= \mathbb{I}^{\mathsf{A}_{\{3,4\}}}\left({W}_k^{(r)};\hat{\overbar{W}}_k^{(R,r)}\Big|\mathfrak{U}_{k,r}\right)\nonumber\\
  &\quad=\mathbb{I}^{\mathsf{A}_{\{3,4\}}}\left(\left\{V(W^{(r)}_k,\ell)\right\}_{\ell\in[d]};\left\{V(\hat{\overline{W}}^{(R,r)}_k,\ell)\right\}_{\ell\in[d]}\Big|\mathfrak{U}_{k,r}\right)\nonumber\\
  &\quad=\mathbb{I}^{\mathsf{A}_{\{3,4\}}}\left(\left\{V(Q^{(r)}_{k,i},\ell)\right\}_{\ell\in[d],i\in\{3,4\}};\left\{V\left(\hat{Q}^{(r)}_{k,i},\ell\right)\right\}_{\ell\in[d],i\in\{3,4\}}\Big|\mathfrak{U}_{k,r}\right)\label{Bound_MI_Binary_SANTP_2}\\
  &\quad\leq \mathbb{I}^{\mathsf{A}_{\{3,4\}}}\left(\left\{V(Q^{(r)}_{k,i},\ell)\right\}_{\ell\in[d],i\in\{3,4\}};\left\{\hat{Q}^{(r)}_{k,\ell,i}\right\}_{\ell\in[d],i\in\{3,4\}},\mathbf{T}^{(r)}_{k,\{3,4\},[d]}\Big|\mathfrak{U}_{k,r}\right)\nonumber\\
  &\quad = \mathbb{I}^{\mathsf{A}_{\{3,4\}}}\left(\left\{V(Q^{(r)}_{k,i},\ell)\right\}_{\ell\in[d],i\in\{3,4\}};\left\{\hat{Q}^{(r)}_{k,\ell,i}\right\}_{\ell\in[d],i\in\{3,4\}}\Big|\mathbf{T}^{(r)}_{k,\{3,4\},[d]},\mathfrak{U}_{k,r}\right)\nonumber\\
  &\quad\qquad+\mathbb{I}^{\mathsf{A}_{\{3,4\}}}\left(\left\{V(Q^{(r)}_{k,i},\ell)\right\}_{\ell\in[d],i\in\{3,4\}};\mathbf{T}^{(r)}_{k,\{3,4\},[d]}\Big|\mathfrak{U}_{k,r}\right)\nonumber\\
  &\quad\leq\mathbb{I}^{\mathsf{A}_{\{3,4\}}}\left(\left\{V(Q^{(r)}_{k,i},\ell)\right\}_{\ell\in[d],i\in\{3,4\}};\left\{\hat{M}^{(r)}_{k,\ell,i}\right\}_{\ell\in[d],i\in\{3,4\}}\Big|\mathbf{T}^{(r)}_{k,\{3,4\},[d]},\mathfrak{U}_{k,r}\right)+\mathbb{H}\left(\mathbf{T}^{(r)}_{k,\{3,4\},[d]}\right)\nonumber\\
  &\quad=\mathbb{I}^{\mathsf{A}_{\{3,4\}}}\left(\left\{V\left(Q^{(r)}_{k,3},\ell\right)\right\}_{\ell\in[d]};\left\{\hat{M}^{(r)}_{k,\ell,i}\right\}_{\ell\in[d],i\in\{3,4\}}\Big|\mathbf{T}^{(r)}_{k,\{3,4\},[d]},\mathfrak{U}_{k,r}\right)\nonumber\\
  &\quad\qquad +\mathbb{I}^{\mathsf{A}_{\{3,4\}}}\left(\left\{V\left(Q^{(r)}_{k,4},\ell\right)\right\}_{\ell\in[d]};\left\{\hat{M}^{(r)}_{k,\ell,i}\right\}_{\ell\in[d],i\in\{3,4\}}\Big|\left\{V\left(Q^{(r)}_{k,3},\ell\right)\right\}_{\ell\in[d]},\mathbf{T}^{(r)}_{k,\left\{3,4\right\},[d]},\mathfrak{U}_{k,r}\right)\nonumber\\
  &\qquad\quad+\mathbb{H}\left(\mathbf{T}^{(r)}_{k,\{3,4\},[d]}\right)\nonumber\\
  &\quad=\mathbb{I}^{\mathsf{A}_{\{3,4\}}}\left(\left\{V\left(Q^{(r)}_{k,3},\ell\right)\right\}_{\ell\in[d]};\left\{\hat{M}^{(r)}_{k,\ell,3}\right\}_{\ell\in[d]}\Big|\mathbf{T}^{(r)}_{k,\{3,4\},[d]},\mathfrak{U}_{k,r}\right)\nonumber\\
  &\quad\qquad+\mathbb{I}^{\mathsf{A}_{\{3,4\}}}\left(\left\{V\left(Q^{(r)}_{k,3},\ell\right)\right\}_{\ell\in[d]};\left\{\hat{M}^{(r)}_{k,\ell,4}\right\}_{\ell\in[d]}\Big|\left\{\hat{M}^{(r)}_{k,\ell,3}\right\}_{\ell\in[d]},\mathbf{T}^{(r)}_{k,\{3,4\},[d]},\mathfrak{U}_{k,r}\right)\nonumber\\
  &\qquad\quad+\mathbb{I}^{\mathsf{A}_{\{3,4\}}}\left(\left\{V\left(Q^{(r)}_{k,4},\ell\right)\right\}_{\ell\in[d]};\left\{\hat{M}^{(r)}_{k,\ell,3}\right\}_{\ell\in[d]}\Big|\left\{V\left({Q}^{(r)}_{k,3},\ell\right)\right\}_{\ell\in[d]},\mathbf{T}^{(r)}_{k,\{3,4\},[d]},\mathfrak{U}_{k,r}\right)\nonumber\\
  &\qquad\quad+\mathbb{I}^{\mathsf{A}_{\{3,4\}}}\left(\left\{V\left(Q^{(r)}_{k,4},\ell\right)\right\}_{\ell\in[d]};\left\{\hat{M}^{(r)}_{k,\ell,4}\right\}_{\ell\in[d]}\Big|\left\{V\left(Q^{(r)}_{k,3},\ell\right)\right\}_{\ell\in[d]},\left\{\hat{M}^{(r)}_{k,\ell,3}\right\}_{\ell\in[d]},\mathbf{T}^{(r)}_{k,\left\{3,4\right\},[d]},\mathfrak{U}_{k,r}\right)\nonumber\\
  &\quad\leq 2\sum_{i\in\left\{3,4\right\}}\mathbb{I}^{\mathsf{A}_{[3]}}\left(\left\{V\left(Q^{(r)}_{k,i},\ell\right)\right\}_{\ell\in[d]};\left\{\hat{M}^{(r)}_{k,\ell,i}\right\}_{\ell\in[d]}\Big|\mathbf{T}^{(r)}_{k,\left\{3,4\right\},[d]},\mathfrak{U}_{k,r}\right)+\mathbb{H}\left(\mathbf{T}^{(r)}_{k,\left\{3,4\right\},[d]}\right)\label{Bound_MI_Binary_SANTP_3}\\
  &\quad=2\sum_{i\in\left\{3,4\right\}}\mathbb{E}_{\mathbf{T}^{(r)}_{k,i,[d]}}\left[\mathbb{I}^{\mathsf{A}_{\left\{3,4\right\}},\mathbf{T}^{(r)}_{k,i,[d]}}\left(\left\{V(Q^{(r)}_{k,i},\ell)\right\}_{\ell\in[d]};\left\{\hat{M}^{(r)}_{k,\ell,i}\right\}_{\ell\in[d]}\Big|\mathfrak{U}_{k,r}\right)\right]+\mathbb{H}\left(\mathbf{T}^{(r)}_{k,\left\{3,4\right\},[d]}\right)\nonumber\\
  &\quad=2\sum_{i\in\left\{3,4\right\}}\mathbb{E}_{\mathbf{T}^{(r)}_{k,i,[d]}}\Bigg[h^{\mathsf{A}_{[3]},\mathbf{T}^{(r)}_{k,i,[d]}}\left(\left\{\hat{M}^{(r)}_{k,\ell,i}\right\}_{\ell\in[d]}\Big|\mathfrak{U}_{k,r}\right)\nonumber\\
  &\quad\qquad\qquad\qquad\qquad-h^{\mathsf{A}_{\left\{3,4\right\}},\mathbf{T}^{(r)}_{k,i,[d]}}\left(\left\{\hat{M}^{(r)}_{k,\ell,i}\right\}_{\ell\in[d]}\Big |\left\{V(Q^{(r)}_{k,i},\ell)\right\}_{\ell\in[d]},\mathfrak{U}_{k,r}\right)\Bigg]+\mathbb{H}\left(\mathbf{T}^{(r)}_{k,\left\{3,4\right\},[d]}\right)\nonumber\\
  &\quad\leq 2\sum_{i\in[3]}\mathbb{E}_{\mathbf{T}^{(r)}_{k,i,[d]}}\Bigg[\sum_{\ell\in[d]}h^{\mathsf{A}_{\left\{3,4\right\}},\mathbf{T}^{(r)}_{k,i,[d]}}\left(\hat{M}^{(r)}_{k,\ell,i}\Big|\mathfrak{U}_{k,r}\right)\nonumber\\
  &\quad\qquad\qquad\qquad\qquad-h^{\mathsf{A}_{[3]},\mathbf{T}^{(r)}_{k,i,[d]}}\left(\left\{\hat{M}^{(r)}_{k,\ell,i}\right\}_{\ell\in[d]}\Big |\left\{V(Q^{(r)}_{k,i},\ell)\right\}_{\ell\in[d]},\mathfrak{U}_{k,r}\right)\Bigg]+\mathbb{H}\left(\mathbf{T}^{(r)}_{k,\left\{3,4\right\},[d]}\right)\nonumber\\
  &\quad=2\sum_{i\in\left\{3,4\right\}}\mathbb{E}_{\mathbf{T}^{(r)}_{k,i,[d]}}\Bigg[\sum_{\ell\in[d]}h^{\mathsf{A}_{\left\{3,4\right\}},\mathbf{T}^{(r)}_{k,\left\{3,4\right\},[d]}}\left(\hat{M}^{(r)}_{k,\ell,i}\Big|\mathfrak{U}_{k,r}\right)-h^{\mathsf{A}_{\left\{3,4\right\}},\mathbf{T}^{(r)}_{k,\left\{3,4\right\},[d]}}\left(\hat{M}^{(r)}_{k,\ell,i}\Big|V(Q^{(r)}_{k,i},\ell),\mathfrak{U}_{k,r}\right)\Bigg]\nonumber\\
  &\quad\qquad\qquad\qquad\qquad+\mathbb{H}\left(\mathbf{T}^{(r)}_{k,\left\{3,4\right\},[d]}\right)\nonumber\\
  &\quad\leq 2\sum_{i\in\left\{3,4\right\}}\Bigg[\left[\mathbb{E}_{\mathbf{T}^{(r)}_{k,i,[d]}}\left[\sum_{\ell\in[d]}T^{(r)}_{k,i,\ell}\right]\log\left(\text{Volume}\left(\mathcal{B}_{m_i}\left(c_{i,1}C+\nu_i\right)\right)\right)\right]\nonumber\\
  &\qquad\qquad\qquad-\mathbb{E}_{\mathbf{T}^{(r)}_{k,i,[d]}}\left[\sum_{\ell\in\mathbf{T}^{(r)}_{k,i,[d]}}T^{(r)}_{k,i,\ell}h^{\mathsf{A}_{\left\{3,4\right\}},\mathbf{T}^{(r)}_{k,\left\{3,4\right\},[d]}}\left(\hat{M}^{(r)}_{k,\ell,i}\Big|V(M^{(r)}_{k,i},\ell),\mathfrak{U}_{k,r}\right)\right]\Bigg]+\mathbb{H}\left(\mathbf{T}^{(r)}_{k,\left\{3,4\right\},[d]}\right)\nonumber\\
  &\quad=2\sum_{i\in[3]}\Bigg[\left[\mathbb{E}_{\mathbf{T}^{(r)}_{k,i,[d]}}\left[\sum_{\ell\in[d]}T^{(r)}_{k,i,\ell}\right]\log\left(\text{Volume}\left(\mathcal{B}_{m_i}\left(c_{i,1}C+\nu_i\right)\right)\right)\right]\nonumber\\
  &\qquad\qquad\qquad-\left[\mathbb{E}_{\mathbf{T}^{(r)}_{k,i,[d]}}\left[\sum_{\ell\in[d]}T^{(r)}_{k,i,\ell}\right]\log\left(\text{Volume}\left(\mathcal{B}_{m_i}\left(\nu_i\right)\right)\right)\right]\Bigg]+\mathbb{H}\left(\mathbf{T}^{(r)}_{k,\left\{3,4\right\},[d]}\right)\nonumber\\
  &\quad\leq 2\sum_{i\in\left\{3,4\right\}} \left[m_i\mathbb{E}_{\mathbf{T}^{(r)}_{k,i,[d]}}\left[\sum_{\ell\in[d]}T^{(r)}_{k,i,\ell}\right]\log\left(\frac{c_{i,1}C+\nu_i}{\nu_i}\right)+\mathbb{H}\left(\mathbf{T}^{(r)}_{k,i,[d]}\right)\right]\\
  &\quad\leq 2\sum_{i\in\left\{3,4\right\}}\left[\left(m_i+1\right)\left(\frac{8\rho C^2}{\theta}\right)\log\left(\frac{c_{i,1}C+\nu_i}{\nu_i}\right)+\frac{1}{2}\right]\nonumber
\end{align}

where 

\begin{itemize}
    \item Equations~\eqref{Bound_MI_Binary_SANTP_1} and~\eqref{Bound_MI_Binary_SANTP_2} follow from the data processing inequality.

\item The inequality in~\eqref{Bound_MI_Binary_SANTP_3} follows from the fact that, for any $k\in[K]$, $r\in[R]$, and $i\in[3]$, conditioned on 
$V\!\left(\{M^{(r)}_{k,i,\ell}\}_{\ell\in[d]}\right)$, the random variable
$\hat{M}^{(r)}_{k,i,[d]}$
is independent of
$V\!\left(\{M^{(r)}_{k,j,\ell}\}_{\ell\in[d]}\right)$ for all $j\neq i$.
\end{itemize}

It remains to upper bound $\mathbb{H}\left(\mathbf{T}^{(r)}_{k,i,[d]}\right)$. This follows the same argument as bounding $\mathbb{H}\left(\mathbf{T}^{(r)}_{k,[d]}\right)$ in Equation~\eqref{Bound_MI_Binary_3}. Since the argument is identical for each $i\in[3]$, we omit the details.


\section{Proof of Lemma \ref{Hoeffdinf_Lemma}}\label{Proof_Hoeffding_employing}

To prove Lemma~\ref{Hoeffdinf_Lemma}, we follow the same steps as
\eqref{multi_round_generalization_bound_1}--\eqref{multi_round_generalization_bound_10}.
In particular, analogously to \eqref{bounding_moment_10}, we obtain
\begin{align}
    &\mathbb{E}_{S_{[K],[R]},\overline{W}^{(R)}}\left[\gen(S_{[K],[R]}, \overline{W}^{(R)})\right]\nonumber\\
    &\qquad\leq
    \frac{1}{KR}\sum_{k=1}^{K}\sum_{r=1}^{R}\frac{1}{\lambda_{k,r}}\left[I(S_{k,r};W^{(r)}_k|\overline{W}^{(r-1)})+\mathbb{E}_{\overline{W}^{(R)}}\left[\log\mathbb{E}_{S_{k,r}}\left[e^{\phi(S_{k,r},S'_{k,r},\overline{W}^{(R)})}\right]\right]\right]
\label{Hoeffding_MGF_Gen_error}
\end{align}

Thus, it remains to bound the MGF term on the right-hand side of
\eqref{Hoeffding_MGF_Gen_error}. To this end, define
\begin{align}
    Y_{k,r,t}
    &\coloneqq\frac{1}{T}\log
    \frac{1}{
        p_{\overline{w}^{(R)}}(X_{k,r,t}|X_{k,r,t-1,r},\ldots,X_{k,r,t-\rho})
    },
    \nonumber\\
    Y_{k,r,t}^{(n)}
    &\coloneqq\frac{1}{T}\log
    \frac{1}{
        p_{\overline{w}^{(R)}}(X_{k,r,t}^{(n)}|X_{k,r,t-1}^{(n)},\ldots,X_{k,r,t-\rho}^{(n)})
    }.\nonumber
\end{align}
Then,
\begin{align}
    \ell(\mathbf{s}_{k,r},\overline{w}^{(R)})
    &=
    \frac{1}{T}\sum_{t=1}^{T}
    \log
    \frac{1}{
        p_{\overline{w}^{(R)}}(x_{k,r,t}| x_{k,r,t-1},\ldots,x_{k,r,t-\rho})
    }\nonumber\\
    &=\sum_{t=1}^{T}Y_{k,r,t}.
    \nonumber
\end{align}
Consequently,
\begin{align}
    &\log
    \mathbb{E}_{\mathbf{S}_{k,r}}\Bigg[
        e^{{\frac{\lambda}{N}}
        \left(
            \mathbb{E}_{\mathbf{S}'}
            \left[
                \sum_{t=1}^{T}Y_{k,r,t}'
            \right]
            -
            \sum_{t=1}^{T}Y^{(n)}_{k,r,t}
        \right)}
    \Bigg]\leq
    \frac{
        \lambda^2
        (2B+\log d)^2
        \tilde{\tau}_{k,\min}
    }{
        8TN^2
    }
    \label{Hoeefding_inequality_MGF_Bounding}
\end{align}
where
\begin{itemize}

\item $\tilde{\tau}_{k,\min}$ is defined in \eqref{new_tau-min}.

    \item \eqref{Hoeefding_inequality_MGF_Bounding} follows by applying the
    McDiarmid-type inequality in Lemma~\ref{lem:window-loss-concentration}.
    Indeed, since $Y_{k,r,t}\geq0$ for every $k\in[K],r\in[R],t\in[T]$, the function
    \begin{align}
        g(Y_{k,r,1},\ldots,Y_{k,r,T})
        \coloneqq\frac{1}{N}
        \sum_{t=1}^{T}Y_{k,r,t}
        \nonumber
    \end{align}
    satisfies
    \begin{align}
        g(Y_{k,r,1},\ldots,Y_{k,r,T})
        -
        g(Y'_{k,r,1},\ldots,Y'_{k,r,T})&\leq\frac{1}{N}
        \sum_{t=1}^{T}
        Y_{k,r,t}\,
        \mathrm{1}\{Y_{k,r,t}\neq Y'_{k,r,t}\}\nonumber\\
        &\leq
        \left(\frac{2B+\log d}{NT}\right)
        \sum_{t=1}^{T}
        \mathrm{1}\{Y_{k,r,t}\neq Y'_{k,r,t}\}
        \label{max_bound_prob_y_t}
    \end{align}
    where \eqref{max_bound_prob_y_t} follows from the uniform upper bound on
    $Y_t$ established in Lemma~\ref{max_probability_conditional_markov}.
\end{itemize}

\section{Proof of Lemma \ref{lem:window-loss-concentration}}\label{Proof:lem:window-loss-concentration}

We begin with the McDiarmid-type inequality for dependent random
variables established in Lemma~\ref{Mdiarmid_type_lemma}. For a
partition $\widehat Y$ with associated mixing matrix $\Gamma$,
the lemma gives
\[
\log\mathbb E\left[
e^{\lambda(f(Y)-\mathbb Ef(Y))}\right]
\le
\frac{\lambda^2}{8}\|\Gamma C(c)\|_2^2.
\]
To obtain the desired concentration bound, it is therefore
enough to construct a Marton coupling for the partitioned loss
sequence and control the norm of its mixing matrix.

Fix $\delta\in(0,1)$ such that $\tilde{\tau}(\delta)<\infty$,
and let $L=\tilde{\tau}(\delta)$. We partition $Y$ into
consecutive blocks $\widehat Y_1,\ldots,\widehat Y_{\widehat T}$
of length $L$, allowing the final block to be shorter.
Consider any $i<\widehat T$ and two admissible histories
in the sense of Definition~\ref{Marton_coupling}, which
agree through block $i-1$ but may differ at block $i$.
Writing $s=iL$ for the endpoint of that block, the two
histories may induce different conditional distributions
of $Z_s$. Nevertheless, the Markov property ensures that,
conditional on the state $Z_s$, the subsequent evolution
does not depend on the preceding losses. We may therefore
couple the conditional distributions of the two states
and extend this coupling to their future trajectories,
while preserving the conditional laws required by
Definition~\ref{Marton_coupling}.

To construct the coupling, we work at successive block
boundaries $\{s+kL\}_{k\ge1}$. Given the pair of states
at a boundary, we use a maximal coupling of their
$L$-step transition distributions. By the definition
of $\tilde{\tau}(\delta)$, the two endpoints disagree
with probability at most $\delta$, uniformly over
their initial states. Conditional on these endpoints,
the intervening trajectories are sampled from the
corresponding Markov bridge distributions, ensuring
that each copy retains its prescribed trajectory law.
Whenever the two states meet, we couple their subsequent
evolution identically, so that they remain equal.
Consequently, disagreement at the next boundary can
occur only if the current states differ, and even
then its conditional probability is at most $\delta$.
More precisely,
\[
\mathbb P\!\left(
Z_{s+(r+1)L}\neq Z'_{s+(r+1)L}
\mid Z_{s+rL},Z'_{s+rL}
\right)
\le
\delta\,
\mathbf 1_{\{Z_{s+rL}\neq Z'_{s+rL}\}}.
\]
Taking expectations on both sides and applying the
tower property, we obtain
\[
\mathbb P(Z_{s+(r+1)L}\neq Z'_{s+(r+1)L})
\le
\delta\,
\mathbb P(Z_{s+rL}\neq Z'_{s+rL}).
\]
Since the probability of disagreement at the initial
boundary is at most one, iterating this inequality
gives
\begin{equation}
\mathbb P(Z_{s+rL}\neq Z'_{s+rL})\le\delta^r.
\label{eq:window-state-coupling}
\end{equation}

We next relate this decay of state disagreement to
the coupling of the loss blocks. Consider block
$i+k$ for some $k\ge2$, whose first time index is
$a=s+(k-1)L+1$. The state immediately preceding
this block is
\[
Z_{a-1}
=(X_{a-1},\ldots,X_{a-\rho}),
\]
which contains the full context needed to generate
its first loss. If $Z_{a-1}=Z'_{a-1}$, the two
trajectories can be continued identically, and all
losses within block $i+k$ consequently agree.
Combining this observation with
\eqref{eq:window-state-coupling} yields
\[
\mathbb{P}(\widehat Y_{i+k}\neq\widehat Y'_{i+k})
\le
\mathbb{P}(Z_{s+(k-1)L}\neq Z'_{s+(k-1)L})
\le\delta^{k-1}.
\]
Because this estimate holds uniformly over the
admissible conditional histories, the mixing
matrix in Definition~\ref{Marton_coupling}
satisfies
\[
\Gamma_{ii}=1,\qquad
\Gamma_{i,i+1}\le1,\qquad
\Gamma_{i,i+k}\le\delta^{k-1},\quad k\ge2.
\]
The geometric decay of its off-diagonal entries
then bounds both the maximum row sum and the
maximum column sum by
\[
1+1+\sum_{r=1}^{\infty}\delta^r
=\frac{2-\delta}{1-\delta}.
\]
It follows from the standard relation between
the induced matrix norms that
\[
\|\Gamma\|_2
\le\sqrt{\|\Gamma\|_1\|\Gamma\|_\infty}
\le\frac{2-\delta}{1-\delta}.
\]

It remains to control the bounded-difference
vector associated with the partition. By the
definition of $C(c)$ in
Lemma~\ref{Mdiarmid_type_lemma}, each of its
coordinates is the sum of the coefficients
corresponding to one block. Since every block
contains at most $L$ observations, the
Cauchy--Schwarz inequality gives
\[
\|C(c)\|_2^2
=
\sum_{j=1}^{\widehat T}
\left(\sum_{t\in I(\widehat Y_j)}c_t\right)^2
\le L\|c\|_2^2.
\]
Combining this estimate with the bound on
$\|\Gamma\|_2$ and applying
Lemma~\ref{Mdiarmid_type_lemma}, we obtain
\[
\log\mathbb E\left[
e^{\lambda(f(Y)-\mathbb Ef(Y))}\right]
\le
\frac{\lambda^2\tilde{\tau}(\delta)}{8}
\left(\frac{2-\delta}{1-\delta}\right)^2
\|c\|_2^2.
\]
Finally, taking the infimum over
$\delta\in(0,1)$ yields
\[
\log\mathbb E\left[
e^{\lambda(f(Y)-\mathbb Ef(Y))}\right]
\le
\frac{\lambda^2\tilde{\tau}_{\min}\|c\|_2^2}{8},
\]
which establishes~\eqref{Concentration_under_condition_A}.



\section{Proof of Lemma \ref{max_probability_conditional_markov}}\label{Proof_max_probability_conditional_markov}

    \begin{align}
        \log\frac{1}{p_{w}(x_{k,t}|x_{k,t-1},\cdots,x_{k,t-\rho})}&=\log\frac{1}{\big<\sigma\left(g_w(\mathcal{E}^{\odot\rho}(\vec{x}))\right),\mathcal{E}(x_{k,t})\big>}\label{worst_case_prob_bound_4}\\
        &=\max_{i\in[d]}\log\left(\frac{\sum_{j\in[d]}\exp({g_{w,j}})}{\exp(g^{(j)}_w)\cdot\mathcal{E}_j(x^{(n)}_{k,t})}\right)\label{worst_case_prob_bound_3}\\
        &\leq \log\left(\frac{de^{B}}{e^{-B}}\right)\label{worst_case_prob_bound_2}\\
        &=2B+\log d\label{worst_case_prob_bound_1}
    \end{align}
    where

    \begin{itemize}
        \item equation \eqref{worst_case_prob_bound_3} follows from definition in \eqref{conditional_distribution_def} and \eqref{soft_max_def},
        \item and \eqref{worst_case_prob_bound_2} is due to  Assupmtion \eqref{subsubsec:bounded_output} which implies that $|{g_{w,j}}|\leq\|g_w\|\leq B$ for any $j\in[d]$.
    \end{itemize}

\section{Language Model Pretraining: Experimental Details}
\label{app:lm-pretraining-details}

This appendix complements Section~\ref{sec:lm-pretraining} with the full experimental protocol used to probe the SANTP regime on TinyStories.

\paragraph{Dataset and tokenization.}
We use the TinyStories corpus of \citep{eldan2023tinystoriessmalllanguagemodels}, a collection of short synthetic stories generated to use only words a typical 3--4 year-old would understand. We retain $8\%$ of the training split, which we found sufficient to expose the communication--generalization tradeoff while keeping the sweep tractable. A byte-level BPE tokenizer with vocabulary size $V = 8{,}000$ is trained from scratch on the selected stories using the \texttt{tokenizers} library, with special tokens \texttt{<pad>}, \texttt{<bos>}, \texttt{<eos>}, \texttt{<unk>} and a \texttt{<bos> $\cdot$ <eos>} template applied at encoding time. After tokenization, $10\%$ of the resulting token stream is reserved as a held-out test set; the remainder is used for training.

\paragraph{Model architecture.}
We use a GPT-style decoder-only transformer with $N_{\text{layer}} = 6$ blocks, embedding dimension $d_{\text{model}} = 512$, $H = 8$ attention heads, and an MLP expansion ratio of $4$, totaling roughly $25$M parameters. Each block uses pre-normalization with RMSNorm \citep{ZhangRMSNorm}, causal self-attention, and a SwiGLU MLP \citep{shazeer2020glu}. Input and output embeddings are tied. The maximum sequence length (block size) is $T = 256$ tokens, and dropout is set to $0.1$. Residual projections are initialized with a depth-scaled standard deviation of $0.02 / \sqrt{2 N_{\text{layer}}}$ following \citep{radford2019language}.

\paragraph{FedAvg protocol.}
Each configuration runs the standard FedAvg loop \citep{mcmahan2017communication}: at each communication round $r \in \{1, \dots, R\}$, every client receives the current global parameters, performs $\tau = M / R$ local AdamW steps on its own shard, and returns its updated parameters; the server then averages them uniformly to form the next global iterate. Local optimizer state (Adam moments) is reinitialized at the start of every round, which is the canonical FedAvg specification. Mini-batches within a round are drawn sequentially from the shard with shuffling at the window level.

\paragraph{Compute budget.}
To isolate the effect of communication frequency from that of total compute, we fix the per-client total gradient-step budget at $M = 10{,}000$ across all configurations and split it evenly across rounds, so that $\tau \cdot R = M$. Concretely, $R = 1$ corresponds to $\tau = 10{,}000$ local steps with a single aggregation, $R = 50$ to $\tau = 200$ local steps per round with $50$ aggregations, and so on. Every configuration thus consumes the same number of token-batches and the same number of optimizer updates per client: only the communication frequency $R$ varies. This rules out the possibility that the observed widening of the gap is an artifact of additional optimization.

\paragraph{Optimization hyperparameters.}
We use AdamW with learning rate $3 \cdot 10^{-4}$, weight decay $0.01$, default $(\beta_1, \beta_2) = (0.9, 0.999)$, and global gradient-norm clipping at $1.0$. The mini-batch size is $64$ sequences of length $256$, i.e.\ $\sim\!16$K tokens per gradient step. Hyperparameters were not tuned per configuration; the same values are used across all $R$ and all $K$, so any sensitivity to $R$ reflects the protocol itself rather than tuning effort.

\paragraph{Evaluation.}
At the end of each run we evaluate the final aggregated global model on (i) a fixed sample of training windows, drawn deterministically from the union of all shards, and (ii) the held-out test set. Each loss estimate averages cross-entropy over $100$ batches of $64$ sequences ($\sim\!1.6$M tokens). The reported \emph{generalization gap} is the difference $\widehat{\mathcal{L}}_{\text{val}} - \widehat{\mathcal{L}}_{\text{train}}$ at the end of training.

\paragraph{Seeds and reporting.}
We run three seeds per configuration. Each seed independently controls model initialization, data-loader shuffling, and AdamW state. Curves in Figure~\ref{fig:lm-pretraining} show the seed mean, with shaded bands at $\pm 1$ standard deviation across seeds. The values of $R$ swept are $\{1, 5, 10, 15, 20, 25, 30, 35, 40\}$.

\section{Time Series Forecasting}
\label{app:ts-forecasting-details}

This section provides supplementary details regarding the experimental setup and additional results for the time series forecasting task discussed in Section~\ref{sec:ts-forecasting}.

\subsection{Data Processing and Sharding}
To maintain the temporal integrity of the data while ensuring i.i.d.-like properties between training and evaluation, we implement a \textbf{block-level splitting strategy}. 
\begin{itemize}
    \item \textbf{Tokenization:} The raw univariate series is discretized into $B=20$ bins using uniform binning.
    \item \textbf{Splitting:} The sequence is divided into contiguous blocks of size $L=200$. Each block is assigned to the test set with a probability of $0.2$.
    \item \textbf{Leakage Prevention:} To prevent the model from "peeking" into the test set via the $\tau$-step history, we apply a buffer. For any training block immediately following or preceding a test block, we discard the first $\tau=24$ windows.
    \item \textbf{Federated Sharding:} The training windows are sharded \textbf{contiguously} across $K \in \{2, 4, 8\}$ clients. This ensures that each client's local data stream retains its Markovian dependence, simulating realistic streaming federated learning scenarios where nodes collect data locally over time.
\end{itemize}

\subsection{Model and Hyperparameters}
We utilize a Recurrent Neural Network (RNN) designed for discrete sequence prediction. The architecture consists of an embedding layer ($d_{emb}=64$), a 2-layer Gated Recurrent Unit (GRU) with a hidden dimension of $256$, and a linear output head. 

We use the Adam optimizer with a learning rate of $0.01$. Crucially, we keep the \textbf{total gradient-step budget} $T$ per client fixed at $1000$. When the number of communication rounds $R$ increases, the number of local steps per round is reduced proportionally ($E = T/R$). This isolates the effect of communication frequency from the effect of total computation.

\subsection{Additional Results}
The following figures provide results for the ETTh1 ~\citep{haoyietal-informer-2021} and Beijing Air Quality ~\citep{zhang2017cautionary} datasets, complementing the ETTh2 results shown in the main text. Both datasets exhibit the same communication--generalization tradeoff: while more frequent communication (higher $R$) reduces training loss, it consistently leads to an increase in the generalization gap.

\begin{figure*}[h]
    \centering
    \includegraphics[width=\textwidth]{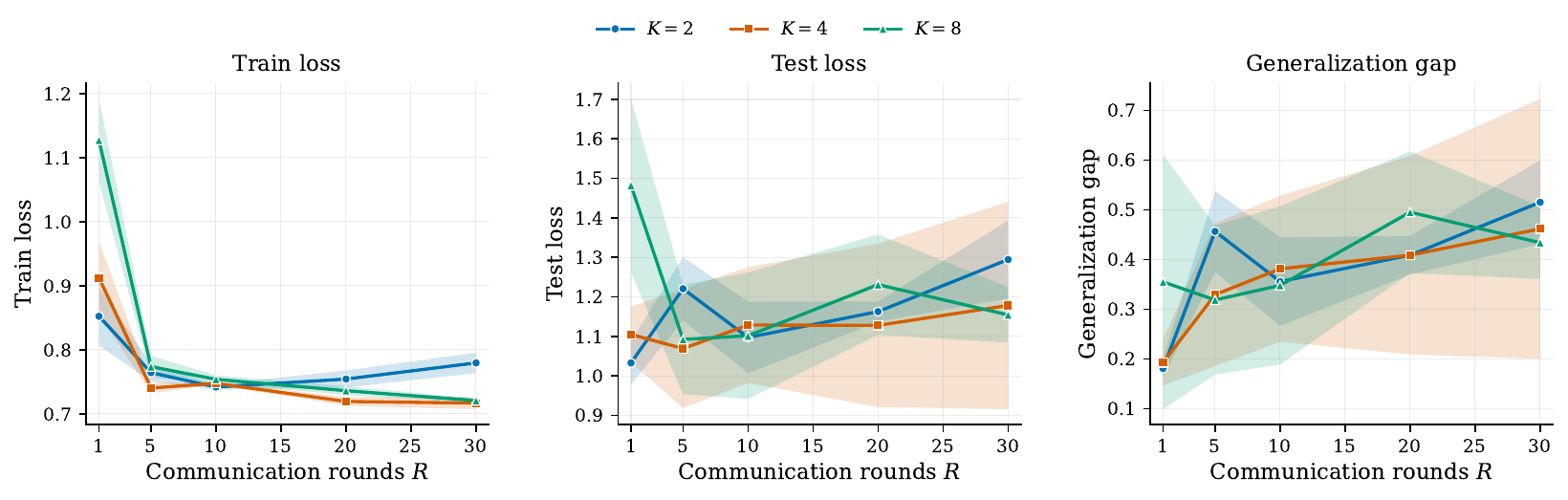}
    \caption{\textbf{Communication--generalization tradeoff: ETTh1 Dataset.} Consistent with ETTh2, increasing communication rounds $R$ leads to lower training loss but a significantly wider generalization gap across all client configurations $K$. Shaded regions indicate standard deviation across three seeds.}
    \label{fig:app-etth1}
\end{figure*}

\begin{figure*}[h]
    \centering
    \includegraphics[width=\textwidth]{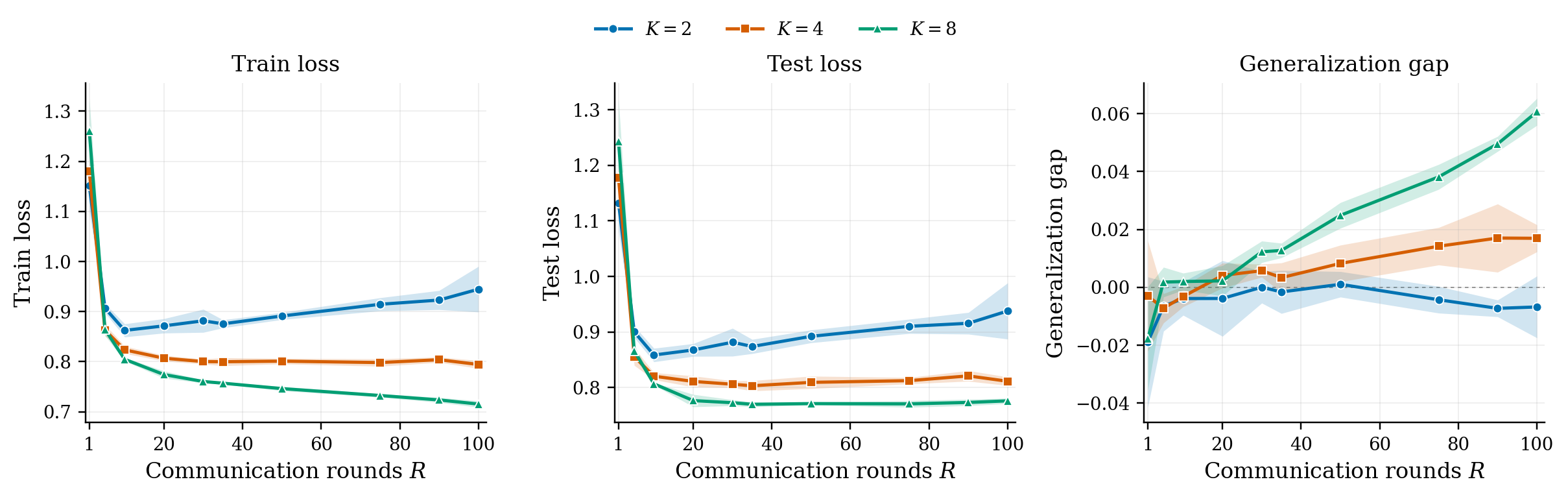}
    \caption{\textbf{Communication--generalization tradeoff: Beijing Dataset.} The same phenomenon is observed in environmental time series (PM2.5 levels), confirming that the tradeoff is a robust feature of streaming federated learning on dependent data. Shaded regions indicate standard deviation across five seeds.}
    \label{fig:app-beijing}
\end{figure*}

\end{document}

%% file: defs.tex
\usepackage{stmaryrd}
\SetSymbolFont{stmry}{bold}{U}{stmry}{m}{n}
\usepackage{euscript,steinmetz,float}
\usepackage[normalem]{ulem}
\usepackage[mode=buildnew]{standalone}
\usepackage{mathabx}
\makeatletter
\@ifundefined{theorem}{\newtheorem{theorem}{Theorem}}{}
\@ifundefined{lemma}{\newtheorem{lemma}{Lemma}}{}
\@ifundefined{proposition}{}{}
\@ifundefined{corollary}{\newtheorem{corollary}{Corollary}}{}
\@ifundefined{definition}{\newtheorem{definition}{Definition}}{}
\@ifundefined{remark}{\newtheorem{remark}{Remark}}{}
\@ifundefined{example}{}{}
\@ifundefined{conjecture}{}{}
\@ifundefined{axiom}{}{}
\@ifundefined{claim}{}{}
\@ifundefined{assumption}{\newtheorem{assumption}{Assumption}}{}
\makeatother

\usepackage{mleftright}\mleftright
\definecolor{darkgreen}{rgb}{0,0.5,0}
\definecolor{midgreen}{rgb}{0,0.7,0}
\definecolor{darkred}{RGB}{128,0,0}
\definecolor{darkpink}{RGB}{230,51,106}
\definecolor{darkestred}{RGB}{139,0,0}

\newcommand{\stkout}[1]{\ifmmode\text{\sout{\ensuremath{#1}}}\else\sout{#1}\fi}

\newcommand{\vc}[1]{\mathbf{#1}}

\DeclareMathOperator{\gen}{gen}

\DeclareMathOperator{\unif}{Unif}

\makeatletter
\DeclareFontFamily{OMX}{MnSymbolE}{}
\DeclareSymbolFont{MnLargeSymbols}{OMX}{MnSymbolE}{m}{n}
\SetSymbolFont{MnLargeSymbols}{bold}{OMX}{MnSymbolE}{b}{n}
\DeclareFontShape{OMX}{MnSymbolE}{m}{n}{ <-6> MnSymbolE5 <6-7> MnSymbolE6 <7-8> MnSymbolE7 <8-9> MnSymbolE8 <9-10> MnSymbolE9 <10-12> MnSymbolE10 <12-> MnSymbolE12 }{}
\DeclareFontShape{OMX}{MnSymbolE}{b}{n}{ <-6> MnSymbolE-Bold5 <6-7> MnSymbolE-Bold6 <7-8> MnSymbolE-Bold7 <8-9> MnSymbolE-Bold8 <9-10> MnSymbolE-Bold9 <10-12> MnSymbolE-Bold10 <12-> MnSymbolE-Bold12 }{}
\let\llangle\@undefined
\let\rrangle\@undefined
\DeclareMathDelimiter{\llangle}{\mathopen}{MnLargeSymbols}{'164}{MnLargeSymbols}{'164}
\DeclareMathDelimiter{\rrangle}{\mathclose}{MnLargeSymbols}{'171}{MnLargeSymbols}{'171}
\makeatother
\newcommand{\overbar}[1]{\mkern 1.5mu\overline{\mkern-1.5mu#1\mkern-1.5mu}\mkern 1.5mu}

\DeclareMathOperator*{\hess}{\mathrm{Hess}}

%% file: bib.bib
@inproceedings{haoyietal-informer-2021,
  author    = {Haoyi Zhou and
               Shanghang Zhang and
               Jieqi Peng and
               Shuai Zhang and
               Jianxin Li and
               Hui Xiong and
               Wancai Zhang},
  title     = {Informer: Beyond Efficient Transformer for Long Sequence Time-Series Forecasting},
  booktitle = {The Thirty-Fifth {AAAI} Conference on Artificial Intelligence, {AAAI} 2021, Virtual Conference},
  volume    = {35},
  number    = {12},
  pages     = {11106--11115},
  publisher = {{AAAI} Press},
  year      = {2021},
}

@article{zhang2017cautionary,
  title={Cautionary tales on air-quality improvement in Beijing},
  author={Zhang, Shuyi and Guo, Bin and Dong, Anlan and He, Jing and Xu, Ziping and Chen, Song Xi},
  journal={Proceedings of the Royal Society A: Mathematical, Physical and Engineering Sciences},
  volume={473},
  number={2205},
  year={2017},
  publisher={The Royal Society}
}

@inproceedings{ZhangRMSNorm,
 author = {Zhang, Biao and Sennrich, Rico},
 booktitle = {Advances in Neural Information Processing Systems},
 editor = {H. Wallach and H. Larochelle and A. Beygelzimer and F. d\textquotesingle Alch\'{e}-Buc and E. Fox and R. Garnett},
 pages = {},
 publisher = {Curran Associates, Inc.},
 title = {Root Mean Square Layer Normalization},
 url = {https://proceedings.neurips.cc/paper_files/paper/2019/file/1e8a19426224ca89e83cef47f1e7f53b-Paper.pdf},
 volume = {32},
 year = {2019}
}

@misc{shazeer2020glu,
      title={GLU Variants Improve Transformer}, 
      author={Noam Shazeer},
      year={2020},
      eprint={2002.05202},
      archivePrefix={arXiv},
      primaryClass={cs.LG},
      url={https://arxiv.org/abs/2002.05202}, 
}

@misc{eldan2023tinystoriessmalllanguagemodels,
      title={TinyStories: How Small Can Language Models Be and Still Speak Coherent English?}, 
      author={Ronen Eldan and Yuanzhi Li},
      year={2023},
      eprint={2305.07759},
      archivePrefix={arXiv},
      primaryClass={cs.CL},
      url={https://arxiv.org/abs/2305.07759}, 
}

@inproceedings{chor2023more,
  title = {More {{Communication Does Not Result}} in {{Smaller Generalization Error}} in {{Federated Learning}}},
  booktitle = {2023 {{IEEE International Symposium}} on {{Information Theory}} ({{ISIT}})},
  author = {Chor, Romain and Sefidgaran, Milad and Zaidi, Abdellatif},
  year = {2023},
  pages = {48--53}
}

@article{sefidgaran2022rate,
  title={Rate-distortion theoretic bounds on generalization error for distributed learning},
  author={Sefidgaran, Milad and Chor, Romain and Zaidi, Abdellatif},
  journal={Advances in Neural Information Processing Systems},
  volume={35},
  pages={19687--19702},
  year={2022}
}

@inproceedings{steinke2020reasoning,
  title={Reasoning about generalization via conditional mutual information},
  author={Steinke, Thomas and Zakynthinou, Lydia},
  booktitle={Conference on Learning Theory},
  pages={3437--3452},
  year={2020},
  organization={PMLR}
}

@article{xu2017information,
  title={Information-theoretic analysis of generalization capability of learning algorithms},
  author={Xu, Aolin and Raginsky, Maxim},
  journal={Advances in neural information processing systems},
  volume={30},
  year={2017}
}

@inproceedings{woodworth2020local,
  title={Is local SGD better than minibatch SGD?},
  author={Woodworth, Blake and Patel, Kumar Kshitij and Stich, Sebastian and Dai, Zhen and Bullins, Brian and Mcmahan, Brendan and Shamir, Ohad and Srebro, Nathan},
  booktitle={International Conference on Machine Learning},
  pages={10334--10343},
  year={2020},
  organization={PMLR}
}

@inproceedings{sefidgaran2024lessons,
  title={Lessons from Generalization Error Analysis of Federated Learning: You May Communicate Less Often!},
  author={Sefidgaran, Milad and Chor, Romain and Zaidi, Abdellatif and Wan, Yijun},
  booktitle={Forty-first International Conference on Machine Learning},
  year={2024}
}

@inproceedings{Sefidgaran2022,
  title={Rate-distortion theoretic generalization bounds for stochastic learning algorithms},
  author={Sefidgaran, Milad and Gohari, Amin and Richard, Gael and Simsekli, Umut},
  booktitle={Conference on Learning Theory},
  pages={4416--4463},
  year={2022},
  organization={PMLR}
}

@inproceedings{mcmahan2017communication,
  title={Communication-efficient learning of deep networks from decentralized data},
  author={McMahan, Brendan and Moore, Eider and Ramage, Daniel and Hampson, Seth and y Arcas, Blaise Aguera},
  booktitle={Artificial intelligence and statistics},
  pages={1273--1282},
  year={2017},
  organization={PMLR}
}

@InProceedings{gronlund2020,
	title = 	 {Near-Tight Margin-Based Generalization Bounds for Support Vector Machines},
	author =       {Gr{\o}nlund, Allan and Kamma, Lior and Larsen, Kasper Green},
	booktitle = 	 {Proceedings of the 37th International Conference on Machine Learning},
	pages = 	 {3779--3788},
	year = 	 {2020},
	editor = 	 {III, Hal Daumé and Singh, Aarti},
	volume = 	 {119},
	series = 	 {Proceedings of Machine Learning Research},
	month = 	 {13--18 Jul},
	publisher =    {PMLR},
}

@inproceedings{yuksel2025sample,
  title={On the sample complexity of next-token prediction},
  author={Y{\"u}ksel, O{\u{g}}uz Kaan and Flammarion, Nicolas},
  booktitle={The 28th International Conference on Artificial Intelligence and Statistics},
  year={2025}
}

@article{russo2016controlling,
  title={Controlling bias in adaptive data analysis using information theory},
  author={Russo, Daniel and Zou, James},
  journal={arXiv preprint arXiv:1511.05219},
  year={2016}
}

@inproceedings{bu2020tightening,
  title={Tightening mutual information-based bounds on generalization error},
  author={Bu, Zhiqi and Zou, Jinghui and Long, Qi and Su, Weijie J},
  booktitle={IEEE Journal on Selected Areas in Information Theory},
  volume={1},
  number={1},
  pages={121--130},
  year={2020}
}

@article{paulin2015concentration,
  title={Concentration inequalities for {M}arkov chains by
         {M}arton couplings and spectral methods},
  author={Paulin, Daniel},
  journal={Electronic Journal of Probability},
  volume={20},
  pages={1--32},
  year={2015}
}

@inproceedings{huynh2025streaming,
  title={Streaming Federated Learning with Markovian Data},
  author={Huynh, Tan-Khiem and Egan, Malcolm and Neglia, Giovanni and Gorce, Jean-Marie},
  booktitle={Advances in Neural Information Processing Systems (NeurIPS)},
  volume={39},
  year={2025},
  url={https://arxiv.org/abs/2503.18807}
}

@inproceedings{marfoq2023federated,
  title={Federated learning for data streams},
  author={Marfoq, Othmane and Neglia, Giovanni and Kameni, Laetitia and Vidal, Richard},
  booktitle={International conference on artificial intelligence and statistics},
  pages={8889--8924},
  year={2023},
  organization={PMLR}
}

@inproceedings{mcmahanetal2017communication,
  author    = {McMahan, Brendan and Moore, Ewan and Ramage, Daniel and Hampson, Seth and y Arcas, Blaise Aguera},
  title     = {Communication-Efficient Learning of Deep Networks from Decentralized Data},
  booktitle = {Proceedings of the 20th International Conference on Artificial Intelligence and Statistics (AISTATS)},
  year      = {2017}
}

@article{liu2023dynamite,
  author={Liu, Wenchao and Zhang, Xiaoxi and Duan, Jingjie and Joe-Wong, Carlee and Zhou, Zhi and Chen, Xu},
  journal={IEEE Transactions on Mobile Computing}, 
  title={DYNAMITE: Dynamic Interplay of Mini-Batch Size and Aggregation Frequency for Federated Learning With Static and Streaming Dataset}, 
  year={2023},
  volume={22},
  number={12},
  pages={7664-7679},
  doi={10.1109/TMC.2022.3217036},
  month={July}
}

@inproceedings{schwarz2018progress,
  title={Progress \& compress: A scalable framework for continual learning},
  author={Schwarz, Jonathan and Czarnecki, Wojciech and Luketina, Jelena and Grabska-Barwinska, Agnieszka and Teh, Yee Whye and Pascanu, Razvan and Hadsell, Raia},
  booktitle={International Conference on Machine Learning (ICML)},
  pages={4528--4537},
  year={2018},
  organization={PMLR}
}

@article{shalev2012online,
  title={Online learning and online convex optimization},
  author={Shalev-Shwartz, Shai},
  journal={Foundations and Trends{\textregistered} in Machine Learning},
  volume={4},
  number={2},
  pages={107--194},
  year={2012},
  publisher={Now Publishers, Inc.},
  doi={10.1561/2200000018}
}

@article{greco2019edge,
  title={An edge-stream computing infrastructure for real-time analysis of wearable sensors data},
  author={Greco, Luca and Ritrovato, Pierluigi and Xhafa, Fatos},
  journal={Future Generation Computer Systems},
  volume={93},
  pages={515--528},
  year={2019},
  issn={0167-739X},
  doi={10.1016/j.future.2018.10.058},
  publisher={Elsevier}
}

@article{hill2010anomaly,
  title={Anomaly detection in streaming environmental sensor data: A data-driven modeling approach},
  author={Hill, David J. and Minsker, Barbara S.},
  journal={Environmental Modelling \& Software},
  volume={25},
  number={9},
  pages={1014--1022},
  year={2010},
  publisher={Elsevier},
  doi={10.1016/j.envsoft.2010.02.002}
}

@inproceedings{hard2018federated,
  title={Federated Learning for Mobile Keyboard Prediction},
  author={Hard, Andrew and Rao, Kanishka and Mathews, Rajiv and Ramaswamy, Swaroop and Beaufays, Fran{\c{c}}oise and Augenstein, Sean and Eichner, Hubert and Kiddon, Chlo{\'e} and Ramage, Daniel},
  booktitle={arXiv preprint arXiv:1811.03604},
  year={2018},
  url={https://arxiv.org/abs/1811.03604}
}

@article{radford2019language,
  title   = {Language Models are Unsupervised Multitask Learners},
  author  = {Radford, Alec and Wu, Jeffrey and Child, Rewon and Luan, David and Amodei, Dario and Sutskever, Ilya},
  journal = {OpenAI Blog},
  year    = {2019},
  url     = {https://cdn.openai.com/better-language-models/language_models_are_unsupervised_multitask_learners.pdf}
}

@book{vankampen1992stochastic,
  title={Stochastic Processes in Physics and Chemistry},
  author={Van Kampen, Nicolaas Godfried},
  year={1992},
  publisher={Elsevier},
  address={Oxford},
  edition={Revised and enlarged},
  isbn={978-0-444-89349-9}
}

@inproceedings{odeyomi2021differentially,
  title={Differentially-Private Federated Learning with Long-Term Constraints Using Online Mirror Descent},
  author={Odeyomi, Olusola and Zaruba, Gergely},
  booktitle={2021 IEEE International Symposium on Information Theory (ISIT)},
  pages={1308--1313},
  year={2021},
  organization={IEEE},
  doi={10.1109/ISIT45174.2021.9518177}
}

@inproceedings{beznosikov2023first,
  title={First order methods with markovian noise: from acceleration to variational inequalities},
  author={Beznosikov, Aleksandr and Samsonov, Sergey and Sheshukova, Marina and Gasnikov, Alexander and Naumov, Alexey and Moulines, Eric},
  booktitle={Advances in Neural Information Processing Systems},
  volume={36},
  pages={44820--44835},
  year={2023},
  publisher={Curran Associates, Inc.}
}

@inproceedings{duchi2011ergodic,
  title={Ergodic mirror descent},
  author={Duchi, John C and Agarwal, Alekh and Johansson, Mikael and Jordan, Michael I},
  booktitle={2011 49th Annual Allerton Conference on Communication, Control, and Computing (Allerton)},
  pages={701--706},
  year={2011},
  organization={IEEE},
  doi={10.1109/Allerton.2011.6120244}
}

@inproceedings{sun2018markov,
  title={On {M}arkov {C}hain {G}radient {D}escent},
  author={Sun, Tao and Sun, Yuejiao and Yin, Wotao},
  booktitle={Advances in Neural Information Processing Systems (NeurIPS)},
  volume={31},
  pages={6496--6505},
  year={2018},
  publisher={Curran Associates, Inc.},
  url={https://proceedings.neurips.cc/paper/2018/file/693892e85ef15852924ec98b4700d5a3-Paper.pdf}
}

@article{malach2023auto,
  title={Auto-regressive next-token predictors are universal learners},
  author={Malach, Eran},
  journal={arXiv preprint arXiv:2309.06979},
  year={2023},
  url={https://arxiv.org/abs/2309.06979}
}

@misc{chen2017beijing,
  title={Beijing Multi-Site Air Quality},
  author={Chen, Song},
  year={2017},
  publisher={UCI Machine Learning Repository},
  doi={10.24432/C5RK5G},
  url={https://archive.ics.uci.edu/dataset/501/beijing+multi+site+air+quality+data}
}

@inproceedings{gronlund2020near,
  title={Near-tight margin-based generalization bounds for support vector machines},
  author={Gr{\o}nlund, Allan and Kamma, Lior and Larsen, Kasper Green},
  booktitle={International Conference on Machine Learning},
  pages={3779--3788},
  year={2020},
  organization={PMLR}
}

@article{kairouz2021advances,
  title={Advances and Open Problems in Federated Learning},
  author={Kairouz, Peter and McMahan, H. Brendan and others},
  journal={Foundations and Trends in Machine Learning},
  volume={14},
  number={1--2},
  pages={1--210},
  year={2021},
  doi={10.1561/2200000083}
}

@article{konecny2016federated,
  title={Federated Optimization: Distributed Machine Learning for On-Device Intelligence},
  author={Kone{\v{c}}n{\'y}, Jakub and McMahan, H. Brendan and Ramage, Daniel and Richt{\'a}rik, Peter},
  journal={arXiv preprint arXiv:1610.02527},
  year={2016}
}

@article{stich2019local,
  title={Local SGD Converges Fast and Communicates Little},
  author={Stich, Sebastian U.},
  journal={International Conference on Learning Representations},
  year={2019}
}

@inproceedings{karimireddy2020scaffold,
  title={{SCAFFOLD}: Stochastic Controlled Averaging for Federated Learning},
  author={Karimireddy, Sai Praneeth and Kale, Satyen and Mohri, Mehryar and Reddi, Sashank and Stich, Sebastian and Suresh, Ananda Theertha},
  booktitle={Proceedings of the 37th International Conference on Machine Learning},
  volume={119},
  pages={5132--5143},
  year={2020},
  series={Proceedings of Machine Learning Research}
}

@inproceedings{li2020fedprox,
  title={Federated Optimization in Heterogeneous Networks},
  author={Li, Tian and Sahu, Anit Kumar and Zaheer, Manzil and Sanjabi, Maziar and Talwalkar, Ameet and Smith, Virginia},
  booktitle={Proceedings of Machine Learning and Systems},
  volume={2},
  pages={429--450},
  year={2020}
}

@inproceedings{reddi2021adaptive,
  title={Adaptive Federated Optimization},
  author={Reddi, Sashank and Charles, Zachary and Zaheer, Manzil and Garrett, Zachary and Rush, Keith and Kone{\v{c}}n{\'y}, Jakub and Kumar, Sanjiv and McMahan, H. Brendan},
  booktitle={International Conference on Learning Representations},
  year={2021}
}

@inproceedings{li2020fedavg,
  title={On the Convergence of FedAvg on Non-IID Data},
  author={Li, Xiang and Huang, Kaixuan and Yang, Wenhao and Wang, Shusen and Zhang, Zhihua},
  booktitle={International Conference on Learning Representations},
  year={2020}
}

@article{yu1994rates,
  title={Rates of Convergence for Empirical Processes of Stationary Mixing Sequences},
  author={Yu, Bin},
  journal={The Annals of Probability},
  volume={22},
  number={1},
  pages={94--116},
  year={1994}
}

@article{marton1996bounding,
  title={Bounding $\bar{d}$-Distance by Informational Divergence: A Method to Prove Measure Concentration},
  author={Marton, Katalin},
  journal={The Annals of Probability},
  volume={24},
  number={2},
  pages={857--866},
  year={1996}
}

@article{samson2000concentration,
  title={Concentration of Measure Inequalities for Markov Chains and $\Phi$-Mixing Processes},
  author={Samson, Paul-Marie},
  journal={The Annals of Probability},
  volume={28},
  number={1},
  pages={416--461},
  year={2000}
}

@article{kontorovich2008concentration,
  title={Concentration Inequalities for Dependent Random Variables via the Martingale Method},
  author={Kontorovich, Leonid Aryeh and Ramanan, Kavita},
  journal={The Annals of Probability},
  volume={36},
  number={6},
  pages={2126--2158},
  year={2008}
}

@inproceedings{mohri2009stability,
  title={Rademacher Complexity Bounds for Non-I.I.D. Processes},
  author={Mohri, Mehryar and Rostamizadeh, Afshin},
  booktitle={Advances in Neural Information Processing Systems},
  volume={21},
  year={2009}
}

@article{ralaivola2010chromatic,
  title={Chromatic PAC-Bayes Bounds for Non-IID Data: Applications to Ranking and Stationary $\beta$-Mixing Processes},
  author={Ralaivola, Liva and Szafranski, Marie and Stempfel, Guillaume},
  journal={Journal of Machine Learning Research},
  volume={11},
  pages={1927--1956},
  year={2010}
}

@article{kuznetsov2017generalization,
  title={Generalization Bounds for Non-stationary Mixing Processes},
  author={Kuznetsov, Vitaly and Mohri, Mehryar},
  journal={Machine Learning},
  volume={106},
  pages={93--117},
  year={2017}
}

@inproceedings{NIPS2017_b22b257a,
 author = {Bartlett, Peter and Foster, Dylan J and Telgarsky, Matus J},
 booktitle = {Advances in Neural Information Processing Systems},
 editor = {I. Guyon and U. Von Luxburg and S. Bengio and H. Wallach and R. Fergus and S. Vishwanathan and R. Garnett},
 pages = {},
 publisher = {Curran Associates, Inc.},
 title = {Spectrally-normalized margin bounds for neural networks},
 url = {https://proceedings.neurips.cc/paper_files/paper/2017/file/b22b257ad0519d4500539da3c8bcf4dd-Paper.pdf},
 volume = {30},
 year = {2017}
}

@inproceedings{negrea2019information,
  title={Information-Theoretic Generalization Bounds for SGLD via Data-Dependent Estimates},
  author={Negrea, Jeffrey and Haghifam, Mahdi and Dziugaite, Gintare Karolina and Khisti, Ashish and Roy, Daniel M.},
  booktitle={Advances in Neural Information Processing Systems},
  volume={32},
  year={2019}
}

@inproceedings{arora2018stronger,
  title={Stronger Generalization Bounds for Deep Nets via a Compression Approach},
  author={Arora, Sanjeev and Ge, Rong and Neyshabur, Behnam and Zhang, Yi},
  booktitle={Proceedings of the 35th International Conference on Machine Learning},
  volume={80},
  pages={254--263},
  year={2018},
  series={Proceedings of Machine Learning Research}
}

@inproceedings{vaswani2017attention,
  title={Attention Is All You Need},
  author={Vaswani, Ashish and Shazeer, Noam and Parmar, Niki and Uszkoreit, Jakob and Jones, Llion and Gomez, Aidan N. and Kaiser, Lukasz and Polosukhin, Illia},
  booktitle={Advances in Neural Information Processing Systems},
  volume={30},
  year={2017}
}

@inproceedings{edelman2022inductive,
  title={Inductive Biases and Variable Creation in Self-Attention Mechanisms},
  author={Edelman, Benjamin L. and Goel, Surbhi and Kakade, Sham M. and Zhang, Cyril},
  booktitle={Proceedings of the 39th International Conference on Machine Learning},
  volume={162},
  pages={5793--5831},
  year={2022},
  series={Proceedings of Machine Learning Research}
}
